\documentclass[Afour,sageh,times]{sagej}

\usepackage[bookmarks=true]{hyperref}
\hypersetup{
    colorlinks,
    linkcolor={red!50!black},
    citecolor={blue!50!black},
    urlcolor={blue!80!black}
}

\usepackage{multicol}
\usepackage{graphicx}
\usepackage{epstopdf}
\usepackage{mathtools}
\usepackage{amsfonts}
\usepackage{algorithm}
\usepackage{algpseudocode}
\usepackage{algorithmicx}
\usepackage{wrapfig}
\usepackage{subfiles}
\usepackage{booktabs}
\usepackage{flushend}

\usepackage{amsmath}
\usepackage[english]{babel}
\usepackage{amsthm} 
\usepackage{amsmath}
\usepackage{thmtools,thm-restate}
\usepackage{float}
\usepackage{siunitx}
\usepackage{subcaption}
\usepackage{enumitem} 
\usepackage{cancel}
\usepackage{bbm}
\usepackage{ifthen}
\usepackage{colortbl}
\usepackage[dvipsnames]{xcolor}
\usepackage{arydshln}
\usepackage[utf8]{inputenc}
\usepackage{color,soul}

\usepackage{eso-pic}
\usepackage[dvipsnames]{xcolor}

\definecolor{lightgray}{gray}{0.9}

\newcommand\BibTeX{{\rmfamily B\kern-.05em \textsc{i\kern-.025em b}\kern-.08em
T\kern-.1667em\lower.7ex\hbox{E}\kern-.125emX}}

\DeclareMathOperator*{\argmax}{arg\,max}
\DeclareMathOperator*{\argmin}{arg\,min}
\newtheorem{definition}{Definition}

\theoremstyle{lemma}

\def\volumeyear{2023} 
\begin{document}

\AddToShipoutPictureBG*{
  \AtPageUpperLeft{%
    \put(\LenToUnit{0.5\paperwidth},\LenToUnit{-1.0cm}){%
      \makebox[0pt][c]{\textcolor{BrickRed}{\textbf{Accepted for publication in the International Journal of Robotics Research (IJRR).}}}
    }%
  }%
}

\runninghead{Murillo-González and Liu}

\title{Situationally-Aware Dynamics Learning}

\author{}
\author{Alejandro Murillo-González\affilnum{1} {\normalsize and}
    Lantao Liu\affilnum{1}%
}%
\affiliation{%
    \affilnum{1} A. Murillo-González and L. Liu are with the Luddy School of Informatics, Computing, and Engineering, Indiana University, Bloomington, IN 47408, USA. Email: {\tt\small \{almuri, lantao\}@iu.edu}
}%
\corrauth{Lantao Liu}
\email{lantao@iu.edu}

\begin{abstract}
Autonomous robots operating in complex, unstructured environments face significant challenges due to latent, unobserved factors that obscure their understanding of both their internal state and the external world. Addressing this challenge would enable robots to develop a more profound grasp of their operational context. To tackle this, we propose a novel framework for online learning of hidden state representations, with which the robots can adapt in real-time to uncertain and dynamic conditions that would otherwise be ambiguous and result in suboptimal or erroneous behaviors. Our approach is formalized as a Generalized Hidden Parameter Markov Decision Process, which explicitly models the influence of unobserved parameters on both transition dynamics and reward structures. Our core innovation lies in learning online the joint distribution of state transitions, which serves as an expressive representation of latent ego- and environmental-factors.  This probabilistic approach supports the identification and adaptation to different operational situations, improving robustness and safety.  Through a multivariate extension of Bayesian Online Changepoint Detection, our method segments changes in the underlying data generating process governing the robot's dynamics. The robot's transition model is then informed with a symbolic representation of the current situation derived from the joint distribution of latest state transitions, enabling adaptive and context-aware decision-making. 
To demonstrate effectiveness, we validate our approach on an unmanned ground vehicle operating in diverse unstructured terrains, both in simulation and in real-world experiments. We also evaluate a quadrotor in simulation under randomly changing wind conditions. Both setups introduce unmodeled and unmeasured environmental factors that substantially affect robot motion.
Extensive experiments in both simulation and real world reveal significant improvements in data efficiency, policy performance, and the emergence of safer, adaptive navigation strategies. 
Website: \url{https://alejandromllo.github.io/research/situational-awareness/}.
\end{abstract}


\maketitle

\section{Introduction} \label{sec:introduction}

\begin{figure*}
    \centering
    \includegraphics[width=\linewidth]{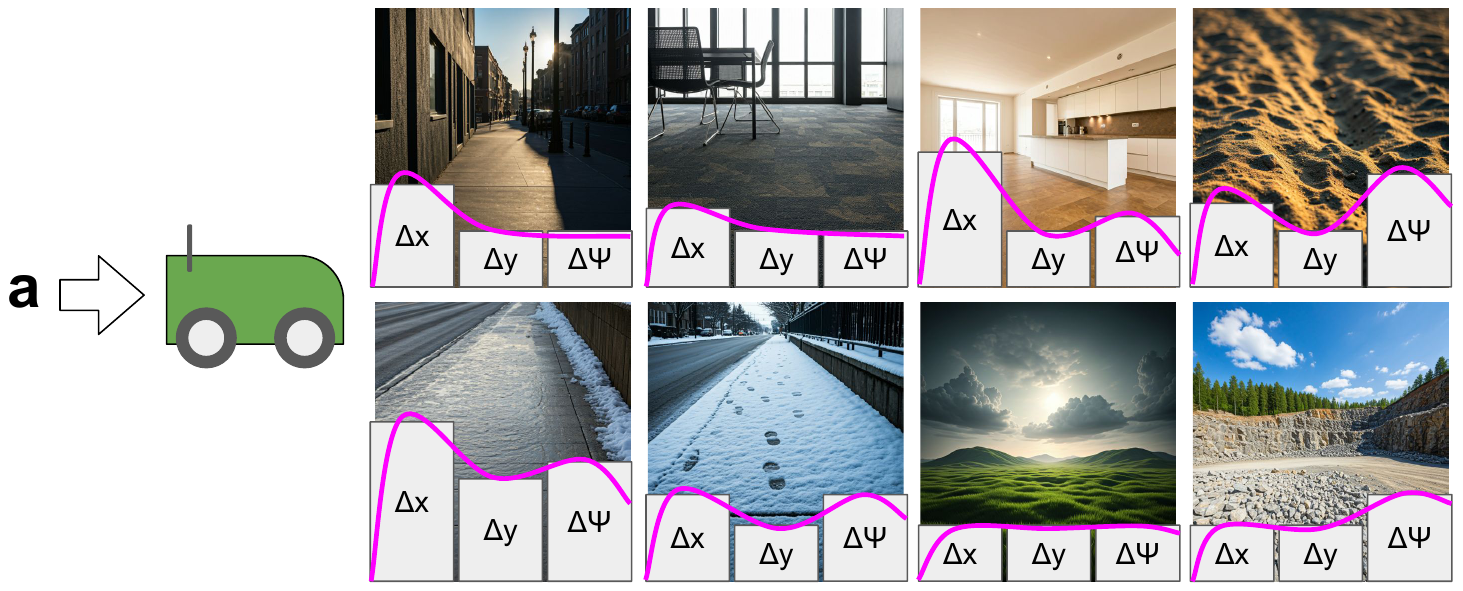} 
    \caption{Motivating Example. When the robot has access only to incomplete state information, the same action, $\mathbf{a}$, can yield significantly different outcomes. In this example, unmodeled terrain types—factors absent from the state and transition model—have an effect on the robot's dynamics, 
    resulting in distinct changes in position and heading ($\Delta x$, $\Delta y$, and $\Delta \psi$, respectively) for the same control action (commanded linear and angular velocity). Generalized Hidden Parameter Markov Decision Processes (GHP-MDPs) address this challenge by modeling the unknown latent factors responsible for these discrepancies and incorporating them into the transition model, thereby enhancing predictive accuracy. (Illustrative environment images generated with Gemini AI.)}
    \label{fig:output-discrepancy}
\end{figure*}    

For reliable robot deployment in unstructured and dynamic environments, it is crucial for robotic systems to adapt to unforeseen challenges and operate effectively under uncertainty. Consider, for example, an Unmanned Ground Vehicle (UGV) navigating steep and rugged terrain. Such a robot may become immobilized due to unexpected factors, ranging from unmodeled terrain properties to variations in its own internal dynamics. These factors are inherently difficult, if not impossible, to anticipate exhaustively \cite[]{du2011robot, kurniawati2012global, green2020planning, siva2021enhancing, yin2023reliable, xu2024boundary}. This challenge is further exacerbated by the impracticality of equipping robots with every conceivable sensor and analysis algorithm needed to fully capture the world's and robot’s full state in all scenarios. Consequently, robots should learn to operate effectively using only partial information derived from their observable state—the limited set of measurements and estimates available in real time. This necessitates the development of adaptive mechanisms capable of inferring and representing the hidden aspects of the robot and its environment to inform decision-making and enhance resilience in uncertain conditions \cite[]{mccallum1996hidden, ferguson2004pao, kroemer2014learning, kobayashi2022latent}.

Furthermore, when a robot relies on partial information to operate in a complex environment, the risk of ambiguous scenario representation becomes significant \cite[]{meyer2011using, pramanick2022talk, celemin2023knowledge}. Such ambiguity is undesirable because it increases uncertainty during decision-making, often producing suboptimal plans or, in extreme cases, unsafe behavior. This challenge is illustrated in Figure~\ref{fig:output-discrepancy}: a UGV with incomplete state information—lacking a (robust) scene interpretation module to detect the terrain it is traversing—executes the \textit{same} action from what appears to be an identical initial state, yet the observed outcomes differ dramatically; these discrepancies cannot be explained solely by aleatoric noise (process or observation noise). While we use terrain as an easy-to-visualize example, many other unobserved factors can similarly perturb dynamics (e.g., payload shifts, wheel or tire wear, actuator degradation, sensor bias, unseen vegetation or obstacles, local slope/contact conditions, etc.). These sources often introduce systematic, context-dependent shifts that complicate learning and control, underscoring the need for methods that can infer and resolve hidden-state ambiguities to ensure reliable robot operation.

To address this, we propose a framework for learning online a representation of the unmodeled and unobserved latent factors that cause such ambiguities. This task is challenging because it requires balancing representation capacity, data- and learning-efficiency, and minimizing inductive biases that might inadvertently exclude critical latent factors \cite[]{langdon2019uncovering, hundt2006representing}. Motivated by these constraints, we introduce an online hidden-state representation learning method that enables the robot to recover and leverage compact representations of latent world and robot factors in real time. By exploiting this learned representation the robot can dynamically adapt its behavior and reliably accomplish tasks even in uncertain and unstructured environments. Prior work has attempted to address this by learning distributions over causal structures—an approach explicitly intended to account for uncertainty in the dynamics model caused in part by unobserved latent factors \cite[]{murillogonzalez2025causaldistributions}; our method instead targets compact, real-time latent state representations for online adaptation.

Recent studies have introduced various methods that leverage learned auxiliary representations during dynamics or policy learning to enable robot adaptation \cite[]{lee2020learning, kumar2021rma, kumar2022adapting, qi2023hand, liang2024rapid}. However, these methods are data-intensive and often rely on multi-stage procedures that require privileged information, such as ground-truth data that is difficult to obtain or is only available in simulation. To overcome these limitations, our proposed method learns such representations in an online and unsupervised manner, thus eliminating the need for multi-stage approaches. We achieved this by formulating the problem within the framework of Generalized Hidden Parameter Markov Decision Processes (GHP-MDPs) \cite[]{perez2020generalized, doshi2016hidden}. 

An important component of our framework is to represent the hidden state by learning online the \textit{joint distribution of the state transitions}. We refer to this state transition distribution at a given moment as the \textit{situation} of the robot. This approach enables the robot to capture an expressive probabilistic representation of its current context, encompassing both its state and the underlying environmental conditions. By leveraging proprioceptive sensing, such as inertial measurements, to assess the robot's state, this representation becomes sufficiently general to effectively disambiguate the robot's state across various scenarios.  
Additionally, as illustrated in Figure \ref{fig:output-discrepancy}, the state evolution exhibits diverse probability density profiles under different conditions. This diversity facilitates the identification of latent factors that govern distinct scenarios.
In this way, the \textit{situation} can be understood as the Underlying Data Generating Process (UDGP)—a conceptual model that characterizes the true, often hidden, factors governing the system's dynamics at any given moment \cite[p. 109]{Goodfellow-et-al-2016}. The UDGP encapsulates the combined influence of both robot-specific and environmental variables that affect the robot’s behavior, enabling a more accurate understanding of its state. To obtain this situation distribution in real-time as the robot operates, we introduce a \textit{multivariate extension} of Bayesian Online Changepoint Detection (BOCD) \cite[]{adams2007bayesian}. This extension allows us to model the situation the robot is currently operating under and detect when a situation change occurs.

 Concretely, in this paper we make the following contributions:

 \begin{itemize}
\item 
 We introduce an efficient method that enables a robot to adapt online to the latent factors influencing its performance. We present a multivariate BOCD approach to estimate the distribution of the transition dynamics in real time. We show that these distributions serve as effective representations for describing the context in which the robot is currently operating. 
 \item 
 By mapping the estimated situation distributions to symbolic representations, we provide the robot’s dynamics model with access to the latent information necessary to adapt to the current context. Notably, we utilize the cumulant generating function of the joint transition distributions to symbolically represent the robot's current situation, facilitating informed decision-making and robust performance in dynamic and uncertain environments.

\item We validate the framework on a UGV navigating unstructured terrain (simulation and real world), where the dynamics model is learned via active data collection within a Model-Based Reinforcement Learning (MBRL) loop, and on a simulated quadrotor under changing winds, whose model is trained using a dataset of state transitions. This allows us to showcase the suitability of our method for both active and passive learning settings, as well as its applicability across different platform types. We observe that the situationally-aware models show substantially improved data efficiency and outperform RL baselines and physics-based controllers. Furthermore, the learned hidden-state representation facilitates the emergence of adaptive behaviors that increase the safety and effectiveness of the downstream planners.
 \end{itemize}

The remainder of this work is structured as follows. Section \ref{sec:relatedwork} reviews the relevant literature, focusing on policy and dynamics learning methods that incorporate learned auxiliary representations. Section \ref{sec:preliminaries} introduces GHP-MDPs and compares them to related MDP families. We also present the problem of changepoint detection. In Section~\ref{sec:method}, we detail our proposed approach, describing how to model transition dynamics online, represent latent factors, and integrate them into the transition model in a data-efficient way. In Section~\ref{sec:results_ugv} we demonstrate that our approach yields safer and more data-efficient adaptation for a UGV, validated in both simulation and real-world trials. Section~\ref{sec:results_quadx} presents complementary results for a quadrotor evaluated in simulation under randomly varying wind fields. Together, these evaluations show that our situationally-aware dynamics representation enables adaptive and more robust control across both ground and aerial platforms.
Finally, Sections \ref{sec:discussion} and \ref{sec:conclusion} discuss our findings and outline the conclusions drawn from our work, respectively.

\section{Related Work} \label{sec:relatedwork}

Our work focuses on efficiently learning a robot's dynamics model in conjunction with a hidden state representation to enable safe and effective adaptation in ambiguous scenarios. 
We review relevant literature on planning and learning with auxiliary representations—both latent and symbolic—along with research that specifically addresses the problem of robot motion control in unstructured environments.

\subsection{Neural-Symbolic Learning in Robotics} 

The field of neural-symbolic learning consists of two primary branches \cite[]{wang2024imperative}. We will first focus on methods rooted in discrete logical reasoning and guiding abstractions. The second branch involves Physics-Infused Neural Networks (PINNs), which handle continuous abstractions to represent physical characteristics, which we cover later on.

Analogous to our work, \citet{gopalan2020simultaneously} used changepoint detection for robotic applications. They use it to segment demonstration trajectories from which the robot learns skills. Then they learn symbols by identifying the termination conditions of the skills via clustering. Their objective was to map behaviors to natural language commands. However, for training they rely on demonstration trajectories that have been paired to natural language instructions, and find the symbols from the provided dataset. Conversely, in our case the agent discovers the symbols online and unsupervised, from the data generated by its proprioception, while it is exploring the environment.

\cite{konidaris2012robot} construct skill trees from demonstration trajectories. Each skill has a goal and an abstraction and they can later be composed. Similarly, they use a changepoint detection algorithm in charge of determining when the most relevant abstraction changes. However, they cannot recognize repeated skills and require the availability of an abstraction library while we directly model the robot's dynamics which can then be used with sampling-based planners. Other authors have also worked along this line \cite[]{kulic2009online, chiappa2010movement, niekum2012learning}. Moreover, learned symbolic abstractions for planning, skill learning and control have been explored in \cite[]{silver2022learning, achterhold2023learning, belta2007symbolic, kumar2023learning, rosen2023synthesizing}.

Research that is in proximity also includes \citet{ly2012learning} which propose a method for learning a system's hybrid dynamics directly from time-series data through multi-modal symbolic regression. Similarly, \cite{kim2022physics} focus on estimating key physical parameters, such as tire slippage and longitudinal force, in the context of an autonomous vehicle model, offering critical insights for improving vehicle performance and control. Closer to our work, \cite{zhao2024physord} employ physics-informed neural networks (PINNs) to predict motion in off-road driving scenarios, addressing the challenges posed by rough and unpredictable terrain while ensuring more accurate trajectory forecasting.

\subsection{Model-Based Reinforcement Learning} 

System Identification (SI) aims to discover and model the relationship between the inputs and outputs of a system \cite[]{aastrom1971system, ljung2010perspectives}. MBRL is a SI method based on balancing exploration and exploitation. It iteratively collects more informative data for the task at hand, thanks to planning with an improved model learned with the data collected in previous iterations. Thus, the model evolves, leading to better planning performance which results in exploring more informative sections of the state-action space \cite[]{sutton2018reinforcement}.

\cite{abbeel2006using} explore how to use an approximate model with few real-life trials. They evaluate a policy using a sample trial from the real world, and leverage the approximate model to estimate the derivative of the evaluation with respect to the policy parameters. \cite{deisenroth2011pilco} introduce a model-based policy search strategy for high-dimensional control tasks. They employ Gaussian Processes to learn the dynamics and incorporate the model's uncertainty during planning and policy evaluation. They highlight the data efficiency of this approach. \citet{chua2018deep} propose PETS, which uses ensembles of uncertainty-aware neural-networks together with sampling-based optimizers for planning. 
\citet{hansen2023td} presented TD-MPC2, which learns five components consisting of an \textit{(i)} observation encoder, \textit{(ii)} latent-space dynamics model, \textit{(iii)} transition reward predictor, \textit{(iv)} return or terminal value predictor, and \textit{(v)} a policy predicting return-maximizing actions; to predict returns from a sequence of actions, usually from high-dimensional observations. Other works have also learned latent dynamics models \cite[]{ha2018world, hafner2019dream}. MoDem-V2 \cite[]{lancaster2023modem} builds on top of TD-MPC to learn vision-based models in the real-world for robot manipulation tasks. They identify unsafe exploration and over-optimism as a significant failure point for applying MBRL in the real-world and develop three techniques for safer real-world learning: policy centering, agency transfer and actor-critic ensembles. 

\subsection{Policy and Dynamics Learning with Learned Auxiliary Representations} \label{sec:related-work-auxreps}

Learned auxiliary representations like symbols or latent variables have also been used to provide robots with valuable environmental or proprioceptive insights. These are then leveraged for improved performance while learning and deploying a policy or dynamics model.

\subsubsection{Model-Free Reinforcement Learning:} 

\cite{lee2020learning} use a temporal convolutional network that takes a window of past proprioceptive states to generate control commands for quadruped locomotion over challenging terrains. This model is trained to ``implicitly reason about contact and slippage events" from the proprioception. However, learning the previous model requires a teacher policy that has been previously obtained using privileged information of the terrain and the robot's contact with it. Furthermore, they also make use of an adaptive curriculum in charge of modulating the terrain's difficulty according to the current performance of the policy.    \cite{kumar2021rma} proposed Rapid Motor Adaptation (RMA), which is a two-stage process to train adaptive policies via Model-Free RL (MFRL). First, an ``environmental factors encoder" model is trained to map ``privileged information" only available in simulation to a latent vector. This latent vector is used as input together with the current state and action to a ``base-policy" trained via Proximal Policy Optimization \cite[PPO]{schulman2017proximal}. Then, during the second phase, an ``adaptation module" is trained to map a fixed-size window of past state-action pairs into a latent vector optimized to match the latent representations produced by the ``environmental factors encoder". This way, during deployment, the latent representation can be obtained without the need for privileged information. \cite{kumar2022adapting} add a third phase to the RMA framework to enable adaptive bipedal locomotion. The new stage involves using PPO to fine-tune the base policy but maintaining the adaptation module weights frozen, to ensure that the policy supports imperfect estimation of the latent representation. Follow up works to RMA include \cite{qi2023hand} which also uses proprioception history and privileged information to train a model for in-hand object rotation. Also, \cite{liang2024rapid} learn the latent representation from depth perception for manipulation tasks.

In contrast to the previous works, our method does not need privileged information, as our representations are obtained directly from proprioception. This also means that we do not need multi-stage training, as we can directly detect the symbols during training of the dynamics model. Furthermore, given the structured nature of our learned representation, the symbols can be similarly detected directly in the real-world, so no further fine-tuning is required.

\subsubsection{Model-Based Reinforcement Learning:} 

\cite{havens2019learning} recognize that model-based methods are often restricted to tasks with limited dimensions or rely on certain assumptions about the system's dynamics. Furthermore, they argue that high-dimensional models are challenging for planning due to increased compounding error, as also argued in \cite[]{stachowicz2024racer, talvitie2014model, talvitie2017self}. Thus, they compress the state space to avoid learning unimportant features that could only contribute to compounding error during planning. Specifically, they propose to learn a state space encoder, a dynamics model in the latent space and action space, as well as a reward predictor. They place special emphasis in optimizing for reward prediction performance to encourage better compression of the state space. Closer to our work, \cite{sharma2019dynamics} combine MBRL with primitives learned via MFRL to make model-based planning easier. In particular, they simultaneously identify ``predictable behaviors" and the corresponding dynamics. Their key idea is to learn skills that are optimized to be easy to identify and for which a simpler accompanying dynamics model can be obtained. Finally, during deployment, the skill and corresponding dynamics model become a set of primitives that can be composed to solve the commanded task via model-based planning. In \cite{sharma2020emergent}, it is shown that this framework can be leveraged for RL in the real world. \cite{lee2020context} introduces Context-aware Dynamics Models (CaDM) to generalize across different dynamics. They learn a context encoder, forward and backward dynamics models. Both dynamics models conditioned by the latent vector in charge of encoding dynamics-specific information.

By comparison, our method is also under the model-based umbrella but we do not learn multiple dynamics models, that is, we have a single model that works under every context instead of requiring a model for each context. Additionally, the structured nature of our symbolic representation enables us to explicitly capture the underlying data-generating process of the dynamics at a given point in time, instead of a harder-to-interpret latent unstructured representation.

\subsubsection{Situational Awareness:}
Other methods that do not rely on RL but leverage representations of the system's situation include pervasive computing, where situation identification is used to abstract noisy sensor data into high-level concepts \cite[]{ye2012situation}. For them, each situation is defined according to how the data is semantically interpreted by the users, which could require preliminary domain knowledge to obtain them. On the other hand, our work automatically builds situations by modelling the underlying data generation process that best explains the present sensor measurements. \citet{langari2005intelligent} use situation identification to determine a policy for intelligent energy management in parallel hybrid vehicles. They do it by abstracting environment, driving style and car operating mode information using short term statistical features. However, they learn to classify the situations in a supervised manner and selecting which statistics to use requires task-specific expert-knowledge, which we avoid by modeling the UDGP of the transition dynamics. Other robot applications involve human-robot-cooperation to determine whether the robot maintains SA based on surprise \cite[]{dahn2018situation}. 

\subsection{Generalized Hidden Parameter Markov Decision Process (GHP-MDP)}

\cite{perez2020generalized} introduce GHP-MDPs and develop a learning-based method where agents can generalize, in simulation, to the same family of tasks using transition models whose input is augmented with estimated latent factors. They model the latent factors with a diagonal Gaussian while minimizing the Kullback–Leibler divergence between the variational and prior distributions. The main difference to our approach, is that we detect the latent representation online, which allows us to directly augment the (state, action, next-state)-tuples for fine-tuning the dynamics with the symbol assigned to each representation of the transition distribution. Furthermore, we use a full-rank covariance matrix, instead of a less expressive diagonal covariance.

Related work also includes Hidden Parameter MDPs (HiP-MDPs) which assume the agent knows what it does not know. Specifically, it does not know the value of some critical environment variables and focuses on estimating these \cite[]{doshi2016hidden}. \cite{zhang2020learning} learn state abstractions to enable multi-task RL. However, they assume the latent parameter is fixed for the duration of the episode, while in our case we are interested in detecting and accommodating to changes in the latent parameters during the episode. \cite{killian2017robust} extend this framework for robust and efficient transfer learning to subtle variations among tasks using Bayesian Neural Networks. \cite{costen2023planning} develop a method to express the transition probabilities in closed-form as a polynomial of the latent parameters, as well as maintaining a closed-form belief distribution of the latent parameters. They present promising results in two simulation tasks. In our case, the transition model's input is augmented with the symbol of the detected latent representation, since we do not assume prior-knowledge about the variables that the agent will require to obtain a better model of the dynamics at deployment time.

\subsection{Unstructured Terrain Navigation}
We validate our approach in the real world through the challenging task of unstructured terrain navigation, where unmodeled and unmeasured terrain characteristics can significantly impact the robot's motion. 
Unstructured terrain navigation has been tackled mostly from two perspectives: adaptation and path planning. Works addressing this from the adaptive perspective typically devise methods that allow the robot to recover from unexpected or hard scenarios. For example, \citet{xu2024kernel} proposes a method to bypass the precise modeling of the state transition function of an MDP with continuous state space and shows that it can be applied for off-road navigation and recovering from unexpected phenomena (e.g., wheel slippage). \citet{wang2021rough} learns a probabilistic dynamics model trained to consider the propagating effects of uncertainty and uses it along a tracking controller and constrained optimization to find trajectories with low divergence. Terrain-aware methods have been recently developed to generate consistent navigation behaviors by learning a control offset from demonstrations \cite[]{siva2019robot, siva2021enhancing}. \cite{pokhrel2024cahsor} develop a multi-stage approach where terrain representations for offroad navigation are first learned in a self-supervised manner from visual, inertial and speed observation. Then in a subsequent stage, the terrain representation weights are frozen and the forward kinodinamic model is learned from a terrain interaction dataset. In contrast, our method learns the situation representations online and unsupervised.
On the other hand, the methods concerned with global planning typically use the robot's sensors to analyze the terrain and try to find the trajectory that seems safer and easier to navigate \cite[]{yoo2024traversability, jardali2024autonomous, leininger2024gaussian}.  For example, \cite{castro2022does} use self-supervised learning to combine exteroceptive and proprioceptive feedback to predict traversability costmaps.
Other methods have done so by focusing on risk-awareness \cite[]{sharma2023ramp, cai2022risk, cai2023probabilistic}. Learning-based approaches have considered ground-robot interaction and surface data \cite[]{sathyamoorthy2022terrapn}. Our method can be framed at the intersection of both perspectives. In particular, we obtain dynamics models with greater robustness to terrain challenges and use those to make local motion plans that help the robot traverse safely under partial state information.

\section{Preliminaries} \label{sec:preliminaries}

\subsection{Generalized Hidden Parameter Markov Decision Process (GHP-MDP)} \label{sec:prelim-ghpmdp}

Markov Decision Processes (MDPs) provide a framework for modeling decision-making with uncertainty over the action outcomes due to an agent's interaction with its environment \cite[]{sutton2018reinforcement}. Hidden Parameter MDPs (HiP-MDP) extend MDPs to scenarios where predicting the future also requires a set of unobservable parameters $\Theta$ that correspond to the sufficient statistics to adapt to a new situation \cite[]{doshi2016hidden}. In HiP-MDPs, $\Theta$ explicitly represent a critical variable of the environment; for example, physical properties like the terrain friction. This makes it necessary to know a-priori which missing variables need to be estimated by $\Theta$, which is a strong assumption we cannot fulfill when the robot is operating in unstructured environments as we cannot anticipate every challenge it will face.

\cite{perez2020generalized} build upon HiP-MDPs and present GHP-MDPs to account for unobserved, latent parameters which influence the environment's dynamics and reward function. These latent factors are not directly observable by the agent, but the agent must infer them through interaction with the environment to optimize its policy.

\begin{definition} \label{def:ghp-mdp}
A Generalized Hidden Parameter MDP (GHP-MDP) is defined by a tuple $\mathcal{M}_{\text{GHP}} = \left( S, A, \Theta, T, R, \gamma \right)$, where \( S \) is the set of \textit{observable states} and \( A \) is the set of \textit{actions} available to the agent. \( \Theta \) is the set of \textit{latent variables} that are not observable by the agent but influence the system’s dynamics and rewards. \( T: S \times \Theta \times A \times S \to [0, 1] \) is the \textit{transition function} which gives the probability of transitioning from state \( s \in S \) to state \( s' \in S \) under action \( a \in A \) and hidden parameter \( \theta \). \( R: S \times \Theta \times A \to \mathbb{R} \) is the \textit{reward function}. And the discount factor \( \gamma \in [0, 1] \) modulates the agent's desire for immediate and future rewards.
\end{definition}

As the agent in a GHP-MDP does not directly observe \( \theta \), it must maintain a \textit{belief} \( b(\theta) \) over the possible values of the latent variables. In this work, we present a novel, efficient and unsupervised approach for representing $b(\theta)$ online as the distribution of $(s, a, s')$-tuples being currently observed. Each distribution is then mapped to a symbol that will be used by $T$ to adapt to the current situation of the robot.

\subsubsection{Comparison to other MDP Models:}

Throughout the years multiple extensions of the MDP framework have been proposed. In the following, we highlight key differences to some of them to justify selecting a GHP-MDP to formulate the decision making problem in unstructured environments with unobservable and changing state variables:

\textit{MDP \cite[]{bellman1958dynamic}:} The state \( s \) is fully observable, and the environment’s dynamics are completely determined by the state and action, i.e., \( T(s, a, s') = P(s' | s, a) \).

\textit{POMDP \cite[]{aastrom1965optimal, kaelbling1998planning}:} The agent maintains a belief over the partially-observable state $s$. It must infer the state using the observations \( o \). Conversely, in a GHP-MDP, we do not necessarily have uncertainty over the observable part of the state \( s \), and the uncertainty arises from the unknown latent factors \( \theta \) that are affecting the system dynamics. 

\textit{Hidden Model MDP, Contextual MDP (cMDP) or Latent-Variable MDP (LMDP) \cite[]{chades2012momdps, hallak2015contextual, kwon2021rl, steimle2021multi}:} The model consists of a set of transition and reward functions from the same state and action space. The number of transition and reward functions is given by the number of contexts the agent could be in and the value function is the expectation over these, weighted according to a set of mixing weights. On the other hand, GHP-MDPs have a single transition and reward function whose input is augmented by the latent factors representation.

\subsection{Model-Based Reinforcement Learning} \label{sec:appendix-mbrl}

An MBRL agent interacts with the environment by learning its model and using it to make decisions. A key difference with other system identification techniques is the exploratory nature of MBRL. Specifically, the agent acts on the environment while balancing exploration and exploitation \cite[]{sutton2018reinforcement}. Every episode it visits more informative regions as it avoids what it expects to be known or irrelevant configurations of the state-action space, which translates to faster model learning via informed data collection. This approach contrasts with MFRL, where the agent directly learns a policy solely based on interactions without explicit knowledge of the environment’s dynamics.

Formally, an MBRL agent learns a model of the environment, defined by a transition function \( \hat{T}(s' | s, a) \) and a reward function \( \hat{R}(s, a) \), which approximate the true environment dynamics \( T(s' | s, a) \) and reward \( R(s, a) \), respectively. To learn $\hat{T}$, 
the transition information $\mathcal{R} = \{(s_i, a_i, s_{i+1})\}_{i=1}^N$ obtained through interaction and/or previously collected data is used for fine-tunning at the beginning of each episode. Specifically, the objective is to minimize the error between the predicted dynamics and the true dynamics of the environment:
\begin{equation}
\min_{\hat{T}, \hat{R}} \mathbb{E}_{(s, a, s') \sim \mathcal{D}} \left[ \left( s' - \hat{T}(s, a) \right)^2 \right].
\end{equation}

Then, during the episode the agent can simulate trajectories using the learned models \( (\hat{T}, \hat{R}) \), and optimize the policy \( \pi \) using any standard planning algorithm, such as value iteration or model predictive control (MPC). For example, to get the optimal action sequence from state $s_t$: 
\begin{equation}
    a^\star_{t:t+H} = \argmax_{a_{t:t+H}} \sum_{\substack{h=0}}^{\substack{H-1}} \mathbb{E}_{\hat{s}_{h}} [\gamma^h \hat{R}(\hat{s}_{h}, a_{h}) | \hat{s}_0 = s_t],
\end{equation}
where $\hat{s}_{h} = \hat{T}(\hat{s}_{h-1}, a_{h-1})$. See Algorithm \ref{alg:mbrl} for an overview of MBRL. Note that Lines 3 \& 8 are the main difference with Model Predictive Control (MPC), since these are responsible of iteratively improving the model as more experience is accumulated. In particular, the process in Line 3 is skipped during deployment once a good model has been obtained and the model can be used by a planner to synthesize a policy for the task.

\begin{algorithm}
\caption{MBRL Framework Overview}\label{alg:mbrl}
\begin{algorithmic}[1]
\State Initialize the replay buffer $\mathcal{R}$ with a random controller or previously obtained data.
\For{\texttt{Episode} k = 1 \texttt{to} K}
    \State Train the \textit{transition model} $\hat{T}$ using $\mathcal{R}$. 
    \For{\texttt{Time} t = 1 \texttt{to} T}  
        \State Find ${a}^\star_{t:t+H}$ using a planner or a learned policy.
        \State Execute the first action ${a}^\star_t$ from ${a}^\star_{t:t+H}$.
        \State Record the new state ${s}_{t+1}$.
        \State Update the dataset: $\mathcal{R} \leftarrow \mathcal{R} \cup \{{s}_t, {a}^\star_t, {s}_{t+1}\}$.
    \EndFor
\EndFor
\end{algorithmic}
\end{algorithm}

\subsection{Bayesian Online Changepoint Detection (BOCD)} \label{sec:prelim-bocd}

The problem of changepoint detection is concerned with determining the point where the observed data distribution changes in an ordered set of measurements, such as in time-series \cite[]{fearnhead2019changepoint}. BOCD was introduced by \citet{adams2007bayesian} to tackle this problem by framing it as an estimation of the posterior distribution of the current ``run length" $r_t$, meaning how likely it is that the measurement at time-step $t$ belongs to the same data generating process that started $r_t$ timesteps ago; while also obtaining the parameters $\boldsymbol{\eta}$ defining such process.

Intuitively, BOCD continuously monitors a data stream to identify points where the underlying data distribution changes. It starts with a prior belief about where changes might occur and updates this belief by evaluating how well the current data fits different scenarios of changepoints (segment/run lengths). At each time step, the algorithm uses the new data to adjust the likelihood of potential changepoints, incorporating prior knowledge and observed evidence. The method recursively updates this distribution to assess the existence of changes in the UDGP of the current segment.

Formally, the method assumes that the sequence of observations $\boldsymbol{x}_1, \boldsymbol{x}_2, \dots, \boldsymbol{x}_T$ can be segmented into non-overlapping intervals, where a changepoint corresponds to the point $\boldsymbol{x}_i$ marking the transition between two adjacent intervals. The set $\boldsymbol{x}_t^{(r)}$ contains the points estimated to belong to run $r_t$. Furthermore, the data within each interval $\rho$ is \textit{i.i.d.} from $P(\boldsymbol{x}_t | \boldsymbol{\eta}_\rho)$, and the parameters $\boldsymbol{\eta}_\rho, \rho = 1, 2, \dots$ are also \textit{i.i.d.} Note that the last assumption is about the \textit{parameters} of the underlying distributions describing the data generating process of each segment. Finally, we should be able to compute the predictive distribution $P(\boldsymbol{x}_{t+1} | r_t, \boldsymbol{x}_t^{(r)})$ and define a conditional prior on the changepoint $P(r_t | r_{t-1})$.

To initialize the algorithm, BOCD considers two scenarios: \textit{(a)} the changepoint took place before the initial observation, then $P(r_0 = 0) = 1$, and \textit{(b)} we are seeing a recent subset of the data, then $P(r_0 = \tau) = \frac{1}{Z} \sum_{t = \tau + 1}^{\infty} P_{\mathtt{gap}}(g=t)$, where $Z$ is a normalizing constant and $P_{\mathtt{gap}}(g)$ is the multinomial prior across the changepoints' domain. 

The conditional changepoint prior $P(r_t | r_{t-1})$ is the key for the algorithm's performance, since it is defined to have non-zero mass in two of three cases: 
\begin{equation}
    P(r_t | r_{t-1}) = \begin{cases}
         H(r_{t-1} + 1) &\text{if } r_t = 0, \\
         1 - H(r_{t-1} + 1) &\text{if } r_t = r_{t-1} +1, \\
         0 &\text{otherwise}.
    \end{cases}  
\end{equation}
where $H(\tau)$ is the \textit{hazard} function \cite[]{forbes2011statistical}. Subsequently, finding the distribution of the run length $P(r_t | \boldsymbol{x}_{1:t})$, involves recursive computation over the run length and the observations: 
\begin{align} \label{eq:joint-runlen-obs}
    P(r_t, &\boldsymbol{x}_{1:t}) = \sum_{\substack{r_{t-1}}} P(r_t, r_{t-1}, \boldsymbol{x}_{1:t}) \\ 
    &= \sum_{\substack{r_{t-1}}} P(r_t, \boldsymbol{x}_t | r_{t-1}, \boldsymbol{x}_{1:t-1}) P(r_{t-1}, \boldsymbol{x}_{1:t-1}) \nonumber \\ &= \sum_{\substack{r_{t-1}}} P(r_t | r_{t-1}) P(\boldsymbol{x}_t | r_{t-1}, \boldsymbol{x}_t^{(r)}) P(r_{t-1}, \boldsymbol{x}_{1:t-1}). \nonumber
\end{align} where the predictive distribution $P(\boldsymbol{x}_t | r_{t-1}, \boldsymbol{x}_{1:t})$ depends only on the data of the current run $\boldsymbol{x}_t^{(r)}$.

$P(r_t | \boldsymbol{x}_{1:t})$ is updated with each new observation $\boldsymbol{x}_t$, by first obtaining its predictive probability 
$\pi_t^{(r)} = P(\boldsymbol{x}_t | \boldsymbol{\eta}_t^{(r)})$, and then calculating the \textit{growth probability} (probability that the run has not ended and we are still observing data from the same UDGP): 
\begin{align} \label{eq:growth-prob-def}
    P(r_t = r_{t-1}+1, &\boldsymbol{x}_{1:t}) =\\ &P(r_{t-1}, \boldsymbol{x}_{1:t-1}) \pi_t^{(r)} (1 - H(r_{t-1})). \nonumber
\end{align} 
and \textit{changepoint probability} (probability that the UDGP has changed): 
\begin{align} \label{eq:changepoint-prob-def}
    P(r_t=0, \boldsymbol{x}_{1:t}) = \sum_{\substack{r_{t-1}}} P(r_{t-1}, \boldsymbol{x}_{1:t-1}) \pi_t^{(r)} H(r_{t-1}).
\end{align}
From which we can obtain,
\begin{equation} \label{eq:pos-run-length}
    P(r_t | \boldsymbol{x}_{1:t}) = \frac{P(r_t, \boldsymbol{x}_{1:t})}{P(\boldsymbol{x}_{1:t})} = \frac{P(r_t, \boldsymbol{x}_{1:t})}{\sum_{\substack{r_t}} P(r_t, \boldsymbol{x}_{1:t})}.
\end{equation}
Finally, after obtaining $P(r_t | \boldsymbol{x}_{1:t})$ using the latest observation $\boldsymbol{x}_t$, we update the sufficient statistics of the data generating process' model. 

\section{Method} \label{sec:method}

We propose a method that learns a situationally-aware dynamics model by identifying latent, unobserved factors from online interaction data. We organize our presentation in a top-down manner. First, we describe the overall pipeline: a model-based reinforcement learning agent that leverages a learned representation of the current `\textit{situation}' to improve planning and control. This high-level overview explains how the system works as a whole. Second, we detail our core contribution: an online and unsupervised method for discovering these situations by detecting changes in the robot's transition dynamics. This section provides the theoretical underpinnings of our approach.

\begin{figure*}[t!]
    \centering
    \includegraphics[width=0.9\linewidth]{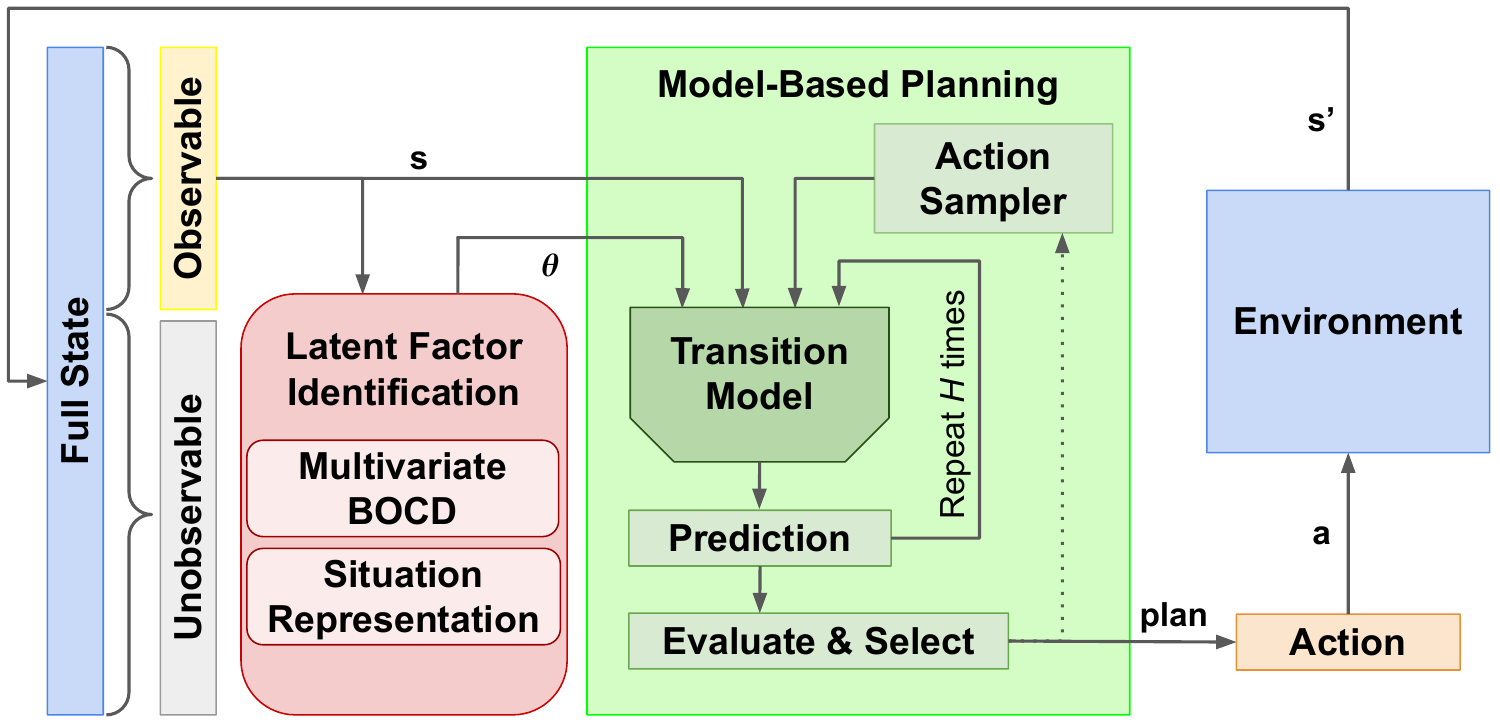}
    \caption{System overview. Our situation identification method leverages the observable state to model online the local transition distribution, in order to build a representation  of how the latent factors currently affect the robot's dynamics. In other words, the learned representation encompasses the hidden or unobservable component of the state. Using the recovered information, the model-based planner can consider the particularities of the current situation via an informed transition model.}
    \label{fig:method-flow-diagram}
\end{figure*}

\subsection{Situationally-Aware Dynamics}

For effective planning and control in changing or unstructured environments, an agent must account for hidden state variables that influence its dynamics. We address this by augmenting the standard MDP transition model with a parameter $\theta \in \Theta$ that represents these unmodeled world characteristics. This results in a GHP-MDP with transition function:
\begin{equation} \label{eq:transition_func_ghpmdp}
    T(s, \theta, a, s') = P(S_{t+1} = s' \mid S_t = s, \Theta_t = \theta, A_t = a),
\end{equation}
where $S_{t}, \Theta_t, A_t \text{ and } S_{t+1}$ are random variables representing the state of the world, the current representation of the latent factors, and the actions taken by the agent. The parameter $\theta$ encapsulates the current ``situation" of the robot (e.g., traversing slippery ice versus firm ground in a UGV navigation scenario), allowing the dynamics model to adapt its predictions accordingly. An overview of our approach is shown in Figure \ref{fig:method-flow-diagram}.

In practical terms, we model the right-hand side of Eq. \eqref{eq:transition_func_ghpmdp} as an ensemble $\mathcal{M}_m : S \times \Theta \times A \rightarrow S$ of $m$ probabilistic (Gaussian) neural networks. The ensemble's input is augmented with a representation of the current situation, $\theta$, in addition to the states and actions. Together, these inputs enable the ensemble to learn the effects of the hidden state on the robot's dynamics. This architecture choice estimates the predictive distribution over the next state and simultaneously considers the epistemic and aleatoric uncertainty \cite[]{chua2018deep}.

\subsubsection{Training and Testing Setup:} \label{sec:training_testing}

The system operates in two distinct phases: training and testing. The \textbf{training phase} is designed to build the situation library $\boldsymbol{\Psi}$ and learn the dynamics model $\mathcal{M}_m$, and can be performed in either an active or passive manner.

In the \textit{active setting}, the agent interacts directly with the environment. It learns by alternating exploration and exploitation, following the MBRL framework described in Algorithm \ref{alg:mbrl}. During this process, transition data is collected and our situation identification algorithm (detailed in Sec. \ref{sec:online_unsupervised}) runs in parallel to discover new situations and populate the library $\boldsymbol{\Psi}$. The collected transitions are stored in a replay buffer $\mathcal{R} = \mathcal{R} \cup \{(s_i, \theta_i, a_i, s_{i+1})\}_{i=1}^{N_e}$, where each tuple also includes the symbol $\theta_i$ representing the identified situation $\psi_i \in \boldsymbol{\Psi}$ and $N_e$ are the number of steps in the episode.

On the other hand, in the \textit{passive setting}, the agent learns from a pre-existing, static dataset of transitions $\mathcal{D} = \{(s_i, a_i, s_{i+1})\}_{i=1}^{N_d}$. Here, no exploration occurs. Instead, to construct the situation-annotated replay buffer, we process the dataset $\mathcal{D}$ sequentially: as each transition is fed to our system, our online situation identification algorithm discovers the underlying situations present in the data and builds the library $\boldsymbol{\Psi}$. This allows us to annotate each tuple with its corresponding situation symbol $\theta_i$, creating the final replay buffer $\mathcal{R} = \{(s_i, \theta_i, a_i, s_{i+1})\}_{i=1}^{N_d}$.

It is worth noting that a hybrid configuration is also possible: the agent can start from a pre-collected dataset (passive learning) and then transition to active data collection for further dynamics model refinement.

Regardless of the training paradigm, the resulting replay buffer $\mathcal{R}$ that will be used to learn the model should contain transitions annotated with situation symbols. Furthermore, given that the period of time the robot experiences in each situation is different, the replay buffer $\mathcal{R}$ will naturally have an unbalanced sample of situations. Thus, to ensure that the dynamics $\mathcal{M}_m$ are learned appropriately for all situations $\psi \in \boldsymbol{\Psi}$, we introduce a loss weighted by the cardinality of each situation. Specifically, we minimize the weighted Negative Log-Likelihood (NLL):
\begin{equation} \label{eq:loss-function}
    \mathcal{L}(\theta) = - \sum_{i=1}^{|\mathcal{R}|} w(\theta_i) \left( \log \hat{\sigma}_i^2 \| \hat{\mu}_i - y_i \|_2^2 + \log \hat{\sigma}_i^2 \right),
\end{equation} 
where $\hat{\mu}_i$ and $\hat{\sigma}_i^2$ are the predicted mean and variance, and $y_i$ is the measured outcome, and each prediction is weighted as: 
\begin{equation}
    w(\theta_i) = \frac{|\mathcal{R}|}{\sum_{j=1}^{|\mathcal{R}|} \mathbb{I}(\theta_j = \theta_i)},
\end{equation}
where \(\mathbb{I}(\theta_j = \theta_i)\) is an indicator function equal to 1 if \(\theta_j = \theta_i\) and 0 otherwise. This frequency-weighted loss enhances the model’s focus on rare values of \(\theta_i\), encouraging balanced learning across all samples.

During the \textbf{testing phase}, the situation library $\boldsymbol{\Psi}$ is held fixed. The agent no longer creates new situations. Instead, it continuously identifies its current situation by comparing its local dynamics to the learned models in $\boldsymbol{\Psi}$ and selecting the best match (detailed in Sec. \ref{sec:sit-modeling-and-id}). This identified situation informs the dynamics model $\mathcal{M}_m$, allowing for adaptive control in response to changing conditions.

\subsubsection{Planning and Control:}
We use our situationally-aware dynamics model $\mathcal{M}_m$ together with the Model Predictive Path Integral (MPPI) method \citep{gandhi2021robust}. The planning process involves generating a set of candidate action sequences over a time horizon $H$ and selecting the sequence that maximizes the expected return. The expected return is computed based on the predicted future states while assuming that the situation symbol $\theta$ remains constant throughout the rollouts. To select an action, MPPI optimizes the function:
\begin{equation}
    a^\star_{t:t+H} = \arg\min_{a_{t:t+H}} \sum_{i=0}^{H-1} \left( \| \hat{s}_{t+i} - s_{\text{goal}} \|_2 + \rho \mathcal{B}(\hat{s}_{t+i}) \right),
\end{equation}
where $\hat{s}_{t+i}$ denotes the predicted state at time $t+i$, $s_{\text{goal}}$ is the goal state, $\rho > 0$ is a weight hyperparameter and $\mathcal{B}(\hat{s}_{t+i})$ is a log-barrier function for safe navigation around obstacles ($s_{obs}$):
\begin{equation}
    \mathcal{B}(\hat{s}_{t+i}) = - \log  \|  \hat{s}_{t+i} - s_{obs} \|_2.
\end{equation} 
Every timestep, the controller applies the first action in $a^\star_{t:t+H}$ and then re-evaluates and updates the action sequence, ensuring that the agent adapts to changes in the environment. 

The effectiveness of the situationally-aware dynamics model described above hinges on its ability to accurately identify the current situation, $\psi$, and represent it as a symbol $\theta$. The following section presents our main theoretical contribution: a novel method for discovering and modeling these situations online and in an unsupervised fashion.

\subsection{Online and Unsupervised Latent Factor Representation Learning} \label{sec:online_unsupervised}

\begin{figure}
    \centering
    \includegraphics[width=\linewidth]{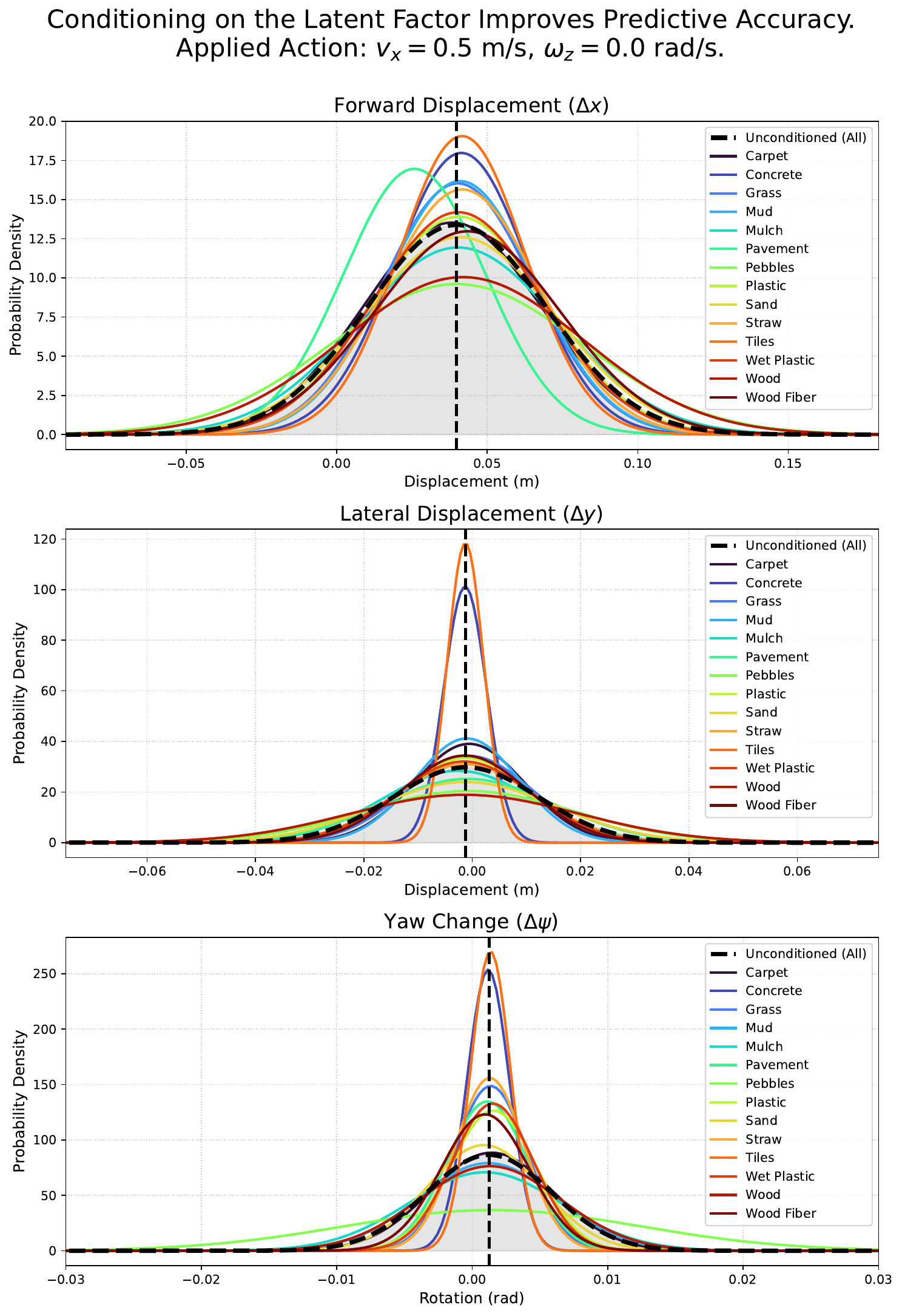}
    \caption{Motivating example showing the predictive benefit of conditioning on latent factors. Plotted are the marginal distributions of the change in state \(\Delta s\) conditioned on each terrain type from a real-world UGV dataset; colored curves correspond to terrain-specific distributions and the dashed black curve shows the unconditioned (aggregated) estimate. The unconditioned estimate smooths over distinct, terrain-dependent modes and therefore fails to capture terrain-specific transition dynamics, illustrating why latent-factor conditioning yields more accurate predictions. See Appendix~\ref{sec:appendix-latent-effect-char} for details.}
    \label{fig:1d_gaussians}
\end{figure}

Our online and unsupervised representation learning method seeks to capture variable and unknown ego- and environmental-factors. This approach models changes in the distribution governing the robot's transition dynamics as it interacts with the world. These changes encapsulate the influence of latent, unobserved factors that are challenging to predict or exhaustively anticipate. By leveraging this representation of the learned distribution of local dynamics, the robot refines its transition model, enabling more effective planning and control in downstream tasks.

\subsubsection{Online Transition Distribution Modeling:}

The transition distribution \( T(s, a, s')\) in a MDP defines a conditional probability distribution that specifies the likelihood of transitioning to a new state \( s' \) given the current state \( s \) and action \( a \). That is,
\begin{align} \label{eq:conditional-transition-prob}
    T(s, a, s') &= P(s' \mid s, a) \nonumber \\ &= \frac{P(S_t = s, A_t = a, S_{t+1} = s')}{P(S_t = s, A_t = a)}.
\end{align}
In particular, $P(s' \mid s, a) \propto P(S_t = s, A_t = a, S_{t+1} = s')$. However, as previously noted, in real-world applications, estimates of $P(s' \mid s, a)$ usually fail to accurately capture the true dynamics of the world. 
Generally, this shortfall arises from latent, unobserved factors—such as unmeasured variables or phenomena for which the robot lacks the necessary sensors—that significantly impact state transitions.
Thus, \textit{\textbf{our key insight}} is that by modeling online the joint transition distribution $P(S_t = s, A_t = a, S_{t+1} = s')$ that best explains the UDGP governing the current state transitions, we can account for the latent factors that influence the world dynamics and consequently improve performance on downstream tasks. 

To achieve this, we break the problem of online transition distribution modeling into determining local joint transition distributions. By ``local", we mean around the latest set or trajectory of states experienced by the agent. For example, a UGV that is stuck in a slippery area while climbing a mountain will estimate a different local joint transition distribution than the same robot when it is moving over a flat and smooth sidewalk. Figure \ref{fig:output-discrepancy} illustrates this concept, showing how the same action executed from an identical state can lead to vastly different outcomes whenever some important latent factors are not accounted for. Figure~\ref{fig:1d_gaussians} demonstrates this effect: it shows marginal transition distributions estimated from real-world UGV data collected while the vehicle executed the same action from comparable initial states. Colored curves correspond to different terrain types, and the dashed black curve denotes the unconditioned (aggregated) estimate, demonstrating how aggregation obscures terrain-specific transition dynamics.\footnote{Note that we do not claim that every terrain type requires a separate latent-factor representation. Instead, there is a trade-off between the number of modeled latent factors and their predictive benefit. For example, Figure~\ref{fig:terrain_latent_factor_effect_statistics} shows that the distributions for \textit{tiles} and \textit{concrete} are very similar, suggesting that a single representation may be sufficient in this case.} Crucially, such discrepancies are often better attributed to an incomplete state representation—e.g., the robot lacks sensors, such as a camera, to detect terrain types—rather than inherent aleatoric errors\footnote{While our approach does account for aleatoric uncertainty in the transition model, as done in traditional dynamics learning methods, our primary focus is on reducing epistemic uncertainty by identifying and accounting for the latent factors that currently influence the robot's dynamics.}.

We represent the {\em local dynamics} $T_L$ of the robot via a multivariate Normal distribution parameterized by its mean vector $\boldsymbol{\mu}$ and precision matrix $\Lambda$ (inverse of the covariance matrix $\Sigma$):
\begin{align} \label{eq:local-dynamics-model}
     T_L(\boldsymbol{x} | \boldsymbol{\mu}, \Lambda) = \frac{|\Lambda|^{1/2}}{(2 \pi)^{d/2}} \exp{ \left( -\frac{1}{2} (\boldsymbol{x} - \boldsymbol{\mu})^\top \Lambda (\boldsymbol{x} - \boldsymbol{\mu}) \right) },
\end{align}
we find this choice to be a computationally efficient and expressive approximation for predicting local dynamics. 
By focusing on the transitions within a neighborhood of states, the multivariate Normal distribution strikes an effective balance between representational power and computational feasibility, ensuring that the model remains practical for real-time applications.

Specifically, we use Eq. \eqref{eq:local-dynamics-model} to model the \textit{joint distribution} of the \textit{(state, action, next state)}-tuples, which we call the \textbf{\textit{situation}} of the robot from now on. Thus, $\boldsymbol{x} = [s_t; a_t; s_{t+1}]^\top \in \mathbb{R}^{d_x + d_a + d_x}$, which corresponds to concatenated state $s_t \in \mathbb{R}^{d_x}$, action $a_t \in \mathbb{R}^{d_a}$ and next state $s_{t+1} \in \mathbb{R}^{d_x}$ vectors. This will describe the situation of the robot, as we will be able to distinguish via the \textit{next state} components of $T_L$ whenever the same action under similar conditions (states) leads to different outcomes. The fact that the robot gets a markedly different outcome is what will let it know that an unknown factor is present. 
Furthermore, since $T_L$ encodes the effect of the unknown factor, we can use it to inform $T$ and adapt to the detected change in the transition dynamics.

However, to effectively model the local dynamics using Eq. \eqref{eq:local-dynamics-model}, we need to estimate $\boldsymbol{\eta} = \{\boldsymbol{\mu}, \Lambda\}$ online and simultaneously detect when a different $\boldsymbol{\eta}^\prime$ describes the current situation of the robot better. For this, we extend BOCD to a multivariate setting and apply it for online and unsupervised discovery of the distinct sets $\boldsymbol{\eta}_{\rho}$. We proceed to introduce our mutivariate extension of BOCD in Section \ref{sec:multivariate-bocd}. Then, in Section \ref{sec:sit-modeling-and-id} we show how we can leverage this results to model online the joint transition distribution of the robot's dynamics.

\subsubsection{Multivariate BOCD:} \label{sec:multivariate-bocd}

In many applications, data is often received in a streaming fashion, meaning that data points arrive sequentially, one at a time. A highly relevant example in the robotics domain is state estimates obtained across time. Each of these incoming data points can be considered as an observation from a statistical distribution, which can help us model and understand the underlying phenomena or UDGP.

We build our work on top of BOCD but with an assumption of {\em multivariate} streaming data points. 
This extension will enable online detection of changes in the distribution describing the underlying phenomena being observed. Additionally, it also allows us to simultaneously model the new phenomena.
Formally, the data points are known to be drawn from a multivariate normal distribution characterized by a set of parameters denoted by \( \boldsymbol{\eta} = \{\boldsymbol{\mu}, \Lambda\} \). 
As a result, at time $t = N$, we have collected a set of observations $\mathcal{D} = \{\boldsymbol{x}_i : \boldsymbol{x}_i \overset{\mathrm{iid}}{\sim} \mathcal{N}(\boldsymbol{\eta}) \}_{i=1}^N$. This means: 
\begin{align} \label{eq:sampling-model}
    p(\mathcal{D} | \boldsymbol{\mu}, \Lambda) &= \prod_{\substack{\boldsymbol{x}_i \in \mathcal{D}}} \mathcal{N} (\boldsymbol{x}_i | \boldsymbol{\mu}, \Lambda)  \\&= \left( \frac{|\Lambda|}{(2 \pi)^{d}} \right)^{n/2} \nonumber \\ & \quad\quad \cdot \exp{ \left( -\frac{1}{2} \sum_{\substack{i=1}}^n (\boldsymbol{x}_i - \boldsymbol{\mu})^\top \Lambda (\boldsymbol{x}_i - \boldsymbol{\mu}) \right) \nonumber}. 
\end{align}

We are interested in detecting the time $t^\prime > t$ when the incoming data no longer comes from $\mathcal{N}(\boldsymbol{\eta})$ and determine the new values of $\boldsymbol{\eta}^\prime$ that describe the new UDGP.

Since we want to understand when the parameters $\boldsymbol{\eta}$ of the UDGP have changed, we need to define a prior for them. For this, we propose the \textit{Normal-Wishart distribution}. 
This conjugate prior is particularly useful because it allows for closed form belief updates for $\boldsymbol{\mu} \in \mathbb{R}^d$ and $\Lambda \in \mathbb{R}^{d \times d}$ as more data becomes available. It is defined as follows \cite[]{murphy2007conjugateforms}:
\begin{align} \label{eq:normal-wishart-prior}
    p(\boldsymbol{\mu}, \Lambda) &=\mathcal{NW}(\boldsymbol{\mu}, \Lambda | \mu_0, \kappa_0, T_0, \nu_0) \\&= \mathcal{N}(\boldsymbol{\mu} | \mu_0, (\kappa_0 \Lambda)^{-1}) \mathcal{W}_{\nu_0}(\Lambda | T_0) \nonumber \\
    &= \frac{1}{Z} |\Lambda|^{1/2} \exp \left( -\frac{\kappa}{2} (\boldsymbol{\mu} - \mu_0)^T \Lambda (\boldsymbol{\mu} - \mu_0) \right) \nonumber \\
    & \quad \quad \quad \cdot |\Lambda|^{(\kappa - d - 1)/2} \exp \left( -\frac{1}{2} \operatorname{tr}(T^{-1} \Lambda) \right), \nonumber
\end{align}
where
\[
    Z = \left( \frac{\kappa}{2 \pi} \right)^{d/2} |T|^{\kappa / 2} 2^{d \kappa / 2} \Gamma_d(\kappa / 2)
\]
is a normalizing factor and $\operatorname{tr}(\cdot)$ is the trace of a matrix.

Essentially, the Normal-Wishart prior combines two distributions: it leverages the Normal distribution for the mean $\boldsymbol{\mu}$, incorporating prior knowledge $\boldsymbol{\mu}_0$ about the likely location of $\boldsymbol{\mu}$, and it employs the Wishart distribution to model the precision matrix $\Lambda$, acting as the prior on the inverse covariance. The Wishart distribution is flexible, encoding varying levels of confidence in the scale of the precision matrix through its degrees of freedom parameter and scale matrix, $T$. Parameters $\kappa_0$ and $\nu_0$ allow us to control the strength of our prior beliefs, with $\kappa_0$ scaling the covariance of the mean and $\nu_0$ defining the degrees of freedom of the Wishart distribution. This prior choice is satisfactory for our application since it assumes an unknown mean vector $\boldsymbol{\mu}$ and precision matrix $\Lambda$, which enables flexibility by allowing uncertainty in the covariance structure to be encoded in a principled way \cite[]{degroot1970optimal, murphy2007conjugateforms}. Consequently, it spares us from imposing restrictive assumptions on the covariance matrix, such as isotropy, homogeneity, or mean-dependence, making it highly suitable for modeling multivariate normal distributions with little assumptions about its parameters.

With the sampling model Eq. \eqref{eq:sampling-model} and prior Eq. \eqref{eq:normal-wishart-prior} determined, we proceed to make the connection to BOCD. Specifically, we will determine the closed form solutions for the growth probability Eq. \eqref{eq:growth-prob-def} and changepoint probability Eq. \eqref{eq:changepoint-prob-def}, from which we can recover the run length probability Eq. \eqref{eq:pos-run-length}. We begin by assuming that the prior on the changepoint probability is $P_{\mathtt{gap}}(g) \sim \mathtt{geometric}(\lambda)$. This makes the process memoryless and the hazard function constant at $H(\tau) = 1 / \lambda$, where $\lambda$ can be interpreted as the expected time between changepoints \cite[]{adams2007bayesian}.

\begin{restatable}[Growth and Changepoint Probabilities]{lemma}{sitproblemma} \label{lemma:sit-prob}
        Let $\boldsymbol{x}_t$ be the observation received at time $t$. Let $\boldsymbol{\eta}^{(r)} = \{\boldsymbol{\mu}_{i}^{(r)}, T_{i}^{(r)}, \nu_{i}^{(r)}, \kappa_{i}^{(r)}\}_{i=t_{cp}}^{t}$ be the parameters of the UDGP $(r)$ that started at time $t_{cp} < t$. Then, 
        \begin{itemize}[leftmargin=0.45cm]
            \item The \textbf{Growth Probability}, or probability that we stayed in the same UDGP, at time $t$ is: 
            \begin{align} \label{eq:permanence-prob}
                P(r_t = &~r_{t-1} + 1, \boldsymbol{x}_{t_{cp}:t}) = \\ &\frac{\lambda - 1}{\lambda} P(r_{t-1}, \boldsymbol{x}_{t_{cp}:t-1})  
                \mathcal{N} (\boldsymbol{x}_t | \boldsymbol{\mu}_{t-1}^{(r)}, \Lambda_{t-1}^{(r)}). \nonumber
            \end{align}
            \item The \textbf{Changepoint Probability}, or probability that we moved to a different UDGP, at time $t$ is:
            \begin{align} \label{eq:jump-prob}
                P(r_t = 0, &~\boldsymbol{x}_{t_{cp}:t}) = \\ & \frac{1}{\lambda} \sum_{\substack{\tau = 0}}^{r_{t-1}} P(\tau, \boldsymbol{x}_{t_{cp}:t-1}) 
                \mathcal{N} (\boldsymbol{x}_t | \boldsymbol{\mu}_{\tau}^{(r)}, \Lambda_{\tau}^{(r)}). \nonumber
            \end{align}
        \end{itemize}
\end{restatable}

\begin{proof} We assume that the situation the robot starts in begins at the exact moment the robot processes its first observation. Therefore, $P(r_0 = 0, \boldsymbol{x}_{0:0}) = 1$. This becomes the initial condition for the recursive computation key in the BOCD framework to find the distribution of the run length $P(r_t | \boldsymbol{x}_{1:t})$. Then, we can obtain the probability of staying in the current UDGP or jumping to a new UDGP as follows:

    \begin{itemize}[leftmargin=0.45cm]
        \item \textit{\textbf{Growth Probability:}} As defined in Eq. \eqref{eq:growth-prob-def}, the probability that the current run length increases is $P(r_t = r_{t-1}+1, \boldsymbol{x}_{1:t}) = P(r_{t-1}, \boldsymbol{x}_{1:t-1}) \pi_t^{(r)} (1 - H(r_{t-1}))$. For us, $\pi_t^{(r)}$ is the probability of the data given by Eq. \eqref{eq:sampling-model}, our sampling model. We also established before that $H(r_{t-1}) = 1 / \lambda$. Then, the growth probability at $t=1$ is:
        \begin{align*}
            P(r_1 = 1,&~\boldsymbol{x}_{0:1} = \{\boldsymbol{x}_1\}) \\&= P(r_0, \boldsymbol{x}_{0:0}) \cdot \pi_1^{(r)} \cdot (1 - H(0)) \nonumber \\
            &= 1 \cdot P(\{\boldsymbol{x}_1\} | \boldsymbol{\eta}_0) \cdot (1 - \frac{1}{\lambda}) \nonumber \\
            &= \frac{\lambda - 1}{\lambda} \mathcal{N}(\boldsymbol{x}_1 | \boldsymbol{\mu}_0, \Lambda_0). \nonumber
        \end{align*}

        And the growth probability at $t > 1$ is:
        \begin{align*}
            P(&r_t=r_{t-1}+1, \boldsymbol{x}_{0:t}) \\&= P(r_{t-1}, \boldsymbol{x}_{0:t-1}) \cdot \pi_t^{(r)} \cdot (1 - H(r_{t-1})) \nonumber \\
            &= P(r_{t-1}, \boldsymbol{x}_{0:t-1}) \cdot \prod_{\substack{r^\prime = 0}}^{\substack{r}_{t-1}} \mathcal{N} (\boldsymbol{x}_t | \boldsymbol{\mu}_{r^\prime}, \Lambda_{r^\prime}) \cdot (1 - \frac{1}{\lambda}) \nonumber \\
            &= \frac{\lambda - 1}{\lambda} \cdot \prod_{\substack{r^\prime = 0}}^{\substack{r}_{t-1}} \mathcal{N} (\boldsymbol{x}_t | \boldsymbol{\mu}_{r^\prime}, \Lambda_{r^\prime}) \cdot P(r_{t-1}, \boldsymbol{x}_{0:t-1}). \nonumber
        \end{align*}

        Therefore, the permanence probability at time $t$, given that we consider the time of the previous changepoint $t_{cp}$ to be zero, is:
        \begin{align}
            P(&r_t = r_{t-1} + 1, \boldsymbol{x}_{0:t}) \nonumber\\ &= \frac{\lambda - 1}{\lambda} \cdot \prod_{\substack{r^\prime = 0}}^{\substack{r}_{t-1}} \mathcal{N} (\boldsymbol{x}_t | \boldsymbol{\mu}_{r^\prime}, \Lambda_{r^\prime}) \cdot P(r_{t-1}, \boldsymbol{x}_{0:t-1}) \nonumber \\
            &= \frac{\lambda - 1}{\lambda} P(r_{t-1}, \boldsymbol{x}_{t_{cp}:t-1})  
                \mathcal{N} (\boldsymbol{x}_t | \boldsymbol{\mu}_{t-1}^{(r)}, \Lambda_{t-1}^{(r)}). \nonumber
        \end{align}

        \item \textit{\textbf{Changepoint Probability:}} It follows from the results above and the fact that the changepoint probability is defined in Eq. \eqref{eq:changepoint-prob-def} as $P(r_t=0, \boldsymbol{x}_{1:t}) = \sum_{\substack{r_{t-1}}} P(r_{t-1}, \boldsymbol{x}_{1:t-1}) \pi_t^{(r)} H(r_{t-1})$, that the probability of jumping to a new situation is:
        \begin{align*}
            P(&r_t = 0, \boldsymbol{x}_{t_{cp}:t}) \\ &= \frac{1}{\lambda}  \cdot \prod_{\substack{r^\prime = 0}}^{\substack{r}_{t-1}} \mathcal{N} (\boldsymbol{x}_t | \boldsymbol{\mu}_{r^\prime}, \Lambda_{r^\prime}) \cdot   \sum_{\substack{\tau = 0}}^{r_{t-1}} P(\tau, \boldsymbol{x}_{t_{cp}:t-1})  \\
            &= \frac{1}{\lambda} \sum_{\substack{\tau = 0}}^{r_{t-1}} P(\tau, \boldsymbol{x}_{t_{cp}:t-1}) 
                \mathcal{N} (\boldsymbol{x}_t | \boldsymbol{\mu}_{\tau}^{(r)}, \Lambda_{\tau}^{(r)}).
        \end{align*}
    \qed
    \end{itemize}
\end{proof}

Subsequently, as new observations arrive, we need to update the run length distribution to determine whether or not the UDGP has changed. Next, we describe in Lemma \ref{lemma:sit-params} the optimal closed-form equations that will enable us to do so efficiently.

\begin{restatable}[Online Distribution Parameters Learning]{lemma}{sitparamslemma} \label{lemma:sit-params}
Given the ordered set of situation parameters $\boldsymbol{\eta}^{(r)}$ and a new observation $\boldsymbol{x}_t$, the posterior parameters $\boldsymbol{\mu}_t^{(r)}, T_t^{(r)}, \nu_t^{(r)}, \kappa_t^{(r)}$ for the current UDGP are:
    \begin{align} 
        &\boldsymbol{\mu}_t^{(r)} = \frac{\kappa_{t-1}^{(r)} \boldsymbol{\mu}_{t-1}^{(r)} + \boldsymbol{x}_t}{\kappa_{t-1}^{(r)} + 1}; \label{eq:post-params-mu} \\~~ &T_t^{(r)} = T_{t-1}^{(r)} + S + \frac{\kappa_{t-1}^{(r)}}{\kappa_{t-1}^{(r)} + 1} \Delta; \label{eq:post-params-T}
        \\~~ &\nu_t^{(r)} = \nu_{t-1}^{(r)} + 1; \label{eq:post-params-nu} \\~~ &\kappa_t^{(r)} = \kappa_{t-1}^{(r)} + 1;    \label{eq:post-params-kappa} 
    \end{align}
where $S = \sum_{\substack{i=1}}^{\substack{t-1}} (\boldsymbol{x}_i - \boldsymbol{x}_t)(\boldsymbol{x}_i - \boldsymbol{x}_t)^{\top}$ and $\Delta = (\boldsymbol{\mu}_{t-1}^{(r)} - \boldsymbol{x}_t)(\boldsymbol{\mu}_{t-1}^{(r)} - \boldsymbol{x}_t)^{\top}$.
\end{restatable}

\begin{proof}
    The proof follows from the results of the posterior distribution parameters for the Normal-Wishart prior in \cite[Sec. 9.10]{degroot1970optimal}. Specifically, they show that for a set of samples $\{\boldsymbol{x}_i: \boldsymbol{x}_i \in \mathbb{R}^d\}_{i = 1}^n$, the posterior parameters for Eq. \eqref{eq:normal-wishart-prior}, given prior parameters $\{\boldsymbol{\mu}_0, \kappa_0, T_0, \nu_0\}$, are:
    \begin{align*}
        \boldsymbol{\mu}_n &= \frac{\kappa_0 \boldsymbol{\mu}_0 + n \overline{\boldsymbol{x}}}{\kappa_0 + n}, \\ T_n &= T_0 + S + \frac{\kappa_0 n}{\kappa_0 + n} \Delta, \\ \nu_n &= \nu_0 + n, \\ \kappa_n &= \kappa_0 + n, 
    \end{align*}
    where $S = \sum_{\substack{i=1}}^{\substack{n}} (\boldsymbol{x}_i - \overline{\boldsymbol{x}})(\boldsymbol{x}_i - \overline{\boldsymbol{x}})^{\top}$, $\Delta = (\boldsymbol{\mu}_0 - \overline{\boldsymbol{x}})(\boldsymbol{\mu}_0 - \overline{\boldsymbol{x}})^{\top}$ and $\overline{\boldsymbol{x}} = \frac{1}{n} \sum_{\substack{i = 1}}^{\substack{n}} \boldsymbol{x}_i$.
    
    Then, the derivation follows from the fact that $n=1$ for the observation $\boldsymbol{x}_t$ at time $t$, and the fact that the parameters $\boldsymbol{\eta}_{t-1}^{(r)} = \{\boldsymbol{\mu}_{t-1}^{(r)}, T_{t-1}^{(r)}, \nu_{t-1}^{(r)}, \kappa_{t-1}^{(r)}\}$ from the previous time step are the current prior. \qed
\end{proof}

Finally, if we would like to make predictions about future observations it is necessary to obtain the posterior distribution. We proceed to show in Lemma \ref{lemma:post-pred} that for our case, the posterior follows Student's $\mathtt{t}$ distribution.

\begin{restatable}[Observation Posterior]{lemma}{postpredlemma} \label{lemma:post-pred}

Building on the predicate established in Lemma \ref{lemma:sit-prob} and incorporating the updated parameters derived from Lemma \ref{lemma:sit-params}, the posterior predictive distribution for $\boldsymbol{\hat{x}}_{t+1}$ is given by:
\vspace{-1mm}
    \begin{align} \label{eq:post-pred}
        p(\boldsymbol{\hat{x}}_{t+1} &| \boldsymbol{x}_{t}) = \nonumber \\ &\sum_{\substack{\tau = 0}}^{r_{t-1}} \mathtt{t}_{\nu_n^{(r)} - d + 1} \left( \boldsymbol{\hat{x}}_{t+1} \middle| \boldsymbol{\mu}_{\tau}^{(r)}, \frac{T_{\tau}^{(r)} (\kappa_{\tau}^{(r)} + 1)}{\kappa_{\tau}^{(r)} (\nu_{\tau}^{(r)} - d + 1)} \right) \nonumber \\ &~~~~\cdot \frac{P(r_t = \tau, \boldsymbol{x}_{0:t})}{\sum_{\substack{\tau^\prime=0}}^{r_t} P(r_t = \tau^\prime, \boldsymbol{x}_{0:t})},
    \end{align}
    where \texttt{t} is the Student's \texttt{t} distribution.
\end{restatable}

\begin{proof}

From the BOCD framework, we assumed it is possible to compute the predictive distribution conditional on a given run length \cite[]{adams2007bayesian}: $$P(x_{t+1} | \boldsymbol{x}_{1:t}) = \sum_{r_t} P(x_{t+1} | r_t, \boldsymbol{x}_t^{(r)}) P(r_t | \boldsymbol{x}_{1:t}).$$
Additionally, $P(\boldsymbol{x}_{t+1} | \boldsymbol{x}_{1:t})$ depends only on the data of the current run $\boldsymbol{x}_t^{(r)}$. Thus, $P(\boldsymbol{x}_t | r_{t-1}, \boldsymbol{x}_{1:t}) = P(\boldsymbol{x}_t | r_{t-1}, \boldsymbol{x}_t^{(r)})$. Therefore, using the result from \cite{murphy2007conjugateforms} that says that the posterior predictive of the Normal-Wishart prior is:
\begin{align*}
p(\boldsymbol{\hat{x}}_{t+1} | \boldsymbol{x}_t) = \mathtt{t}_{\nu_n - d + 1} \left( \boldsymbol{\mu}_n, \frac{T_n (\kappa_n + 1)}{\kappa_n (\nu_n - d + 1)} \right),
\end{align*} 
where $\boldsymbol{\mu}_n, T_n, \kappa_n, \nu_n$ are the posterior parameters as derived in Lemma \ref{lemma:sit-params}, we arrive to Eq. \eqref{eq:post-pred}. \qed

\end{proof}

\subsubsection{Online Situation Modeling and Identification:} \label{sec:sit-modeling-and-id}

To determine when $T_L$ changes and given that we are modeling the system's situations via distribution parameters $\boldsymbol{\eta}_0 \coloneqq \{\boldsymbol{\mu}_0, \kappa_0, T_0, \nu_0\}$, we can use our previous multivariate extension to BOCD. This process is summarized in Algorithm \ref{alg:sit-modelling}. Concretely, we first determine if we have changed situations (Lemma \ref{lemma:sit-prob}), then learn the current situation's model parameters (Lemma \ref{lemma:sit-params}) and finally predict future transitions (Lemma \ref{lemma:post-pred}).

\begin{algorithm}[H]
    {\small \caption{Online Situation Modeling $\mathtt{(SM)}$}\label{alg:sit-modelling}
    \begin{algorithmic}[1]
    \State Get $\boldsymbol{x}_t = [s_{t-1}; a_{t-1}; s_{t}]^\top$
    \State Get the \textit{Growth Probability}, $j = Eq.~\eqref{eq:permanence-prob}$
    \State Get the \textit{Changepoint Probability}, $k = Eq.~\eqref{eq:jump-prob}$
    \State \textit{Update} parameters $\boldsymbol{\eta}_{t-1}$ according to $Eqs.~\eqref{eq:post-params-mu}-\eqref{eq:post-params-kappa}$
    \State $\mathtt{situation\_change} = \mathbbm{1}(k > j)$
    \State \textit{return} $\mathtt{situation\_change}, \boldsymbol{\eta}_t$
    \end{algorithmic}
    }
\end{algorithm}

To estimate $T_L$ (Eq. \ref{eq:local-dynamics-model}) online we use the Maximum Likelihood Estimate (MLE) of its parameters given the data from the current run $\boldsymbol{x}^{(r)}_t$. The estimated parameters are added to our library of situations, $\boldsymbol{\Psi}$. The data for the current run, $\boldsymbol{x}^{(r)}_t$, is collected in a variable-sized window; incoming state-action-state tuples are added to $\boldsymbol{x}^{(r)}_t$ until a changepoint is detected, at which point $\boldsymbol{x}^{(r)}_t$ is reset to the empty set.

To identify which situation $\psi \in \boldsymbol{\Psi}$ the robot is currently experiencing we employ the average Mahalanobis distance \cite[]{mahalanobis1930test}:
\begin{align} \label{eq:situation-distance}
    Mah(&\psi) \coloneqq \\ \nonumber &\left( \frac{1}{\mid \boldsymbol{x}^{(r)}_t \mid} \sum_{\boldsymbol{x}_i \in \boldsymbol{x}^{(r)}_t} \sqrt{(\mathbf{x}_i - \boldsymbol{\mu}_{\psi})^\top \Lambda_{\psi} (\mathbf{x}_i - \boldsymbol{\mu}_{\psi})} \right), 
\end{align}
where $\psi = (\boldsymbol{\mu}_{\psi}, \Lambda_{\psi})$. In this way we can estimate how far each point in the current run data $\boldsymbol{x}^{(r)}_t$ is to the distributions describing the detected local transition models. Then, using a threshold $\tau$ we determine if the distribution of the closest situation is appropriate to represent the current local dynamics of the robot. If it is not, we create a new situation $\psi = (\boldsymbol{\mu}, \Lambda)$ that we add to $\boldsymbol{\Psi}$ during training. At test time, we use the closest experienced situation in $\boldsymbol{\Psi}$ to the one the robot is currently experiencing. The full identification process is detailed in Algorithm \ref{alg:sit-identification}.
 
\begin{algorithm}[H]
{
    \small
    \caption{Online Situation Identification $\mathtt{(SI)}$}\label{alg:sit-identification}
    \begin{algorithmic}[1]
        \State Get $\boldsymbol{x}_t = [s_{t-1}; a_{t-1}; s_{t}]^\top$
        \State Get data from current run $\boldsymbol{x}^{(r)}_t$
        \State $\mathtt{situation\_change}, \boldsymbol{\eta}_t \gets \mathtt{SM}(\boldsymbol{x}_t, \boldsymbol{x}^{(r)}_t, \boldsymbol{\eta}_{t-1})$ \Comment{Alg. \ref{alg:sit-modelling}}
        \State $\psi^* \gets \argmin_{\psi \in \boldsymbol{\Psi}} Mah(\psi)$  \Comment{$Mah(\psi) =$ Eq. \eqref{eq:situation-distance}}
        \If{$\mathtt{situation\_change}$}

            \If{ $Mah(\psi^*) \geq \tau$ and $\mathtt{training}$}
            \State Add $(\boldsymbol{\mu}, \Lambda) \gets \text{MLE}(\boldsymbol{x}_t^{(r)})$ to $\boldsymbol{\Psi}$  \Comment{Create Situation}
            \State $\psi^* \gets (\boldsymbol{\mu}, \Lambda)$
            \EndIf
            
            \State $\boldsymbol{x}_t^{(r)} \gets \emptyset$    \Comment{Variable-sized window of past states}
        \Else
            \State Add $\boldsymbol{x}_t$ to $\boldsymbol{x}_t^{(r)}$
        \EndIf
        \State \textit{return} $\psi^*$

    \end{algorithmic}
}
\end{algorithm}

\subsubsection{Representing the Local Transition Model:}

To provide the dynamics model $\mathcal{M}_m$ with the current situation information, we need to map each situation $\psi \in \boldsymbol{\Psi}$ to a symbolic representation $\theta \in \Theta$. For this, we define a function $\boldsymbol{\Theta} : \boldsymbol{\Psi} \rightarrow \Theta$. We need this function to distinguish between multiple instances of the local distribution $T_L$ (Eq. \ref{eq:local-dynamics-model}), all of which belong to a family of distributions that share a common form but vary in their parameters. 

We find that the Moment Generating Function (MGF) provides an useful tool for this purpose, as it uniquely characterizes a distribution \cite[Theorem 1]{curtiss1942note}. In general, for a multivariate Gaussian random vector \( \boldsymbol{X} \sim \mathcal{N}(\boldsymbol{\mu}, \Sigma) \), the MGF is defined as 
\begin{equation} \label{eq:mgf-multigaussian}
M_{\boldsymbol{X}}(\boldsymbol{t}) = \mathbb{E}[e^{\boldsymbol{t}^\top \boldsymbol{X}}] = e^{\boldsymbol{t}^\top \boldsymbol{\mu} + \frac{1}{2} \boldsymbol{t}^\top \Sigma \boldsymbol{t}}.
\end{equation}

Thus, to symbolically represent the current situation $\psi$ we use its MGF, with a fixed, nonzero $\boldsymbol{t} \in \mathbb{R}^d$, as follows: 
\begin{equation} \label{eq:situation_representation}
    \boldsymbol{\Theta}_{\boldsymbol{t}}(\psi) = \log M_{\boldsymbol{X}}(\boldsymbol{t}; \psi) = \boldsymbol{t}^\top \boldsymbol{\mu}_{\psi} + \frac{1}{2} \boldsymbol{t}^\top \Lambda^{-1}_{\psi} \boldsymbol{t},
\end{equation} 
where we apply the $\log$ function for numerical convenience. This mapping provides the concrete value for $\theta$ used by the high-level dynamics model in Eq. \eqref{eq:transition_func_ghpmdp}. Although it could happen that for a given $\boldsymbol{t}$ the representations of two situations would end up being equal, this is highly unlikely and we did not experience any such collisions. Then, whenever $\mathcal{M}_m$ is being called, we augment the robot's state with the current situation's symbol $\theta$. Note that the current situation is constantly updated as the output from Algorithm \ref{alg:sit-identification}, which is executed every time a new observation $\boldsymbol{x}_t$ is obtained.

\subsubsection{Additional Details:}

To prevent numerical errors while computing the precision matrix during situation creation (Line 7 in Algorithm \ref{alg:sit-identification}), we also employ an additional hyperparameter describing the minimum support needed. As a rule of thumb, we set it slightly above the dimensionality of the situation space. Additionally, to compute the situation symbol (Eq. \ref{eq:situation_representation}) we use a constant value of $\mathbf{t} = \mathbf{0.1}$.

\section{Results: \textit{Unmanned Ground Vehicle}} \label{sec:results_ugv}

In this section, we evaluate the effectiveness of our situationally-aware dynamics model in addressing the following key questions:  

\begin{enumerate}

    \item \textbf{Does our situational awareness module enable effective online hidden state representation learning?} 
    This includes our multivariate BOCD method for situation change detection together with the proposed situation identification procedure.
    We compare its performance against other approaches that learn auxiliary representations alongside the dynamics learning process.  

    \item \textbf{Does our method lead to faster dynamics learning?}  
    We benchmark it against competitive model-based reinforcement learning (MBRL) methods and include model-free reinforcement learning (MFRL) baselines to highlight the efficiency gains of model-based approaches.  

    \item \textbf{Can our approach produce safe action policies in unstructured environments under partial state information?}  
    We analyze its performance in scenarios where partial state information introduces uncertainty in motion outcomes and environmental characteristics.  
\end{enumerate} 

We begin by introducing the evaluated scenarios in Section \ref{subsec:eval-scenarios-ugv} and highlight how our approach overcomes the challenges faced by traditional approaches when used in such settings. 
We then describe the baselines in Section \ref{sec:baselines-ugv}. In Section \ref{sec:exp-setup-ugv} we cover the evaluation setup. We detail the state and action space as well as the tools for simulation and real world evaluation. In Section \ref{subsec:results-training-ugv} we show that learning the situationally-aware dynamics model results in faster learning of better policies when compared to context-aware, model-based and model-free baselines. In Section \ref{sec:ablation-studies-ugv} we present a series of ablations to better understand the contribution of the method's components as well as the effect in the dynamics prediction with the learned hidden state representation. Finally, in Section \ref{subsec:results-unstructured-terrain} we carry out extensive experiments in simulation and the real world to conclude that our model augmented with the hidden state representation achieves better performance and safer policies in the unstructured terrain navigation task. Notably, we also summarize some emerging behavioral patterns that we observed during the real-world experiments.

\begin{figure}
    \centering
    \includegraphics[width=0.92\linewidth]{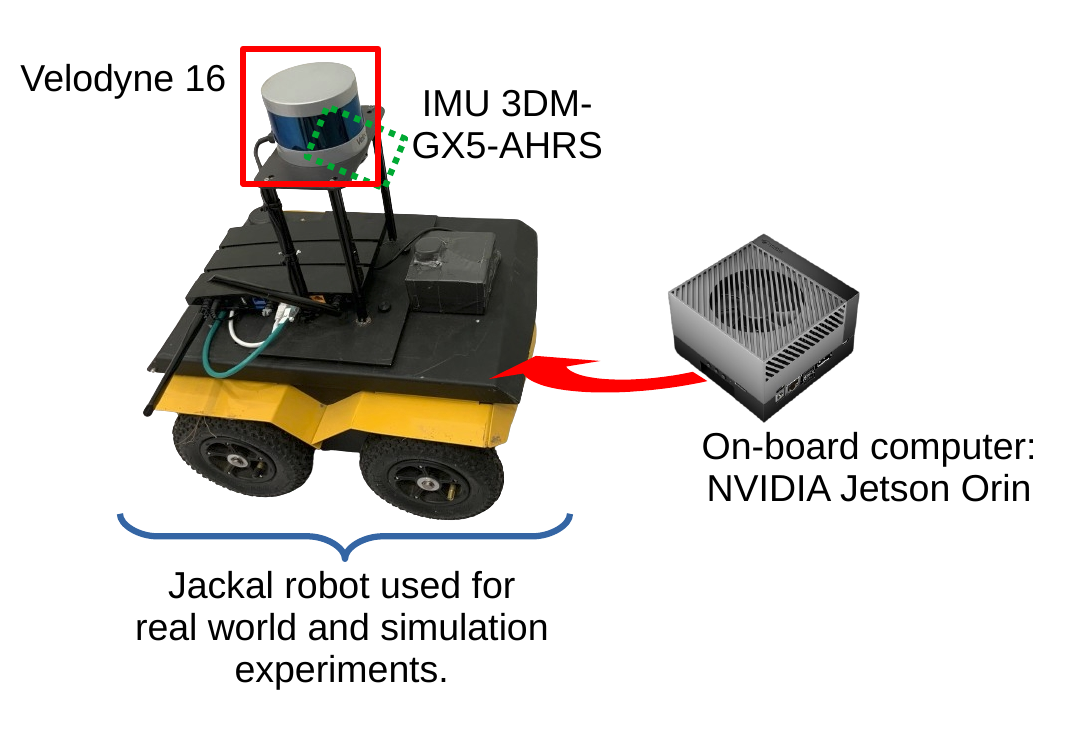}
    \caption{Hardware setup used during real world experiments.}
    \label{fig:hardware-setup}
\end{figure}

\subsection{Evaluation Scenarios} \label{subsec:eval-scenarios-ugv}

To demonstrate real-world effectiveness, we validate our approach on the challenging task of unstructured terrain navigation, where unknown and unpredictable terrain characteristics can heavily influence the robot's motion. 
We perform extensive evaluations in both simulation and real world. 

\subsubsection{Simulation:}  \label{subsubsec:sim-eval-scenarios-ugv}

We use the Gazebo simulator \cite[]{Koenig-2004-394}. Specifically, we use Clearpath's Inspection World \cite[]{inspection_world}. 
It combines uneven surfaces, slopes, and varying friction properties. These conditions mimic real-world complexities, requiring the robot to adapt dynamically to changes in stability and traction while maintaining precise control, making it an ideal testbed for training and assessing our method.

Furthermore, we create four additional worlds with increasing levels of difficulty. Concretely, these are randomly generated environments created using height maps sampled using Perlin noise \cite[]{perlin1985image} that provide a complex and realistic testing ground. These environments feature highly variable and irregular terrain patterns, including sudden changes in elevation, unpredictable slopes, and patches of different roughness and friction. The use of Perlin noise ensures natural-looking, non-repetitive landscapes that challenge the robot’s ability to adapt to dynamic and uneven surfaces. This makes them an effective way to assess the robustness and adaptability of our method in navigating real-world-like conditions. The generated height maps serve as input to the \texttt{gazebo\_terrain} package for automatic Gazebo model terrain generation \cite[]{gazebo_terrain}. The parameters used to generate the height maps are presented in Table \ref{tab:simulation-environments-params}. 

\begin{table} 
    \centering
    \small
    \begin{tabular}{ccccc}
    \toprule
    \textbf{Difficulty} & \textbf{Scale} & \textbf{Octaves} & \textbf{Persistence} & \textbf{Lacunarity} \\
    \midrule

        \textbf{1} & 50 & 4 & 0.4 & 1.5 \\
        \textbf{2} & 45 & 4 & 0.43 & 1.6 \\
        \textbf{3} & 50 & 5 & 0.6 & 7 \\
        \textbf{4} & 45 & 7 & 0.65 & 9 \\

    \bottomrule
    \end{tabular}
    \captionof{table}{Parameters for randomly generated simulation terrains with increasing levels of difficulty. \textit{Scale} adjusts slope steepness, with larger values producing smoother hills. \textit{Octaves} define terrain detail, while \textit{persistence} affects roughness by reducing octave amplitude. \textit{Lacunarity} controls octave frequency, with higher values adding ruggedness.}
    \label{tab:simulation-environments-params}
\end{table}

\subsubsection{Real World:}

\begin{figure*}
    \centering
    \includegraphics[width=0.61\linewidth]{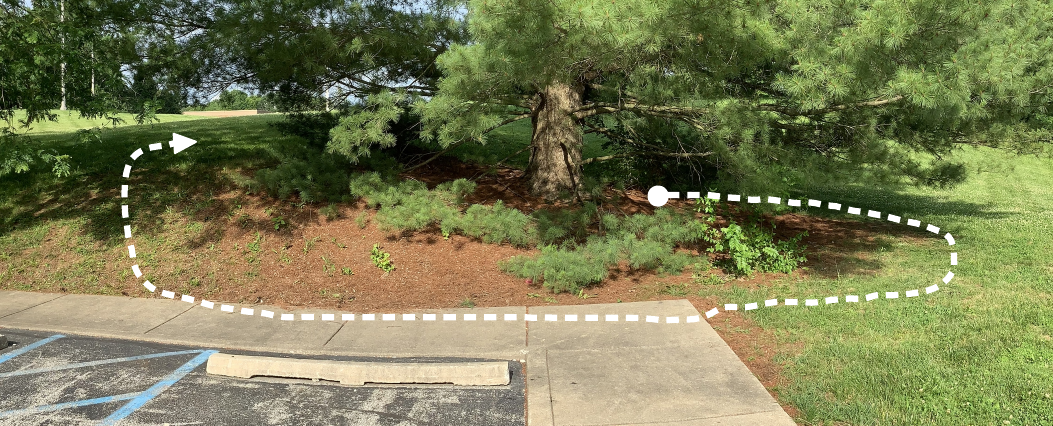}
    \includegraphics[width=0.368\linewidth]{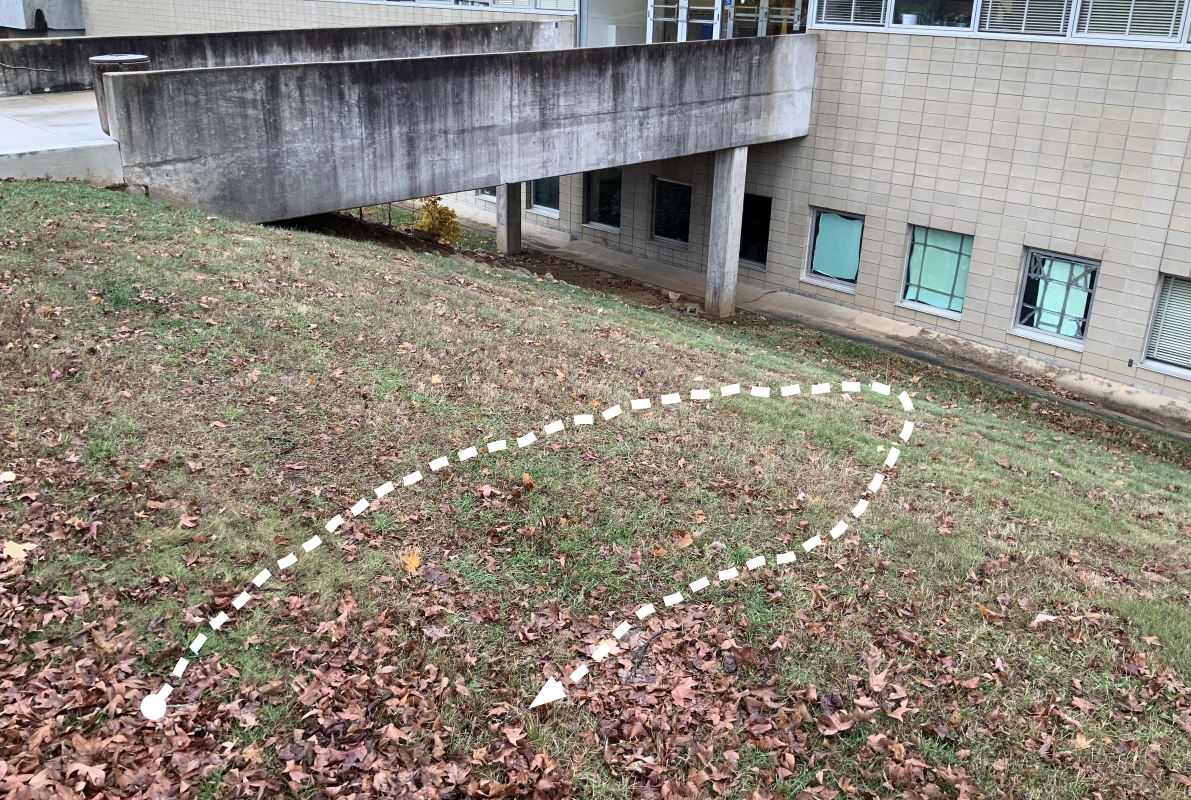}
    \caption{\textit{Some} real-world scenarios used to evaluate and compare the proposed dynamics model against the competitive baselines. \textit{(Left)} Navigation across diverse terrain types, including mulch, grass, concrete, and mud, each with distinct surface properties and geometries. The final section involves a short but steep climb on mulch and grass, which presents a significant challenge due to reduced traction. \textit{(Right)} Traversing a highly demanding terrain with steep ascents and descents, compounded by numerous hidden bumps and irregularities. These features make the environment difficult to navigate safely and efficiently. Notably, our method is the only one capable of successfully completing this scenario.} 
    \label{fig:task1}
\end{figure*}

Figure \ref{fig:task1} contains two examples of the diverse real-world environments used to evaluate the situationally-aware dynamics model against the baselines, with additional scenarios detailed in Figure \ref{fig:nav_summary} and showcased in the accompanying videos. 

Given that the observation space is limited to proprioceptive odometry estimates, as described in Section \ref{sec:exp-setup-ugv}, critical environmental factors such as terrain type, surface irregularities, and characteristics like friction are not explicitly available to the model. Instead, \textit{these factors will need to be implicitly encoded within the hidden state representation.} This representation serves to capture the unobserved dynamics that govern the robot's interaction with its environment, allowing the model to adapt its behavior to varying conditions without direct access to external sensory data.

In the first scenario, Figure \ref{fig:task1} \textit{(left)}, the robot must navigate across diverse terrain types, transitioning seamlessly between mulch, grass, concrete, and mud. Each terrain introduces distinct surface properties, such as varying friction levels, compaction, and uneven geometries. The environment is further complicated by a mix of downhill, flat, and uphill sections, with the final steep and slippery mulch-covered climb posing a significant challenge. This is caused by the mulch's composition of small, loose pieces that shift and move under pressure, similar to the behavior of sand, making it difficult for the robot to maintain traction. This scenario highlights the model's ability to adapt to abrupt changes in terrain dynamics and optimize transitions efficiently.

In the second scenario, Figure \ref{fig:task1} \textit{(right)}, the robot encounters an even more challenging terrain characterized by a very steep descent and ascent, interspersed with numerous bumps and irregularities that hinder the robot's motion due to the hardware's low profile. The difficulty of safely and effectively traversing these features highlights the importance of a model capable of representing and predicting dynamics in complex, unstructured environments. By leveraging situational awareness to represent the hidden factors affecting the state evolution, we aim to show that the model demonstrates its capacity to synthesize adaptive control policies, making these demanding scenarios navigable. 

\subsection{Baselines} \label{sec:baselines-ugv}

This work is interested in online hidden state representation learning. Additionally, we show that this can be leveraged for improved dynamics modeling. Therefore, we select learning-based and physics-based baselines that can be applied to the unstructured terrain navigation problem. This problem let us verify that the resulting situationally-aware dynamics model can adapt and devise a motion plan that accounts for the motion uncertainty in challenging scenarios, leading to safe and effective navigation policies. That is, we are interested in highlighting how adding situational awareness to the robot's model enables faster learning and safer navigation behaviors, which we illustrate in the challenging unstructured terrain navigation task.

Our learning-based and physics-based robot dynamics model baselines, with and without hidden state representations, are:

\begin{enumerate} 
    \item \textit{CaDM:} Context-aware Dynamics Model, which aims to learn a ``global model that can generalize across different dynamics" \cite[]{lee2020context}. This is the closest baseline to our work, as they also learn a model that can detect and adapt to changes in the underlying dynamics of the agent. Additionally, although they need to learn three models (context encoder, forward and backward dynamics), they do so in parallel instead of relying on a multi-stage curriculum as other recent methods. We use an author-provided and publicly available implementation.
    \item \textit{PE+LSTM:} We also aim to learn the hidden state representation using an LSTM network that takes a fixed-sized window of the last 50 state transitions. We combine the recurrent network's output with the input to a probabilistic ensemble (PE), analogous to the one used for our situationally-aware dynamics model. Note that our method does have a variable-sized window, which allows fast disposal of past uninformative states, which is not possible when using the LSTM as we do not have a readily available way of determining a variable window size.
    \item \textit{TD-MPC2:} This is one of the leading MBRL methods at the time of writing \cite[]{hansen2023td}. It is characterized for learning a latent-space dynamics model, together with an observation encoder, transition reward predictor, terminal value predictor and a return-maximizing action policy. We use an author-provided and publicly available implementation.
    \item \textit{PETS:} This method proposed by \cite{chua2018deep} is a strong state-space dynamics learning method based on a probabilistic ensemble architecture, which in fact inspires our choice of architecture. Our implementation is based on a derivation from MBRL-Lib \cite[]{Pineda2021MBRL}.
    \item \textit{CDL-CMI:} This is a causally-inspired, model-based RL method that aims to increase predictive robustness by avoiding unnecessary dependencies between state variables and actions \cite[]{wang2022causal}. They also learn a state abstraction for better out-of-distribution predictive performance and improved learning efficiency. We use an author-provided and publicly available implementation.
    \item \textit{Model-Free}: We train model-free policies with PPO \cite[]{schulman2017proximal}, SAC \cite[]{haarnoja2018soft} and TD3 \cite[]{fujimoto2018addressing}. This is done to present a well-rounded coverage of general RL methods and to get an idea of how recent competitive methods that also rely on auxiliary representations but with model-free policies would compare. 
    \item \textit{Dubins:} This is a physics-based model \cite[]{dubins-dyns} that has been used recently as an unstructured terrain navigation baseline in multiple works \cite[]{siva2024self, yu2024adaptive}.
\end{enumerate}

\subsection{Experimental Setup} \label{sec:exp-setup-ugv}

\subsubsection{State and Action Space:} The state space $s \in \mathbb{R}^{12}$ corresponds to: 
\begin{equation} \label{eq:state-space}
s = \begin{bmatrix} \mathbf{p} & \mathbf{e} & \mathbf{v} & \mathbf{\Omega} \end{bmatrix}^\top,
\end{equation}
where $\mathbf{p} = [x, y, z]\in \mathbb{R}^3$ is the robot's position, 
$\mathbf{e} = [\phi, \zeta, \gamma] \in \mathbb{R}^3$ are the roll, pitch, and yaw angles (orientation), $\mathbf{v} = [v_x, v_y, v_z] \in \mathbb{R}^3$ are the linear velocity components, and $\mathbf{\Omega} = [\omega_x, \omega_y, \omega_z] \in \mathbb{R}^3$ are the angular velocity components.
We limit the state space to the elements in Eq. \eqref{eq:state-space} in order to give room to our hidden state representation method to encode remaining necessary but unobserved ego- and world-characteristics.

The action space $a \in \mathbb{R}^2$ is,
\begin{equation} \label{eq:action-space}
    a = \begin{bmatrix}
    v_{\text{cmd}} & \omega_{\text{cmd}}
    \end{bmatrix}^\top,
\end{equation}
where $v_{\text{cmd}} \in [-1, 1]$ and $\omega_{\text{cmd}} \in [-\frac{\pi}{2}, \frac{\pi}{2}]$ denote the commanded linear and angular velocities, respectively.

\subsubsection{Situation Space:} As defined in Section \ref{sec:method}, the \textit{situation} of the robot corresponds to the joint transition distribution. Therefore, the situation space $\mathbf{x} \in \mathbb{R}^{26}$ corresponds to the concatenation of state and action vectors:
\begin{equation} \label{eq:situation-space-ugv}
    \mathbf{x}_t = \begin{bmatrix} s_{t-1}; & a_{t-1}; & s_t \end{bmatrix}^\top,
\end{equation}
where the state vectors ($s_{t-1}$ and $s_t$) undergo preprocessing to \textit{(i)} remove position information (as the dynamics should be invariant to map-specific information) and \textit{(ii)} replace orientation information (roll, pitch and yaw) with their sine and cosine transformations which enable the algorithm to uniquely identify them while avoiding discontinuities when the angles are wrapped around. 

\subsubsection{Software:} \label{sec:setup-software}
All code was developed using ROS Noetic \cite[]{Quigley09}, Python 3.10.11 and PyTorch 2.3.0+cu121 \cite[]{NEURIPS2019_9015}. For the sampling-based controllers we use the \texttt{MBRL-Lib} package \cite[]{Pineda2021MBRL}. In the real-world, for localization we used FasterLIO with IMU preintegration on manifold \cite[]{9718203, yu2023fast} with the Velodyne 16 LiDAR.

\subsubsection{Hardware:} Figure \ref{fig:hardware-setup} shows the hardware used to test our proposed situationally-aware dynamics model and the baselines in the real world. We ran our code on-board the robot with an NVIDIA Jetson Orin. We used Velodyne 16 LiDAR and IMU 3DM-GX5-AHRS for pose, orientation, linear and angular velocity estimation to obtain the robot's state as formulated in Eq. \eqref{eq:state-space}.

\subsection{Situationally-Aware Dynamics Learning} \label{subsec:results-training-ugv}

\begin{figure}
 \centering
  \includegraphics[width=\linewidth]{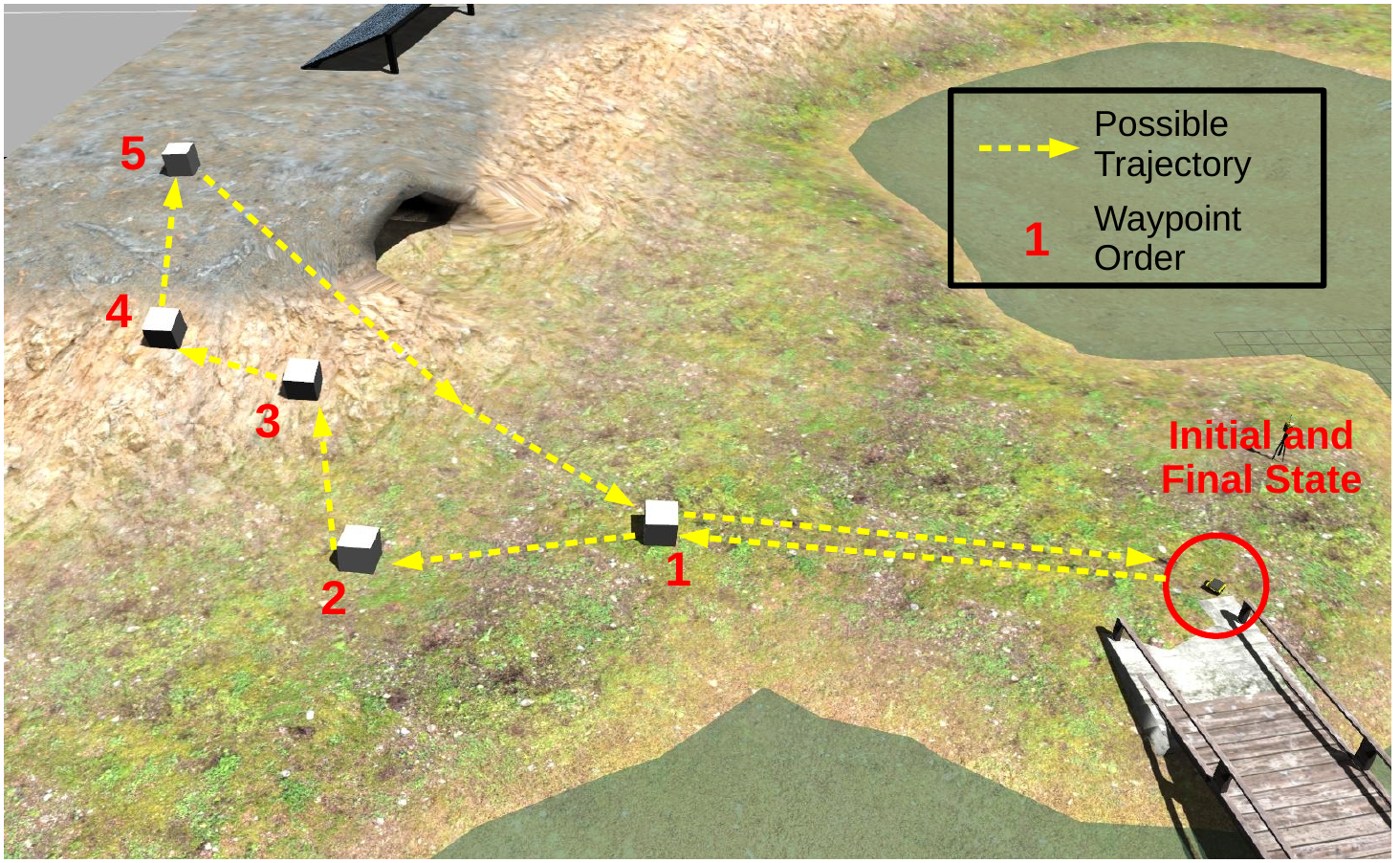}
    \captionof{figure}{Training task in the {\em Inspection World Environment} \cite[]{inspection_world}. The region around waypoint \#3 is slippery and difficult to reach.}
    \label{fig:task-overview}
\end{figure}

The situational awareness module processes the 26-dimensional situation observation (Eq. \ref{eq:situation-space-ugv}) as they arrive, in order to return the current representation of the hidden state $\theta \in \Theta$ symbolizing the local dynamics of the robot. This process follows the procedure outlined in Algorithm \ref{alg:sit-identification}.
Then, the learned dynamics model $\mathcal{M}_m$ takes the current situation $\theta_t$ together with the robot's orientation and commanded linear and angular velocities in order to predict the change in position and orientation.

The model $\mathcal{M}_m$ is composed of $m=5$ probabilistic neural networks, with three dense layers with 200 neurons each and LeakyReLU non-linearities in between. The output is the mean and log-variance of a Gaussian distribution. Inputs and outputs are normalized using data statistics computed during training. At inference time, predictions are de-normalized with respect to the data statistics obtained during training. We use the Adam optimizer with default hyperparameters \cite[]{kingma2014adam}. The batch size is 256 with an early-stopping threshold of 1e-3. The model $\mathcal{M}_m$ starts with randomly initialized parameters and is trained at the beginning of each episode with the updated replay buffer $\mathcal{R}$, analogous to the MBRL procedure shown in Algorithm \ref{alg:mbrl}. That is, we follow the \textit{active} training procedure described in Section \ref{sec:training_testing}. A summary of the hyperparameters for the situational awareness module (online hidden state representation learning) and the controller is provided in Table \ref{tab:combined_hyperparams_ugv}.

\begin{table} 
    \centering
    \small
    \begin{tabular}{llr}
    \toprule
    \textbf{Algorithm} & \textbf{Hyperparameter} & \multicolumn{1}{c}{\textbf{Value}} \\
    \midrule
    \textbf{Situational Awareness} & & \\
    & Hazard Rate $\lambda$ & $60$ \\
    & Threshold $\tau$ & $80$ \\
    & Situation Support & $60$ \\
    & & \\
    \midrule
    \multicolumn{3}{l}{\textbf{MPPI}} \\
    & Number of iterations & $2$ \\
    & Population size & $1500$ \\
    & Reward scaling $\gamma$ & $0.9$ \\
    & Noise scaling $\sigma$ & $0.4$ \\
    & Correlation term $\beta$ & $0.4$ \\
    \bottomrule
    \end{tabular}
    \caption{Hyperparameters used for Situational Awareness and MPPI in the UGV unstructured terrain navigation task.}
    \label{tab:combined_hyperparams_ugv}
\end{table}

\subsubsection{Training Results:} 
To train the model and baselines we design a \textit{sparse} waypoint mission in the Inspection World simulator shown in Figure \ref{fig:task-overview}. However, the region around waypoint 3 is steep and makes the robot slip whenever it tries to go there directly. This makes it necessary to come up with a non-trivial plan to reach it (e.g., approach it diagonally or from above). To trade-off exploration and exploitation, an episode ends if the robot takes longer than 90 seconds to reach the next waypoint. We train all methods sequentially on a single simulator. 

Figure \ref{fig:train-performance} shows that \textit{the situationally-aware dynamics enable faster learning and higher average reward} (number of waypoints reached). 
Additionally, it is possible to see that on average our model is able to complete the task faster and, more significantly, \textit{it is the only one able to complete the task within 50 episodes}.

\begin{figure*}
    \includegraphics[width=\linewidth]{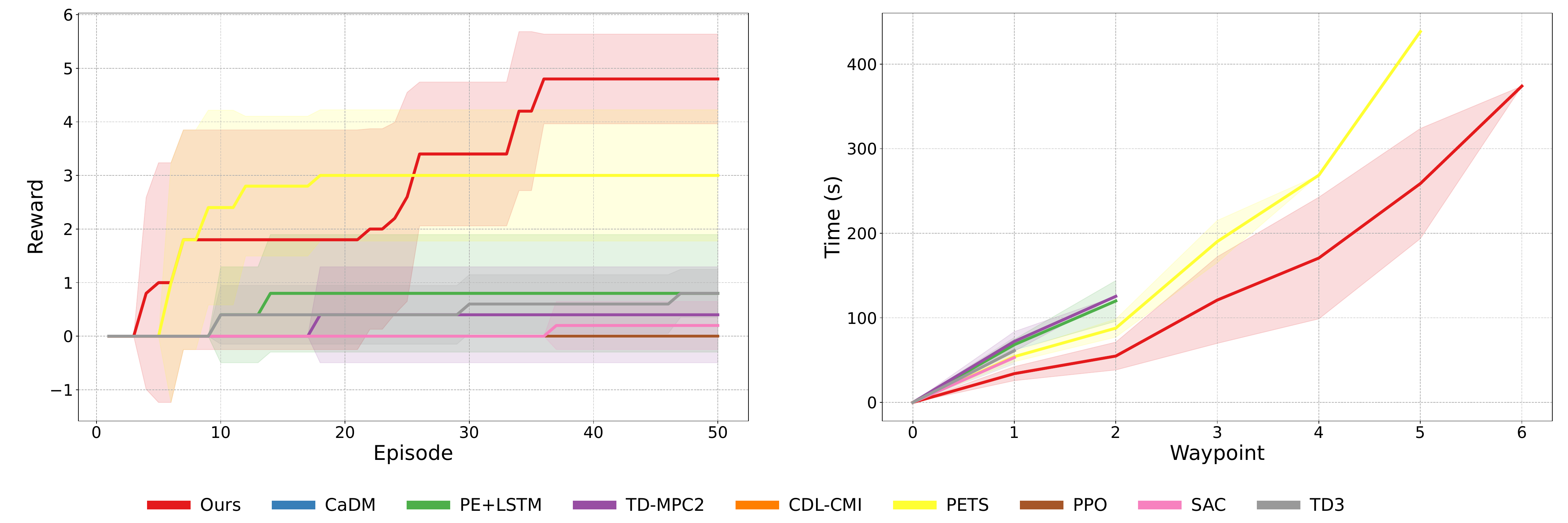}
    \caption{Dynamics Learning Results. Our model achieves a high reward faster and is able to complete its tasks in a shorter time frame. \textit{(Left)} Training task reward. We compare against model-based and model-free RL methods. \textit{(Right)} Mean $\pm$ standard deviation of the time that it takes the robot to reach the waypoints in the task. Each experiment was repeated five times.}
    \label{fig:train-performance}
\end{figure*}

\begin{figure}
 \centering
  \includegraphics[width=\linewidth]{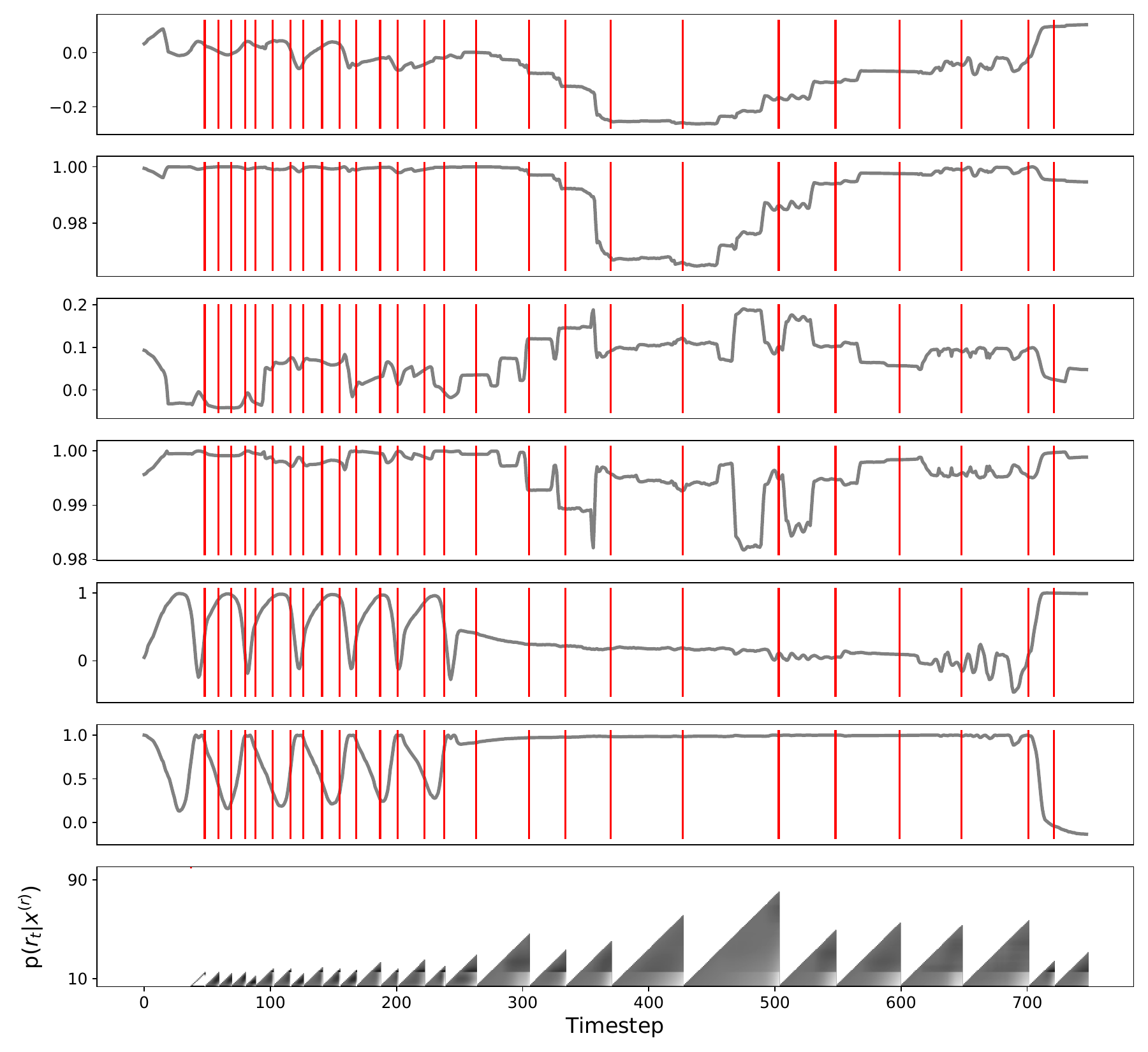}
    \captionof{figure}{\small{Observations received by the proposed situation identification algorithm during a mission in the Inspection World. The red vertical lines mark the time when we predict a situation change happened. For ease of visualization we only present the first six features of the 26-dimensional observation, the full figure is available in Appendix \ref{sec:appendix-inspection_world_sa_results}.
    Each row shows one variable's evolution over time.
    The last row shows the run length probabilities indicating how likely it is that the robot stayed in the same situation (run length grows) or a situation change happened (run length becomes zero).
    }}
    \label{fig:reduced-observation-history}
\end{figure}

\begin{figure}
    \centering
    \includegraphics[width=\linewidth]{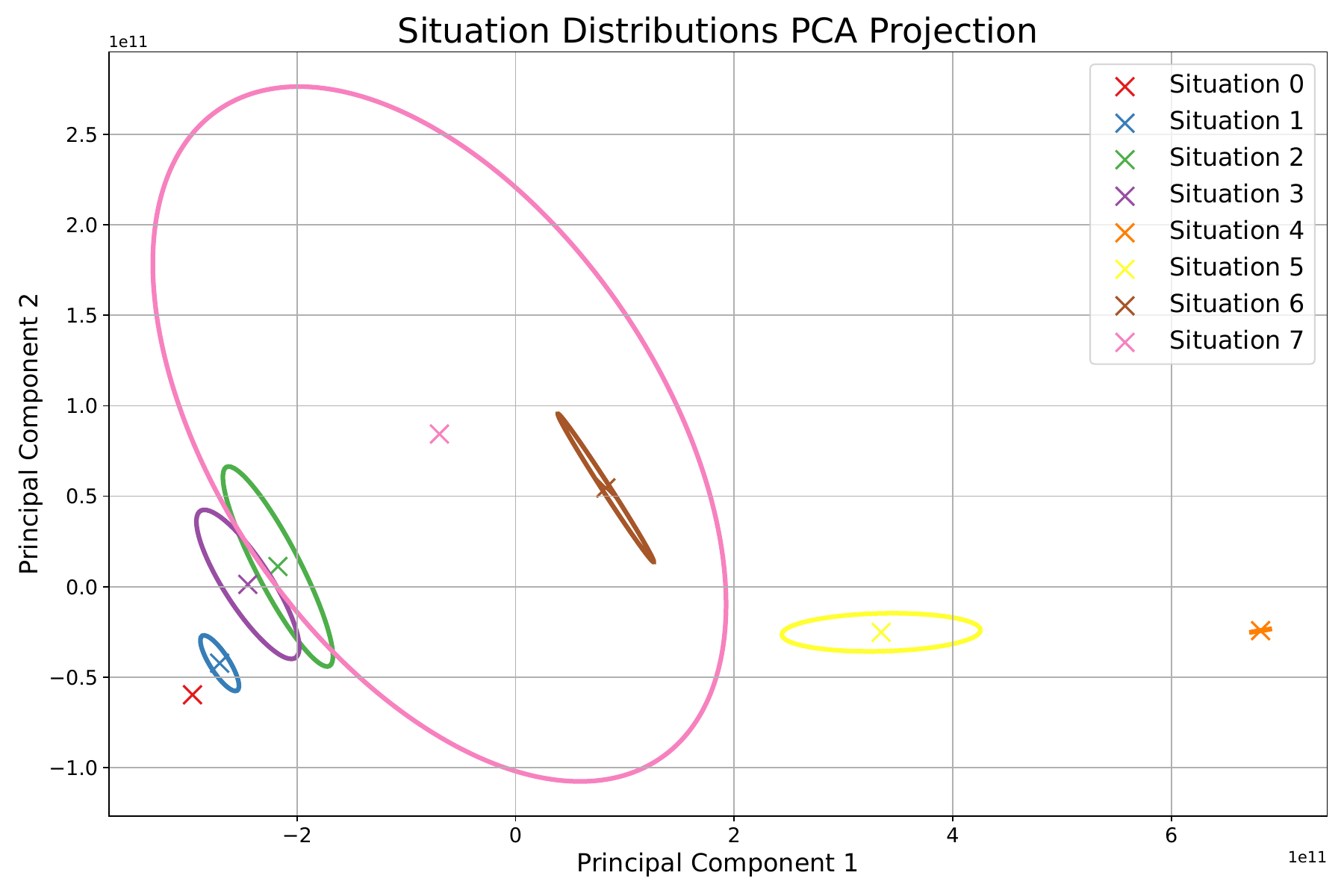}
    \captionof{figure}{PCA decomposition of the Situations (26-dimensional local transition dynamics distributions) found during training.}
    \label{fig:situations-pca}
\end{figure}

Quantitatively evaluating the performance of online hidden state representation learning is inherently challenging for this task, as the true hidden state and its transitions are naturally unobservable. To address this, we qualitatively assess the performance of the situational awareness module through two complementary analyses. First, we examine whether the module identified meaningful situation changepoints during the robot's operation, as illustrated in Figure \ref{fig:reduced-observation-history}. \textit{The results show that the model effectively segmented the local transition dynamics under varying patterns, accurately capturing and representing the changes in the dynamics.} The full 26-dimensional picture is available in Appendix \ref{sec:appendix-inspection_world_sa_results}. Second, we visualize the diversity of learned hidden state representations by mapping the 26-dimensional local transition dynamics $T_L$ identified during training onto the first two principal components, shown in Figure \ref{fig:situations-pca}. \textit{This visualization reveals distinct clusters corresponding to different situations, demonstrating that each representation encodes unique patterns in the transition dynamics without collapsing into a single, generalized representation.} 

\subsection{Ablation Studies} \label{sec:ablation-studies-ugv}

\subsubsection{Learning Efficiency:}

\begin{figure*}
    \centering
    \includegraphics[width=\linewidth]{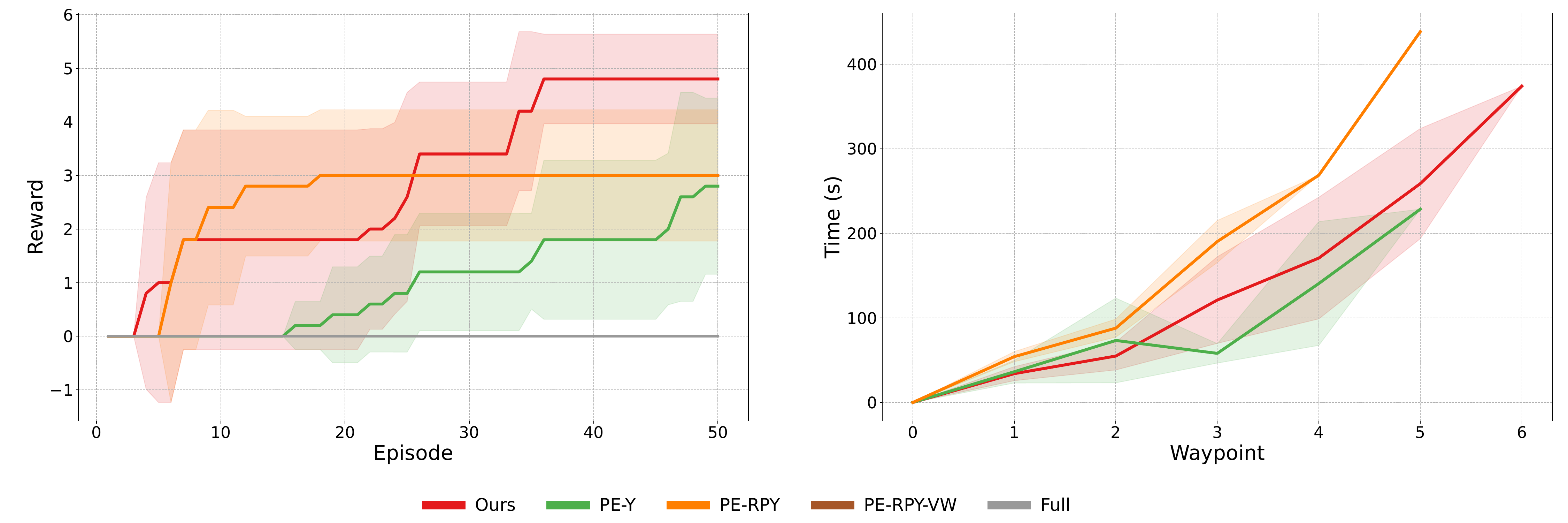}
    \captionof{figure}{Ablation Results. \textit{(Left)} Training task reward. We show that the learned hidden state representation helps the robot complete its tasks faster. \textit{(Right)} Mean $\pm$ standard deviation of the time that it takes the robot to reach the waypoints in the task. Each experiment was repeated five times.}
    \label{fig:ablation_results}
\end{figure*}

We evaluate whether the advantages of our method stem from the situation information or input feature selection. We compare against other probabilistic ensembles with equivalent architectures but relying on different input features that do not include the hidden state representation. Specifically, the models include commanded linear and angular velocity as well as: \textit{(a)} the robot's heading as state input (which we call ``PE-Y"); \textit{(b)} all the orientation information (``PE-RPY"); \textit{(c)} orientation along with linear velocities of the robot (``PE-RPY-VW"); and \textit{(d)} one that also includes position to see if additionally learning the map information results in better performance (``Full"). These models let us evaluate the contribution of our hidden state representation, as they involve multiple possible input configurations for the dynamics model but \textit{without} the situation information. 

The results, shown in Figure \ref{fig:ablation_results}, reveal that while some of the ablated models achieve better early performance than the baselines within the first 50 episodes, they consistently underperform compared to our situationally-aware dynamics model. 
This pattern highlights the importance of situational information in capturing complex dynamics. Models with fewer input features tend to perform better in the initial stages, likely due to reduced compounding error in auto-regressive predictions during trajectory sampling. However, their limited input representation results in inferior long-term performance, underscoring the critical role of situational awareness in achieving fast, robust and adaptive dynamics modeling.

\subsubsection{Effect of the Hidden State Representations:}

We use the Integrated Gradients \cite[]{integrated_gradients} feature attribution method to verify that our situationally-aware dynamics model relies on the different situations to make its predictions, thus influencing the downstream planner's control strategies. Figure \ref{fig:feature-contributions} shows that on average, 60\% of the prediction is explained by the current situation the robot is in, while the remaining input features together account for roughly 40\%. This clearly highlights that the model is aware of the situation the robot is experiencing and modulates its predictions accordingly.

Furthermore, in Figure \ref{fig:dynamics-predictions-analysis} it is possible to see how when we apply the same action sequence from the same initial state but assuming the robot stays in different situations, the dynamics predictions differ. What this means is that the hidden state information represented as $\theta \in \Theta$ that is fed to the dynamics model is successfully accomplishing its mission of informing about the unobserved latent factors that could hinder the dynamics learning process when such differentiation is not possible. Specifically, we can see how the model can encode subtle differences in position changes (e.g., Situation 1 vs. Situation 3), larger changes (e.g., Situation 4 vs. Situation 6) or completely different dynamics (e.g., Situations 0, 2 \& 5). This could be explained, for example, by how the robot's motion is limited in certain terrains or how the same action can lead to different reactions under different circumstances (e.g., when going downhill vs. moving on flat and smooth terrain). This phenomena is also intuitively illustrated in Figure \ref{fig:output-discrepancy}.

\begin{figure}
 \centering
  \includegraphics[width=0.7\linewidth]{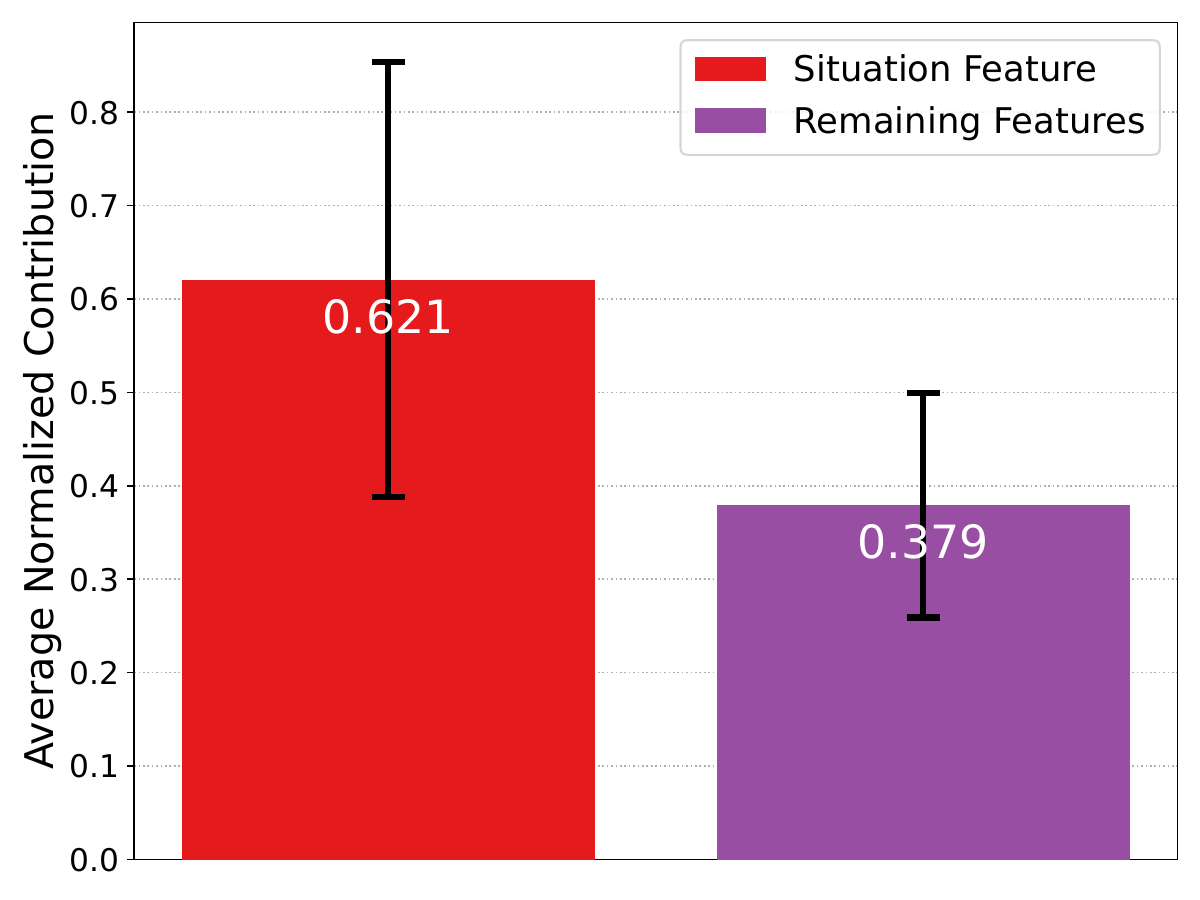}
    \captionof{figure}{Dynamics model input features contributions to the prediction.}
    \label{fig:feature-contributions}
\end{figure}

\begin{figure*}
 \centering
  \includegraphics[width=0.85\linewidth]{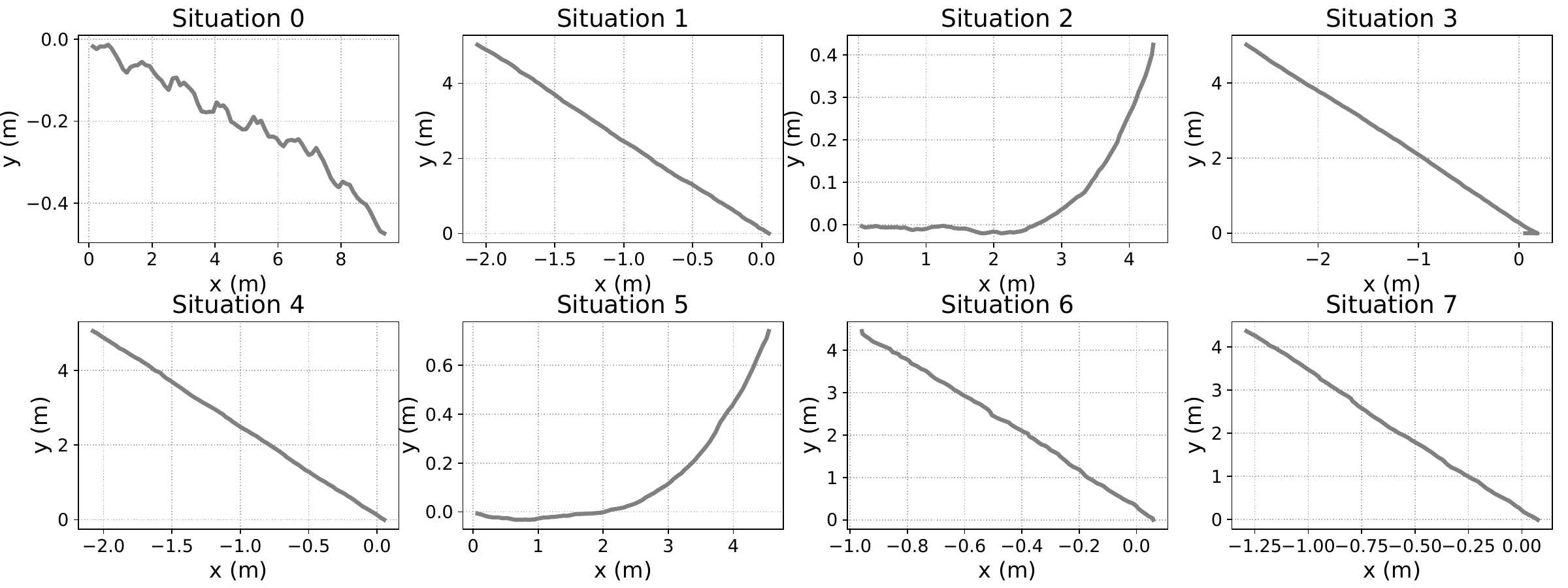}
    \captionof{figure}{Predicted dynamics for the same action across different situations. The hidden state effectively captures both subtle and significant variations in the robot's dynamics, providing critical information that can be leveraged during decision-making to enhance performance.}
    \label{fig:dynamics-predictions-analysis}
\end{figure*}

For a more detailed examination of the situational awareness module's behavior under controlled dynamics changes, we conduct additional experiments in a simulated environment with a manipulable UDGP. These experiments, detailed in Appendix \ref{sec:appendix-lake_sim_sa_results}, use the Lake World Simulator to evaluate how changes in wave generation patterns affect the identification of situation transitions by an autonomous boat. The results highlight the impact of the hyperparameter $\lambda$ in segmenting the observation space and provide further insights into the relationship between UDGP variations and situation changepoints.

\subsection{Situationally-Aware Unstructured Terrain Navigation} \label{subsec:results-unstructured-terrain}

We evaluate the dynamics models further through extensive simulation and real-world tasks, comparing against the physics-based Dubins model and the top-performing learning-based baselines: PE-Y and PETS (PE-RPY).

\subsubsection{Simulation Experiments:}

The complementary simulation experiments were carried out in the four randomly-generated simulation environments introduced in Section \ref{subsubsec:sim-eval-scenarios-ugv}, whose parameters are described in Table \ref{tab:simulation-environments-params}. These environments feature increasing levels of difficulty with successively steeper and more rugged terrains. The experiments comprised navigation tasks from fifty randomly-generated initial and goal positions for every environment. Table \ref{tab:simulation-results} shows that our situationally-aware dynamics model had a higher success rate than the remaining methods. Nonetheless, we identify that situation switching leads the robot to re-plan under the new context which seems to slightly affect some of the path smoothness metrics (position jerk, orientation vibration and orientation jerk).

\begin{figure}
 \centering
  \includegraphics[width=\linewidth]{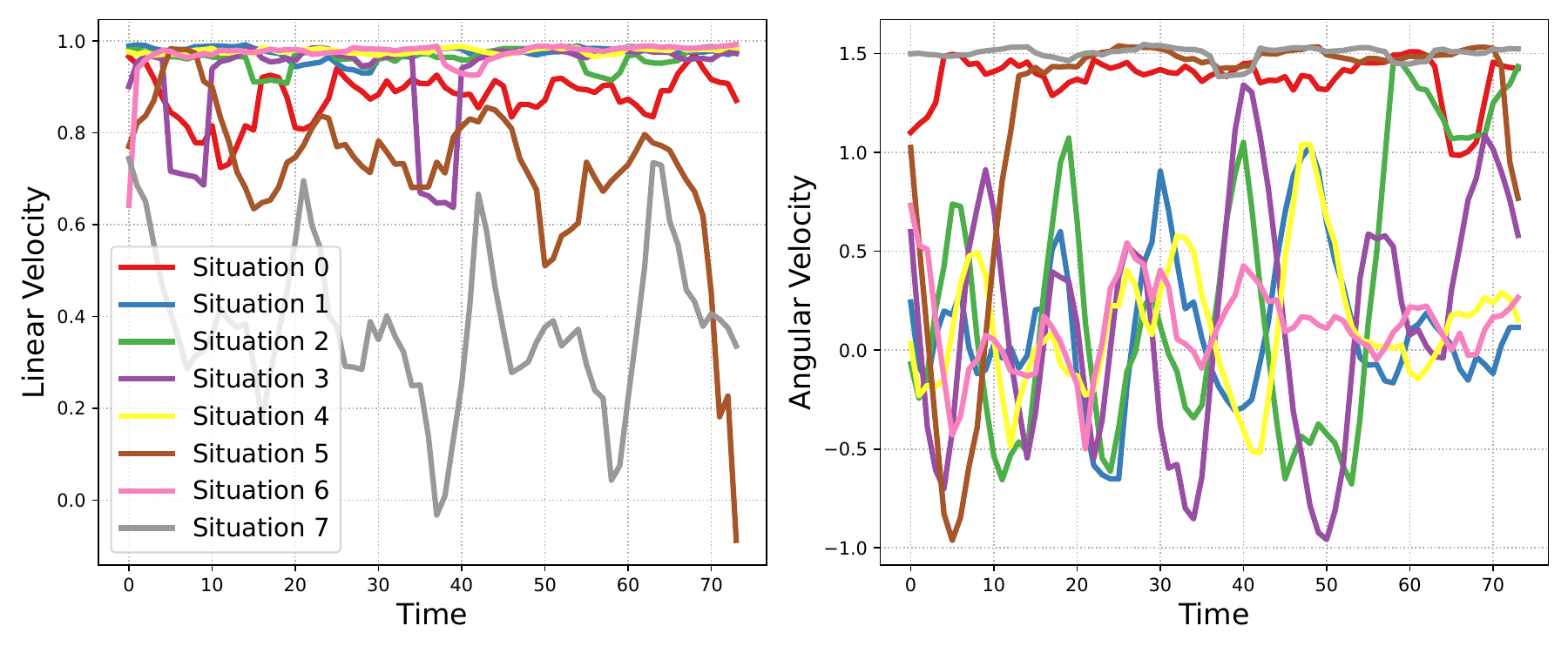}
    \captionof{figure}{Control plans assuming the robot stays in each situation. As the transition dynamics change, the robot is able to generate successful plans that adapt to the constraints of each situation. For ease of visualization we present the moving average of the signals (window size 5).}
    \label{fig:dynamics-actions-analysis}
\end{figure}

\begin{figure} 
    \centering
    \includegraphics[width=0.9\linewidth]{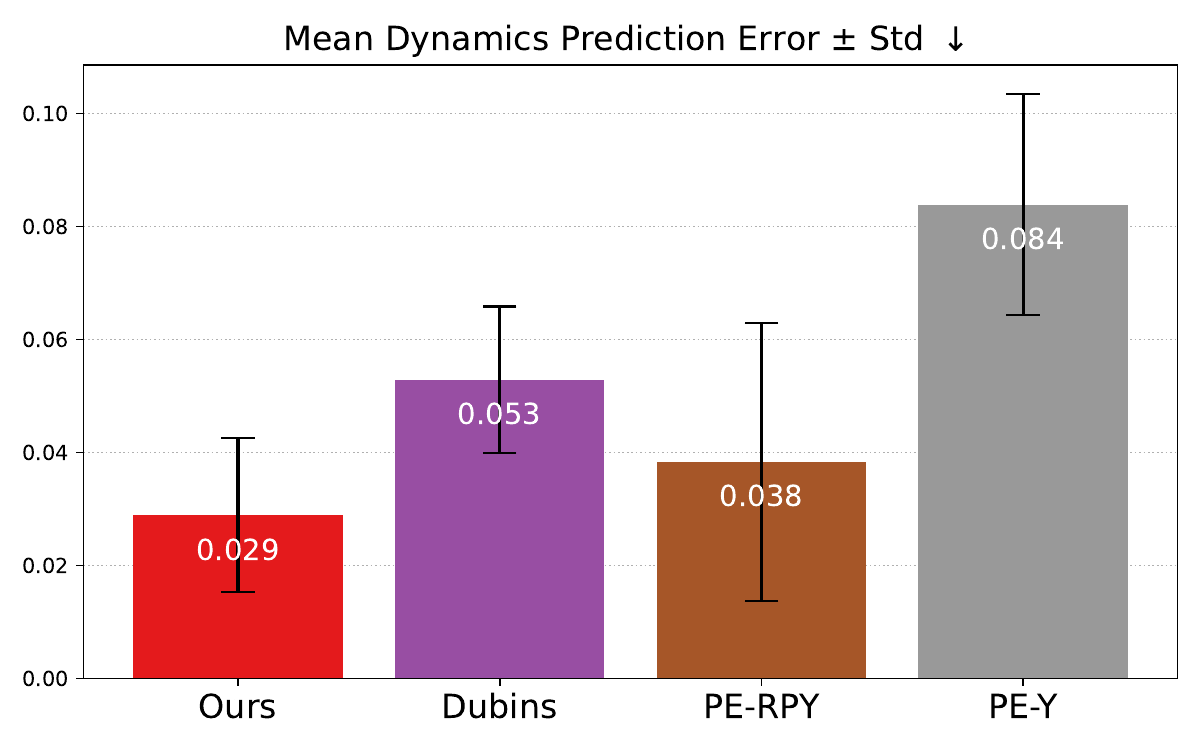}
    \caption{Control Inconsistency. Mean squared error between the predicted and actual behaviors resulting from a control command. Our approach achieves the lowest error, demonstrating superior accuracy in modeling the robot's dynamics.}
    \label{fig:inconsistency}
\end{figure}

We also carry out additional experiments to better understand the plans under each situation. As shown in Figure \ref{fig:dynamics-predictions-analysis}, different situations result in distinct dynamics models. This means that during model-based planning with the sampling-based controller the robot will synthesize different policies for the same task whenever it is in a different situation. We illustrate this in Figure \ref{fig:dynamics-actions-analysis}, where we can see how for the same mission, when we assume the robot stays in the same situation, it selects significantly different linear and angular velocity commands. 
Notably, Situations 0, 5, and 7 exhibit an emergent behavior where the robot tends to move more cautiously, with reduced velocities, suggesting it perceives these situations as requiring safer actions. This behavior aligns with observations in real-world evaluations, as detailed in the following section.

\textbf{Control Inconsistency.} In Figure \ref{fig:inconsistency} we measure the control \textit{inconsistency}. Formally, it relates to ``the error between the robot’s expected navigational behavior and its actual behavior in terms of robot poses" \cite[]{siva2024self}. Our method demonstrates the best performance, followed by the model that incorporates all orientation information (PE-RPY). 

\textit{This result underscores the suitability of our model for sampling-based control.} Specifically, our approach strikes the best balance for auto-regressive prediction tasks, such as planning with MPPI. It benefits from the lower one-step prediction error characteristic of more complex models that are harder to learn and generalize, while also offering the reduced prediction space of simpler models (e.g., Dubins), which helps mitigate compounding errors during trajectory rollout.

\begin{table}
    \centering
    \small
    \begin{tabular}{lcccc} 
    \toprule
    \textbf{Method} & \textbf{S. Rate} & \textbf{Pos. Jerk}  & \textbf{Ori. Vib.} & \textbf{Ori. Jerk} \\ 
\toprule \textit{Difficulty:} 1 \\ \hdashline 
PE-Y & 0.32 & 0.00 & 0.38 & 0.01 \\ 
PE-RPY & 0.68 & 0.00 & 0.44 & 0.13 \\
Dubins & 0.98 & 0.02 & 0.18 & 0.03 \\
\rowcolor{lightgray}   Ours & 1.0 & 0.04 & 0.30 & 0.08 \\ 
\toprule \textit{Difficulty:} 2 \\ \hdashline 
PE-Y & 0.26 & 0.01 & 0.32 & 0.05 \\ 
PE-RPY & 0.64 & 0.00 & 0.48 & 0.02 \\ 
Dubins & 1.0 & 0.00 & 0.19 & 0.02 \\ 
\rowcolor{lightgray}   Ours & 1.0 & 0.03 & 0.34 & 0.02 \\ 
\toprule \textit{Difficulty:} 3 \\ \hdashline 
PE-Y & 0.34 & 0.01 & 0.42 & 0.09 \\ 
PE-RPY & 0.86 & 0.01 & 0.57 & 0.09 \\ 
Dubins & 0.98 & 0.02 & 0.30 & 0.12 \\
\rowcolor{lightgray}   Ours & 1.0 & 0.07 & 0.42 & 0.03 \\ 
\toprule \textit{Difficulty:} 4 \\ \hdashline 
PE-Y & 0.22 & 0.09 & 0.55 & 0.35 \\ 
PE-RPY & 0.56 & 0.02 & 0.66 & 0.62 \\ 
Dubins & 0.82 & 0.08 & 0.52 & 1.02 \\
\rowcolor{lightgray}   Ours & 0.86 & 0.09 & 0.58 & 0.77 \\ 
    \bottomrule
    \end{tabular}
    \captionof{table}{Results on four simulation environments with increasing level of difficulty. The robot needs to traverse to fifty randomly generated initial and goal locations in each environment. We evaluate the Success Rate, Position Jerk $[\frac{m}{s^3}]$, Orientation Vibration $[\frac{rad}{s}]$ and Orientation Jerk $[\frac{rad}{s^3}]$.}
    \label{tab:simulation-results}
\end{table}

\subsubsection{Real World Experiments:} 

\begin{figure*} 
    \centering
    \includegraphics[width=\linewidth]{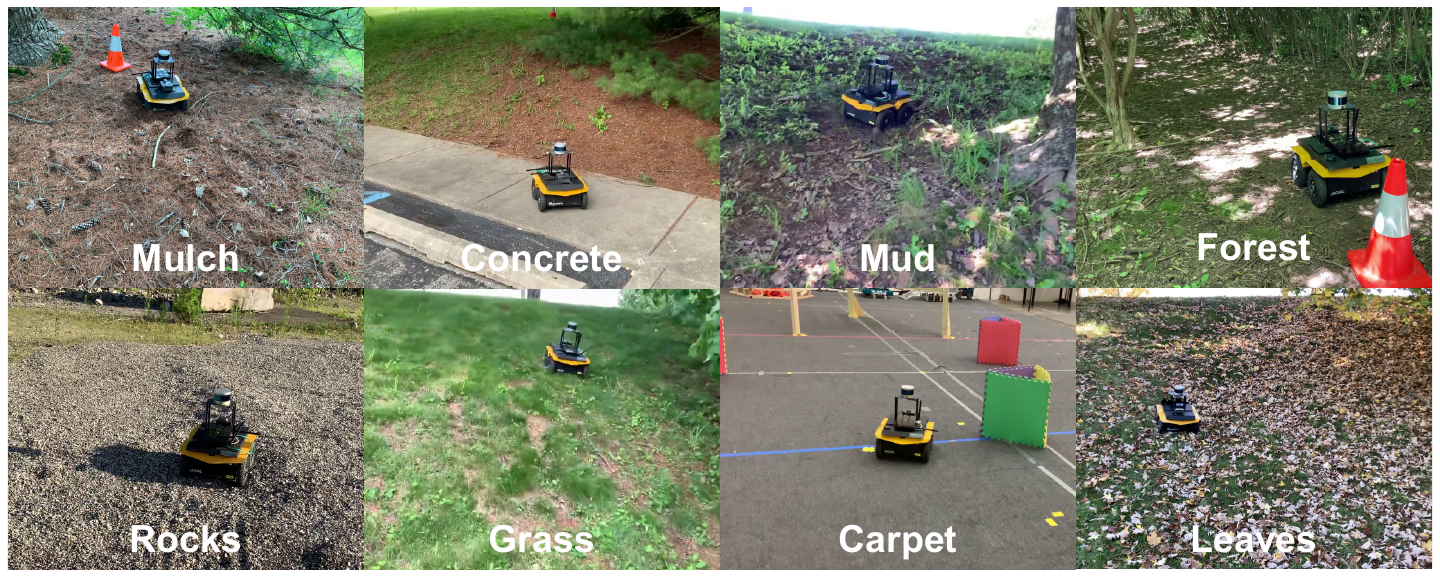}
    \caption{Using the learned Situationally-Aware Dynamics Model the robot can successfully traverse multiple unstructured terrains.}
    \label{fig:nav_summary}
\end{figure*}

We evaluate the zero-shot performance of our situationally-aware dynamics model by applying it directly from simulation to real-world scenarios without additional fine-tuning. Figure \ref{fig:nav_summary} highlights various challenging terrains successfully navigated using our method, demonstrating its robustness and adaptability in real-world conditions. The \textbf{Supplementary Video 1 (SV1)}: \url{https://youtu.be/7a-NzAackE4}, showcases the robot successfully navigating these scenarios. Additionally in \textbf{Supplementary Video 2 (SV2)}: \url{https://youtu.be/Eo68c5APuAo}, we also analyze the real time responsiveness of the situational awareness module as the robot traverses new scenarios.

\textbf{Tasks \& Performance Assessment.} To quantitatively evaluate our method in the real-world, we compare it against the baselines on the navigation task shown in Figure \ref{fig:task1} \textit{(left)}. This task involves traversing diverse terrain types (mulch, grass, mud, and concrete) and geometries (downhill, flat, and uphill). Notably, the final segment features a steep, short climb that includes mulch, where reduced friction presents a significant challenge. Each method was tested 10 times, with results summarized in Figure \ref{fig:task1-results}. Overall, our method consistently outperformed the baselines, being the only approach to successfully complete the task in all trials.
Although our method shows a slightly higher completion time compared to the successful runs of the best-performing baseline (Dubins), this can be attributed to the emergent behavior highlighted in Figure \ref{fig:dynamics-actions-analysis}. Specifically, in certain situations, the robot adopts slower, more cautious policies, prioritizing safety over speed. 

We also move to more challenging tasks with steeper and more rugged terrain, which is shown in Figure \ref{fig:task1} \textit{(right)}. To protect the hardware, we only evaluated the best performing baseline in the previous task (Dubins). However, we had to stop the experiment with the Dubins model as it repeatedly led to the robot turning over and rolling down the hill, as shown in Figure \ref{fig:downhill-task}. In contrast, our method repeatedly completed the task, as can be seen starting at the \href{https://youtu.be/7a-NzAackE4?t=102}{1:42} minute mark of \textbf{SV1}. Overall, our method can lead the robot to successfully climb bumpy and muddy hills, whereas the baselines fail as they get stuck on terrain irregularities or are unable to devise a safe and successful plan.

\textbf{Emerging Behaviors.} A highlight of the real-world experiments is the emergence of some behavioral patterns that helped the robot navigate safely and overcome terrain irregularities. For instance,
\begin{itemize}
    \item \textit{Backing Up:} Whenever forward motion is impeded, we see the robot back up and try moving with a slight orientation change to overcome the terrain irregularities. This mirrors the behavior of human drivers when their vehicles become stuck in challenging terrains. Check the clips from \textbf{SV1} at the: \href{https://youtu.be/7a-NzAackE4?t=60}{1:00}, \href{https://youtu.be/7a-NzAackE4?t=87}{1:27 \textit{(trial 4)}}, \href{https://youtu.be/7a-NzAackE4?t=126}{2:06}, \href{https://youtu.be/7a-NzAackE4?t=156}{2:36 \textit{(trial 4)}}, \href{https://youtu.be/7a-NzAackE4?t=210}{3:30} \& \href{https://youtu.be/7a-NzAackE4?t=374}{6:14 \textit{(rocks)}} minute marks for evidence of this.   
    \item \textit{Velocity Control:} We have seen multiple evidence of how the robot's planned velocity commands differ according to its current situation. For example, in Figure \ref{fig:dynamics-actions-analysis} we saw, for every situation, different plans to reach the same goal from the same starting state. Moreover, at the \href{https://youtu.be/7a-NzAackE4?t=240}{4:00} minute mark of the \textbf{SV1}, the robot maintains a slower pace while navigating downhill and then deliberately accelerates (minute \href{https://youtu.be/7a-NzAackE4?t=255}{4:15}) upon a situation change--reaching safer, more stable terrain.
    \item \textit{Trajectory Smoothness:} On very rugged and dangerous terrains the robot followed curved patterns (see the clips from \textbf{SV1} at minutes \href{https://youtu.be/7a-NzAackE4?t=102}{1:42}, \href{https://youtu.be/7a-NzAackE4?t=142}{2:22} \& \href{https://youtu.be/7a-NzAackE4?t=198}{3:18}). On the other hand, when the terrain was not as rugged and steep, the robot approached its targets directly, relying only on the velocity control to maintain its integrity (\href{https://youtu.be/7a-NzAackE4?t=240}{4:00}, \href{https://youtu.be/7a-NzAackE4?t=263}{4:22} \& \href{https://youtu.be/7a-NzAackE4?t=364}{6:04} minute marks).
    
\end{itemize}
It is important to clarify that during training we did not try to enforce learning of such patterns in any way. 

We dive deeper into why we believe this is happening in Section \ref{sec:discussion}. At a high-level, we hypothesize that upon a situation change, the optimization process carried out by the sampling-based planner will be bounded by new constraints given by the dynamics of the current situation, thus generating new diverse action plans. 

Furthermore, we want to highlight that the UGV platform used (Clearpath Jackal) is a rigid-body robot with no Ackermann-style steering mechanism, minimal suspension, and relatively low ground clearance. This means that the robot often experiences contact between the vehicle body and external factors (grass, rocks, tree roots, etc.), so the chassis is not consistently clear of obstructions. These physical constraints become latent factors that can also affect the dynamics and some maneuvers that might be feasible for other vehicles can be unsafe or damaging for this platform.

\begin{figure*}
 \centering
  \includegraphics[width=\linewidth]{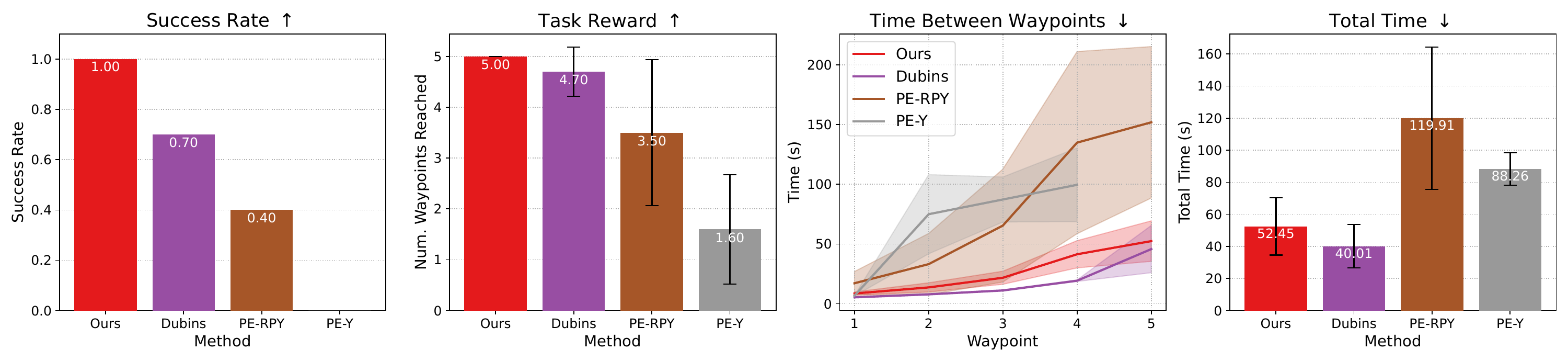}
    \captionof{figure}{Results from the UGV experiments in the real world. We repeat the experiments ten times with each baseline and report the success rate, number of waypoints reached, the cumulative time between successive waypoints and the total time to complete the task.}
    \label{fig:task1-results}
\end{figure*}

\begin{figure*}
 \centering
  \includegraphics[width=\linewidth]{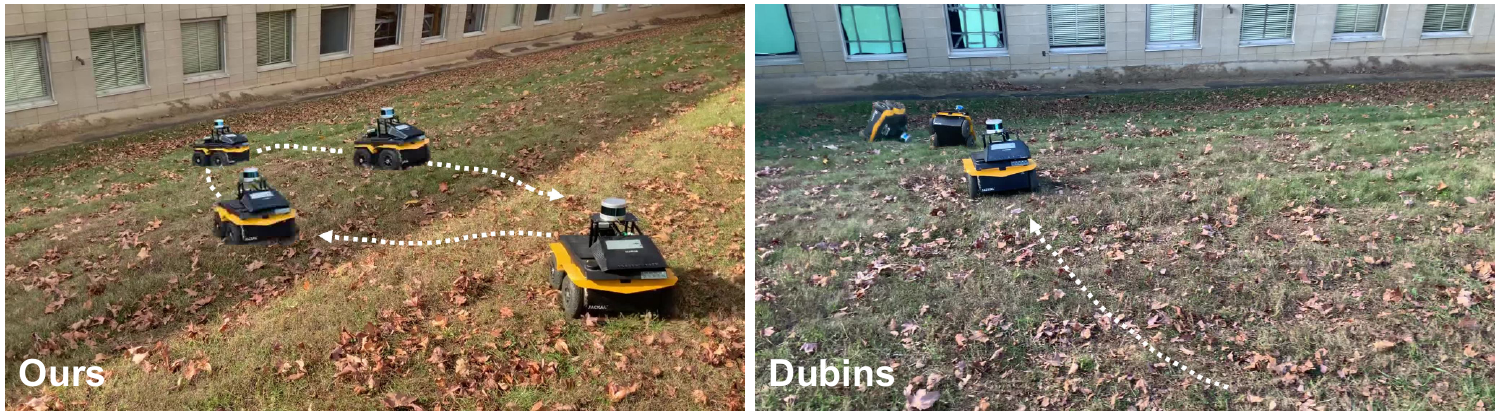}
    \captionof{figure}{We test our situationally-aware dynamics model against the physics-based Dubins model (best performing baseline) in the challenging task of descending a steep hill with wet leaves and many large hidden bumps. Across multiple trials, our model consistently succeeded in completing the task, whereas the baseline failed, causing the robot to flip.}
    \label{fig:downhill-task}
\end{figure*}

\section{Results: \textit{Quadrotor}} \label{sec:results_quadx}

\begin{figure}[h!]
    \centering
    \includegraphics[width=0.475\linewidth]{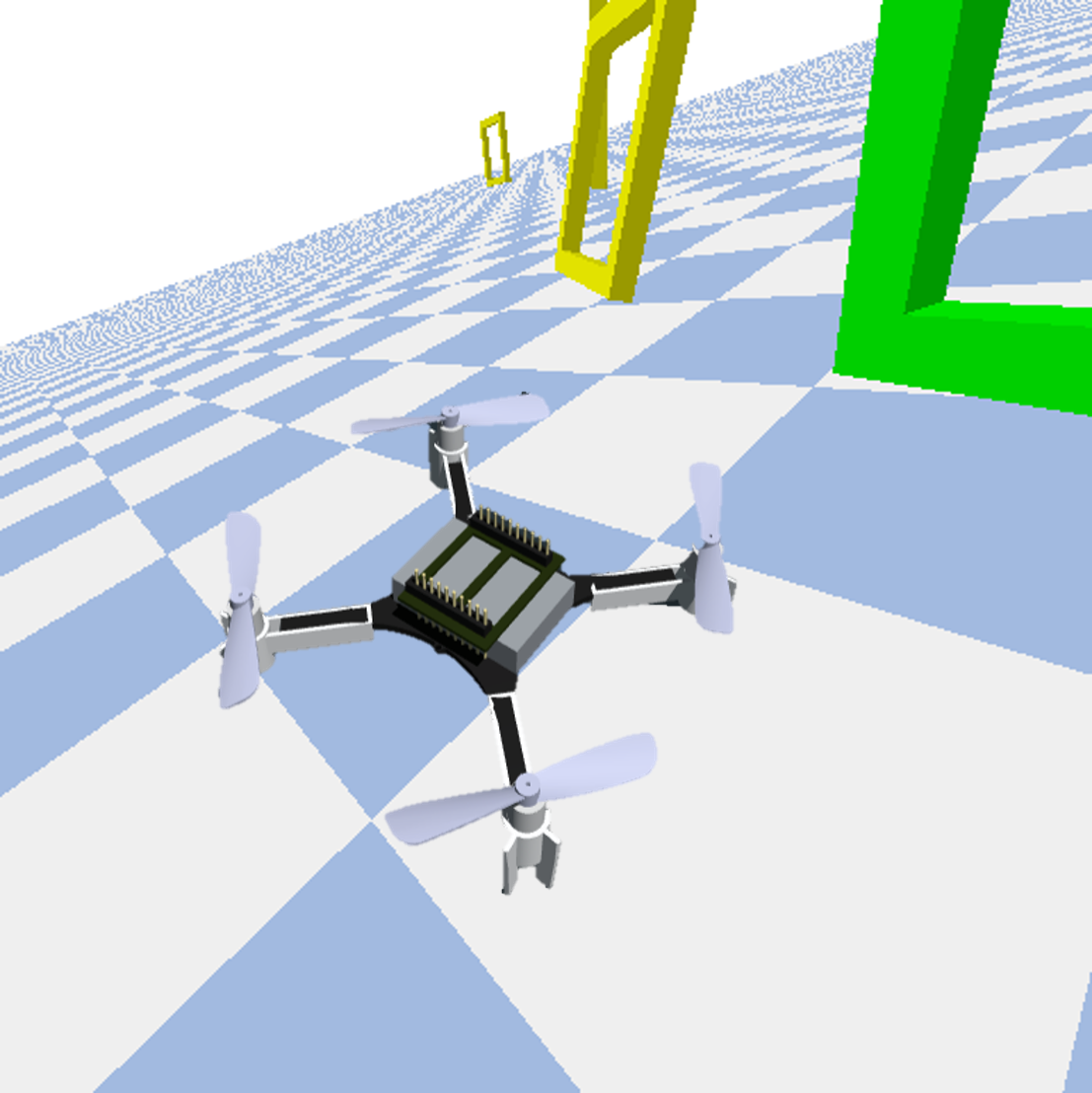}
    \includegraphics[width=0.475\linewidth]{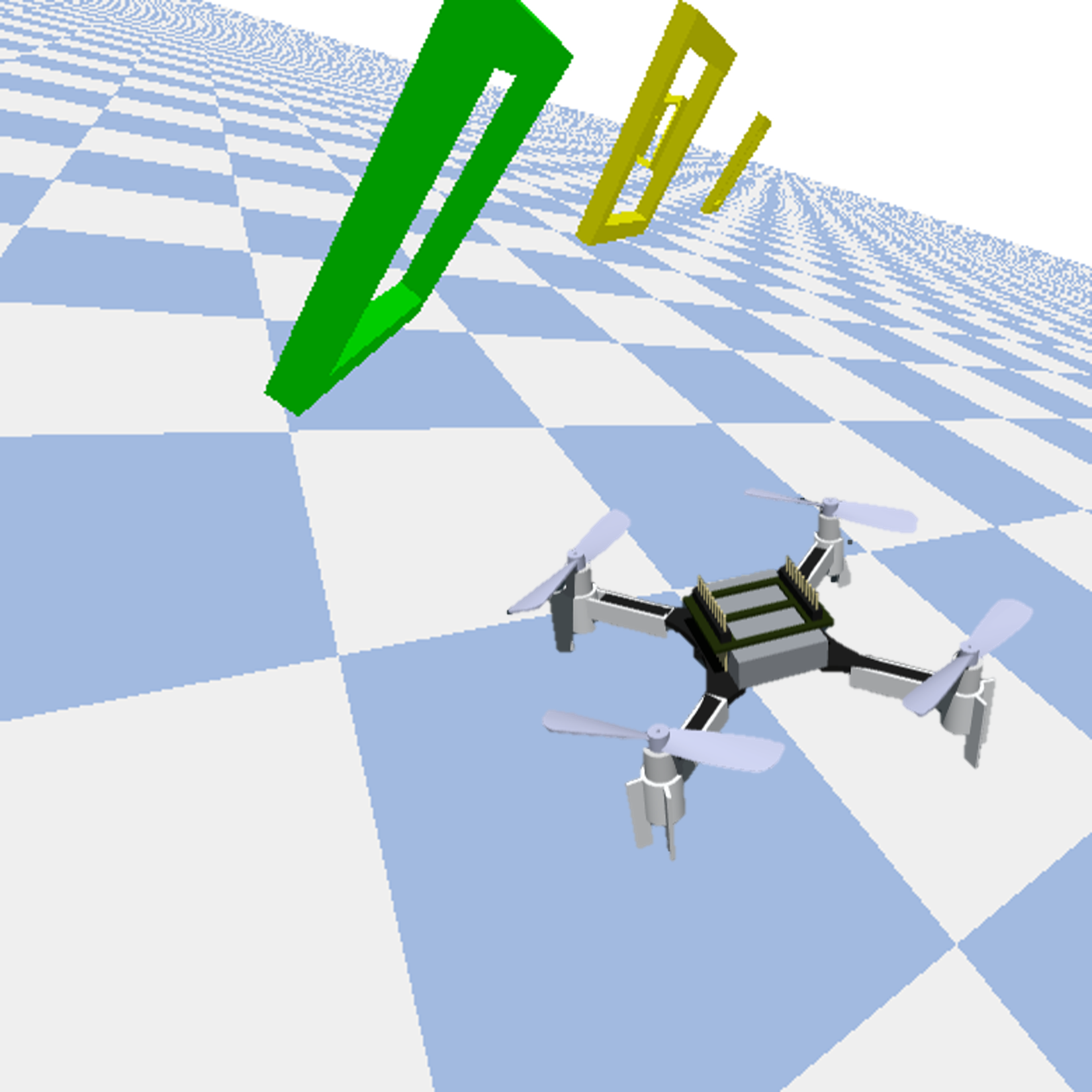}
    \caption{Crazyflie simulated in PyFlyt's gates environment with changing wind conditions. The next target gate is colored green and the remaining targets yellow.}
    \label{fig:quadx_simulator}
\end{figure}

We continue to assess our situationally-aware dynamics model with a different platform type, seeking to provide further evidence to answer the three questions posited in Section \ref{sec:results_ugv}. Specifically, we investigate the performance of SA for a quadrotor tasked with crossing a series of gates suppend in the air while being disturbed by sudden changes in wind patterns. Figure \ref{fig:quadx_simulator} presents an example of the task. Note that the agent has neither knowledge of when the change happens nor a dedicated sensor to measure specific variables that would let it determine when the wind conditions have changed or how they have changed. Therefore, it needs to rely on its proprioception to decide how to adapt.

In Section \ref{subsec:eval-scenarios-quadx} we introduce the multiple evaluation scenarios. Section \ref{sec:exp-setup-quadx} covers the evaluation setup, including a description of the state, action and situation space. Finally, in Section \ref{subsec:results-training-quadx} we show that leveraging a situationally-aware dynamics model the quadrotor achieves a higher reward on the task across environments with different levels of difficulty and with multiple data budgets.

\subsection{Evaluation Scenarios}  \label{subsec:eval-scenarios-quadx}

We leverage the PyFlyt simulator \cite[]{tai2023pyflyt}. The main driver for our selection was the ease to implement different wind patterns, including randomly switching wind fields. In particular, we focused on the gate-crossing task, where a quadrotor needs to cross five gates suspended in the air. We modified the base task to include switching wind fields with three different levels of difficulty, but leave the reward computation and remaining components of the original task unmodified.

We designed three distinct environments, with different levels of complexity. In each environment, the active wind condition is randomly selected from a predefined set and persists for a random duration before switching to a new, different condition. The specifics of each environment are detailed in Table~\ref{tab:wind_configs}.

\begin{table*}[h!]
\centering
\small
\begin{tabular}{@{}clll@{}}
\toprule
\textbf{Environment} & \textbf{Switching Interval (steps)} & \textbf{Wind Type} & \textbf{Key Parameters} \\ \midrule
1 & {[}15, 45{]} & \begin{tabular}[c]{@{}l@{}}Nominal \\ Thermal Updraft\end{tabular} & \begin{tabular}[c]{@{}l@{}}- \\ Strength = 1.0\end{tabular} \\ \midrule
2 & {[}15, 45{]} & \begin{tabular}[c]{@{}l@{}}Nominal \\ Wind Gradient \\ Wind Gradient\end{tabular} & \begin{tabular}[c]{@{}l@{}}- \\ Speed at 10m = 2.0 m/s, From 180°, $z_0=0.1$ \\ Speed at 5m = 5.0 m/s, From 45°, $z_0=0.1$\end{tabular} \\ \midrule
3 & {[}10, 50{]} & \begin{tabular}[c]{@{}l@{}}Nominal \\ Wind Gradient \\ Wind Gradient \\ Thermal Updraft\end{tabular} & \begin{tabular}[c]{@{}l@{}}- \\ Speed at 10m = 5.0 m/s, From 270°, $z_0=0.3$ \\ Speed at 3m = 5.0 m/s, From 90°, $z_0=0.1$ \\ Strength = 3.0\end{tabular} \\ \bottomrule
\end{tabular}
\caption{Wind field configurations for the quadrotor gate-crossing task. During execution, the environments randomly switch between their available wind types without revealing this information to the agent. The agent has to rely on its proprioception to detect and characterize this conditions and adapt accordingly.}
\label{tab:wind_configs}
\end{table*}

\subsection{Experimental Setup}  \label{sec:exp-setup-quadx}

\subsubsection{State and Action Space:} The state space $s \in \mathbb{R}^{12}$ corresponds to $s = \begin{bmatrix} \mathbf{\Omega} & \mathbf{e} & \mathbf{v} & \mathbf{p} \end{bmatrix}^\top$,
where $\mathbf{\Omega} = [\omega_x, \omega_y, \omega_z] \in \mathbb{R}^3$ are the angular velocity components, $\mathbf{e} = [\phi, \zeta, \gamma] \in \mathbb{R}^3$ are the roll, pitch, and yaw angles (orientation), $\mathbf{v} = [v_x, v_y, v_z] \in \mathbb{R}^3$ are the linear velocity components and $\mathbf{p} = [x, y, z]\in \mathbb{R}^3$ is the drone's position.

The action space is defined as $a = \begin{bmatrix} \omega^x_{\text{cmd}} & \omega^y_{\text{cmd}} & \omega^z_{\text{cmd}} & T\ \end{bmatrix}^\top \in \mathbb{R}^4$, where $\omega^x_{\text{cmd}}, \omega^y_{\text{cmd}}, \omega^z_{\text{cmd}} \in [-\pi, \pi]$ are the commanded angular velocities, and $T \in [0, 0.8]$ is the thrust command.

\subsubsection{Situation Space:}  
Since we defined the \textit{situation} of the robot as the joint transition distribution, we concatenate the state and action variables to obtain a situation vector $\mathbf{x} \in \mathbb{R}^{28}$:
\begin{equation} \label{eq:situation-space-quadx}
    \mathbf{x}_t = \begin{bmatrix} s_{t-1}; & a_{t-1}; & s_t \end{bmatrix}^\top,
\end{equation}
where the states $s_{t-1}$ and $s_t$ are preprocessed to ensure invariance and continuity as described in Section \ref{sec:exp-setup-ugv}.

\subsection{Situationally-Aware Dynamics Learning}   \label{subsec:results-training-quadx}

To learn the quadrotor's dynamics we leverage a dataset of state transitions collected in each of the three environments. Since no exploration by the agent needs to take place, we proceed in the \textit{passive} training fashion described in Section \ref{sec:training_testing}. That is, the dataset is fed sequentially to the agent and the model is then trained after the situation-augmented replay buffer has been obtained.

We train two agents. One corresponds to a probabilistic ensemble (PETS) \cite[]{chua2018deep}. The other one, our agent which we call SA-PETS, is the exact same probabilistic ensemble but its input has been augmented to also include the situation symbol obtained by our online and unsupervised latent factor representation method.

Both probabilistic ensembles consist of five independently trained Gaussian multi-layer perceptrons (MLPs) with residual connections. Each network takes as input the concatenation of the current state and control vector, and outputs the mean and log-variance of a Gaussian distribution of the predicted state change. The architecture begins with a fully connected layer mapping the input dimension to a hidden size of 1024, followed by Layer Normalization and a Mish activation to stabilize training and improve gradient flow. This is followed by two residual blocks, each containing two linear layers with LayerNorm and Mish activations, where the input to the block is added back to its output and passed through an additional Mish nonlinearity. Finally, a linear projection layer maps the hidden representation to twice the dimensionality of the state space, producing both the mean and log-variance of the Gaussian distribution.

Both inputs and targets are normalized using statistics computed from the training data, and predictions are de-normalized at inference time. The model is trained using the Gaussian negative log-likelihood (NLL) loss with the Adam optimizer (learning rate $10^{-3}$). As described in Section \ref{sec:training_testing}, SA-PETS uses a weighted version of NLL. Training employs a learning rate scheduler (factor $0.1$, patience $10$) and early stopping with a patience of $16$ epochs. The batch size is set to 256, and the maximum number of epochs is 256. Table \ref{tab:combined_hyperparams_quadx} presents the SA and MPPI hyperparameters. The sampling based controller uses the same cost function and parameters with SA-PETS and PETS.

\begin{table} 
    \centering
    \small
    \begin{tabular}{llr}
    \toprule
    \textbf{Algorithm} & \textbf{Hyperparameter} & \multicolumn{1}{c}{\textbf{Value}} \\
    \midrule
    \textbf{Situational Awareness} & & \\
    & Hazard Rate $\lambda$ & $60$ \\
    & Threshold $\tau$ & $200$ \\
    & Situation Support & $40$ \\
    & & \\
    \midrule
    \multicolumn{3}{l}{\textbf{MPPI}} \\
    & Number of iterations & $2$ \\
    & Population size & $1024$ \\
    & Reward scaling $\gamma$ & $0.9$ \\
    & Noise scaling $\sigma$ & $0.4$ \\
    & Correlation term $\beta$ & $0.4$ \\
    \bottomrule
    \end{tabular}
    \caption{Hyperparameters used for Situational Awareness and MPPI in the Quadrotor flying under changing wind task.}
    \label{tab:combined_hyperparams_quadx}
\end{table}

\begin{figure*}
    \centering
    \includegraphics[width=0.84\linewidth]{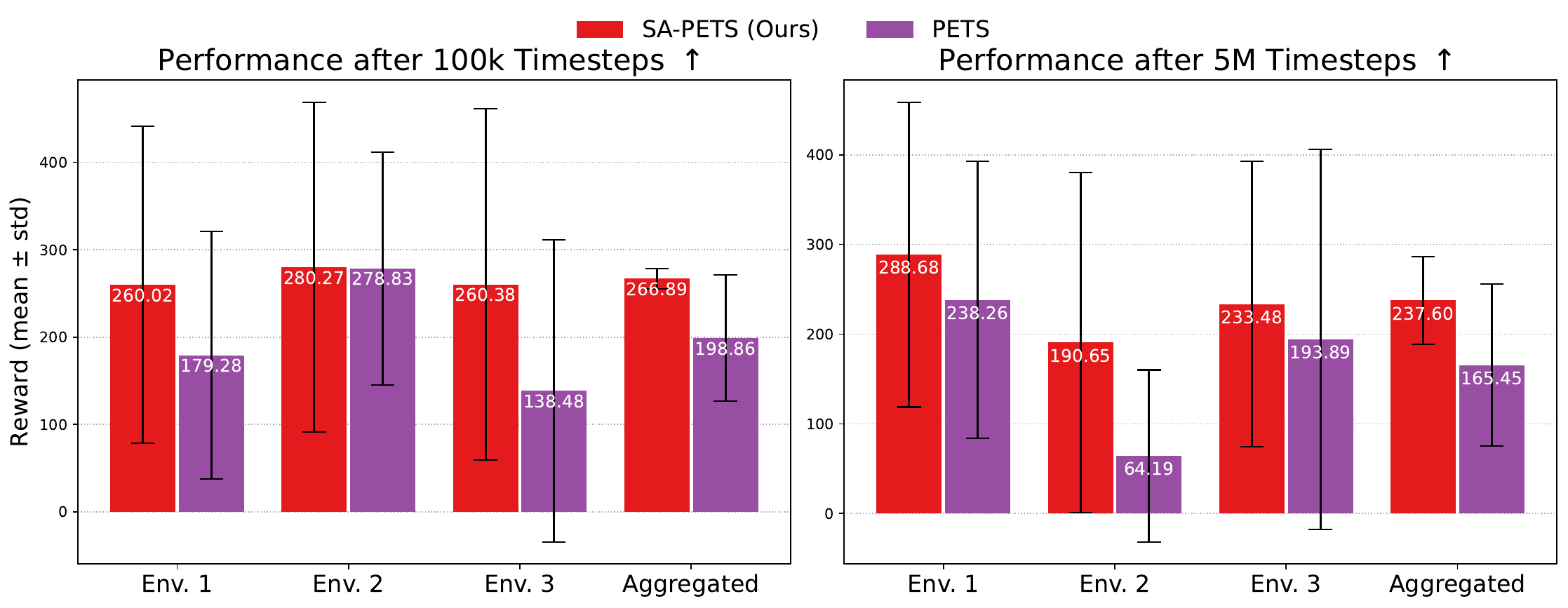}
    \caption{Quadrotor performance on three environments with different randomly changing wind fields. Shown are task rewards across ten episodes after training on 100k (left) and 5M (right) state transitions; higher reward indicates greater robustness.}
    \label{fig:quadx_results}
\end{figure*}

\subsubsection{Training Results:}

Figure \ref{fig:quadx_results} presents the results after training PETS and SA-PETS with 100.000 (left) and 5 million (right) state transitions. In general, we can see SA-PETS significantly outperforms the probabilistic ensemble that does not have access to situation information.

Our proposed method, SA-PETS, demonstrates marked superiority over the PETS baseline in both data-limited and data-rich regimes, as detailed in Figure \ref{fig:quadx_results}. In the low-data scenario with only \textit{100k} transitions, SA-PETS already achieves a substantial \textbf{34.2\%} overall performance improvement, showcasing significant gains in data efficiency. This advantage becomes even more pronounced as the models are exposed to more experience. With \textit{5M} transitions, the performance gap widens, culminating in a \textbf{43.6\%} overall improvement. Most strikingly, we observe a dramatic performance divergence in Environment 2, where the \textit{improvement leaps} from 0.5\% at 100k transitions to a remarkable 197.0\% at 5M transitions. This results strongly indicate that adding situational awareness to the standard PETS model provides a fundamental advantage in model capacity, enabling it to capture latent factors (wind conditions in this case) and exploit complex system dynamics that the baseline model fails to identify, even when provided with large datasets. Such a characteristic is critical for achieving robust, high-performance control in complex, real-world robotic applications.

Upon closer inspection, we also notice an interesting pattern. For instance, in Environment 2, the performance of both PETS and SA-PETS actually degrades as the robot has access to more data. We hypothesize this is due to the fact that larger datasets include the presence of more latent factors ---as the robot experiences more wind patterns--- that confound the dynamics learning process, thereby making it harder for standard models to generalize effectively. Nonetheless, we see that the model augmented with situational awareness is significantly more robust: its task reward only decreased moderately (from 280.27 to 190.65), whereas the baseline PETS suffered a much sharper drop (from 278.83 to 64.19). This highlights not only the stabilizing effect of situational awareness in the face of noisy or confounding data but also suggests that simply collecting more data is not always beneficial—what matters is whether the model can disentangle and adapt to the hidden factors underlying the environment’s variability.

We present videos of the quadrotor flying using both dynamics models in the three environments in \textbf{Supplementary Video 3 (SV3)}: \url{https://youtu.be/a9oVk-Qor9c}.

\section{Discussion} \label{sec:discussion}

This study presented a framework for learning online and unsupervised hidden state representations to inform the dynamics models of robots operating in complex, unstructured environments. By integrating these learned representations with the dynamics models of ground and aerial vehicles we demonstrated significant performance improvements over traditional approaches that rely solely on raw sensory inputs. Particularly interesting was the emergence of safe behaviors that unlocked a new range of tasks the robot could safely accomplish. Here, we dive into the implications of these results, their connection to prior work, and the broader insights they provide for learning-based dynamics modeling in robotics.

\textbf{Impact of Hidden State Representations.} The introduction of hidden state representations proved pivotal in capturing unobservable or implicit environmental factors that influence the robot's dynamics. Figure \ref{fig:dynamics-predictions-analysis} shows that the hidden state representations give the robot a grasp of how latent factors affect its dynamics. Specifically, the transition model incorporates the situation information to get a better per-situation prediction of the states' evolution. It is clear from the plots that such effect has different levels of magnitude, which in case they were unaccounted for, would be absorbed by a dynamics model with higher uncertainty about its predictions. This is undesirable, as increased uncertainty results in less accurate predictions, ultimately reducing the performance of downstream planners. The combined effect of better prediction accuracy leading to better control plans is evidenced in Figure \ref{fig:inconsistency}, where we show our model attained the lowest \textit{inconsistency}. This means the situationally-aware dynamics achieved the lowest error between what the robot predicted an action would generate and the actual measured effect after executing such action in the environment. 
Furthermore, Figure \ref{fig:dynamics-actions-analysis} demonstrates how the robot adjusts its planned behaviors based on its current situation. These representations allow the dynamics model to account for terrain features not explicitly captured in the sensor data, such as ground compliance, surface irregularities, and friction. Additionally, transitioning between situations prompts the robot's sampling-based planner to optimize under varying sets of constraints. We hypothesize that this mechanism contributes to the observed emerging behaviors, as they arise from sequences of optimal actions tailored to the specific conditions of each situation.
Altogether, this underscores the potential of latent representations in bridging the gap between sensor limitations and complex real-world dynamics.

\textbf{Generalization to Novel Terrains.}
One of the most encouraging findings was the UGV model's ability to generalize to unseen terrains directly without additional work from simulation to reality and the emergence of safe navigation behaviors. This suggests that the learned hidden states encode features that are transferable across a broad range of terrain types. Such generalization capabilities are crucial for autonomous systems that must operate in diverse environments without extensive retraining, particularly in scenarios such as disaster response or planetary exploration.

\textbf{Data Efficiency and Model Scalability.}
Despite the complexity of the augmented dynamics model, it demonstrated significant data efficiency compared to baseline methods. By leveraging the representations learned online, the UGV's model required fewer training episodes to achieve superior performance, as evidenced in Figures \ref{fig:train-performance} \& \ref{fig:ablation_results}. Similar results were observed in the quadrotor's case, as demonstrated in Figure \ref{fig:quadx_results}. This efficiency highlights the potential for scaling the approach to larger, more complex systems, and even online dynamics learning tasks, since we abandon the need for privileged information and extensive offline training from previous works on auxiliary representation learning.

\textbf{Interpretability of Learned Representations.}
A notable advantage of our approach is its potential for interpretability. By analyzing the latent space (Figure \ref{fig:situations-pca}), we observed clusters corresponding to distinct transition dynamics, suggesting that the model learns distinct physically meaningful representations. Additionally, as the situations are represented using the structured nature of the local transition dynamics (Eq. \ref{eq:local-dynamics-model}), their values can be interpreted in such framework, in contrast to other approaches that learn representations in an unstructured latent space. This could be further explored to provide insights into environment dynamics for online transition model learning, improving the transparency and trustworthiness of autonomous systems—a critical consideration for deployment in safety-critical applications.

\textbf{On the Situation Change Responsiveness.}
Quantitatively measuring detection lag for an unobserved latent context is inherently difficult, since ground-truth labels for latent factor changes are not available. To address responsiveness within these limits, we rely on complementary qualitative and diagnostic analyses: representative changepoints detected during operation (Figure \ref{fig:reduced-observation-history} and Appendix~\ref{sec:appendix-lake_sim_sa_results}) and a 2D projection of the learned 26-dimensional local-transition descriptors that reveals distinct situation clusters (Figure~\ref{fig:situations-pca}). We also recorded three \textit{new} videos in simulation and the real world (\textbf{SV2}: \url{https://youtu.be/Eo68c5APuAo}) where the accompanying dashboard let us visualize the system responsiveness to a situation change as well as the situation histories. There we can also appreciate the emergence of behavior such as the situation-dependent velocity control when entering more difficult terrain (as also show in Figure~\ref{fig:dynamics-actions-analysis}). Together, these materials show that the module segments local transition dynamics and that the learned representations encode different transition patterns rather than collapsing to a single generic embedding.

How quickly the system switches situations depends on deliberate design trade-offs in the BOCD's changepoint logic. Two elements are particularly important: the changepoint prior, which controls how readily the model hypothesizes a switch (a higher prior yields faster detections at the cost of more false positives), and the minimum-support requirement used before instantiating a new situation (included to prevent numerical issues during precision-matrix inversion), which introduces a short, controlled detection delay. Whether a single state–action sample triggers a switch therefore depends on how out-of-distribution that sample is: in BOCD the run-length posterior depends on the likelihood of the latest observation (see Eqs.~\ref{eq:joint-runlen-obs} and~\ref{eq:pos-run-length}), so a very low likelihood can rapidly shift posterior mass toward a changepoint hypothesis. In practice, with our chosen hazard and minimum-support settings the system favors robust detection of sustained changes rather than flipping on isolated, noisy samples.

\textbf{Integration with Planning and Control.}
The improved dynamics predictions directly impacted downstream planning and control tasks, leading to safer and adaptive navigation/flying. As shown in Figures \ref{fig:nav_summary}, \ref{fig:task1-results}, \ref{fig:downhill-task} \& \ref{fig:quadx_results}, the performance improvement is notably higher in challenging scenarios regardless of their real or simulated nature. By augmenting the state space with the hidden state representations, the controller could generate trajectory plans that proactively accounted for terrain- and wind-induced dynamics changes.

\textbf{Comparison with Learning-Based Models.}
The performance of the CaDM and PE+LSTM baselines highlights the inherent difficulty of learning hidden state representations alongside dynamics in a data-efficient manner. \textit{While these methods have shown promise in other scenarios, our online hidden state representation framework, grounded in multivariate BOCD, offers a more effective alternative. By leveraging structured latent spaces and online learning, our approach eliminates the need for interdependent, separately trained models, addressing the problem from a novel perspective.} We attribute these performance gains to the online nature of our method, which mitigates the non-stationarity in context representations that arises when dynamics and context encoders are trained simultaneously, as in the case of PE+LSTM. Moreover, by avoiding the need to learn a separate context model, our approach reduces data requirements and maintains a smaller overall model size. As discussed in Section \ref{sec:related-work-auxreps}, other methods that leverage auxiliary representations usually have to rely on multi-stage approaches and privileged environment information in order to train successful policies.

On the dynamics learning side, the PE-Y model, despite using and predicting the same variables as the Dubins model, exhibited poor real-world performance. Similarly, PE-RPY struggled to generalize to new terrains; its requirement to predict the robot's orientation (roll, pitch, and yaw) introduced compounding errors during model-based planning, which led to suboptimal strategies, as demonstrated in minute \href{https://youtu.be/7a-NzAackE4?t=51}{0:51 \textit{(PE-RPY)}} of the \textbf{SV1}. We further hypothesize that TD-MPC2's slower learning is attributable to its large model size, necessitating the simultaneous learning of five components (encoder, latent dynamics, reward predictor, terminal value predictor and policy), which may introduce unnecessary overhead for this specific task. Finally, while the model-free alternatives predictably underperformed in terms of learning efficiency, they underscore the advantages of a model-based approach for adaptive robot behavior. This is particularly compelling when contrasted with the policy-learning strategies for adaptive robot behaviors employed by other works, as discussed in Section \ref{sec:related-work-auxreps}.

\textbf{Comparison with Physics-Based Models.} The Dubins model was the most effective baseline but with critical flaws in complex real world scenarios. Note that its performance stems from being designed by an expert with extensive domain knowledge. This adds some limitations due to the stringent prerequisites to replicate this approach for other robots. Additionally, it operates under the assumption of ideal conditions, which limits its applicability in real-world, unstructured terrain navigation tasks. Critically, the model lacks a feedback mechanism to account for the specific characteristics of varying terrains, making it ill-suited for dynamic and complex environments. This limitation was evident in several scenarios where the robot's performance was unsatisfactory. For instance, the model was unable to compensate for latent factors affecting the robot's motion—as captured in minute \href{https://youtu.be/7a-NzAackE4?t=37}{0:37 \textit{(Dubins)}} of \textbf{SV1}—and frequently resulted in unsafe behaviors, both for the robot and its surroundings, as demonstrated in Figure \ref{fig:downhill-task}.

\textbf{On the Role of Exteroceptive Perception.} Integrating exteroceptive perception (e.g., camera or LiDAR) into our situational-awareness module offers a promising route to both accelerate and make more proactive the adaptation process. Rather than replacing the latent modeling core, perceptual cues can be used to form an informed prior over changepoint probabilities—replacing the constant hazard function used in Section~4.2.2—and thereby adaptively bias the model to expect particular transitions (e.g., entering gravel or grass) before they fully manifest in dynamics. This perceptually-informed prior reduces the representational burden on the SA module (it no longer needs to infer gross, easily observable factors from dynamics alone) and enables the planner to adopt anticipatory strategies instead of purely reactive ones.
Perception can also be coupled via existing traversability planners to provide intermediate objectives or risk estimates that guide the model-based controller. Important caveats for future work include robustly handling noisy or out-of-distribution sensor inputs and explicitly representing perceptual uncertainty so that incorrect priors do not degrade performance; principled fusion methods (e.g., uncertainty-aware Bayesian updates or learned perceptual priors) are natural directions. Overall, exteroceptive perception acts as a complementary information channel that can make latent adaptation faster and planning more proactive while leaving the SA module responsible for modeling residual, unmeasured latent factors (e.g., \textit{wet} vs.\ \textit{dry} grass).

\textbf{Broader Implications for Robotics.}
As we show in Section \ref{sec:results_ugv} (ground), Section \ref{sec:results_quadx} (aerial) \& Appendix \ref{sec:appendix-sa_results} (aquatic), our online hidden state representation learning method naturally extends to other platforms and environments. This suggests that our framework for situational awareness is a promising direction for adaptive methods, specially in field robotics applications. 
By learning hidden states that augment ego- and world-information, robots can handle variability and uncertainty in a wide range of tasks, thanks to reduced ambiguity. This ability to adapt to environmental and system-specific factors unlocks many possibilities without significantly adding a lot of complexity to the systems.

\subsection{Limitations and Scope} \label{sec:limitations}
While the proposed approach demonstrates significant advantages, it is not without its limitations. In simulation, we observed a slight decline in path smoothness metrics for the UGV evaluation, even as the overall success rate remained high. This trade-off appears to stem from the situation changes, which trigger re-planning with a new set of constraints provided by the dynamics model associated to the updated situation. These transitions, while enabling emergent adaptive behaviors, occasionally resulted in less smooth trajectories. However, this effect did not critically impact the robot's performance or overall mission success, suggesting that the benefits of adaptability outweighed the costs associated with these deviations.

A more fundamental limitation relates to the model's generalization capabilities, particularly when encountering dynamics that are out-of-distribution (OOD) with respect to the situations learned during training, which is a common restriction of most learning-based methods. Our SA framework excels at identifying and representing latent factors that are either represented within its context library or can be inferred from online transition data. However, its performance will predictably degrade if the test-time dynamics are entirely outside the support of this learned distribution. Such scenarios could include encountering fundamentally different physical phenomena (e.g., transitioning from rigid-body to deformable contact mechanics), experiencing novel forms of actuator or sensor degradation not seen during training, or operating in an environment with entirely new external forces. However, when the robot encounters a new environment, its observed state transitions are associated with the most probable situation symbol from the existing library. The planner then utilizes a dynamics model conditioned on this best-match context/situation. For instance, Figures~\ref{fig:1d_gaussians} \& \ref{fig:terrain_latent_factor_effect_statistics} illustrate that the dynamics for \textit{tiles} and \textit{concrete} are closely related, while both are distinct from those of \textit{pebbles}. Consequently, if a robot trained on \textit{concrete} were to encounter \textit{tiles} for the first time, the system is expected to reuse the same situation it employed for `concrete'. This mechanism is intended to reduce compounding error in the dynamics model and improve downstream planning performance.

In practice, OOD conditions can be diagnosed by monitoring key statistical indicators, such as a sustained increase in the predictive uncertainty of the dynamics model or a significant divergence between observed state transitions and those predicted by any available context. To mitigate this, the framework is amenable to established techniques; for instance, the context library can be made more robust by proactively expanding its support during training using methods like domain randomization to cover a wider array of potential dynamics. For truly novel situations, this diagnostics could trigger a safe-mode or a request for human intervention, forming a clear direction for future work in lifelong adaptation.

\section{Conclusion} \label{sec:conclusion}

To advance learning-based dynamics modeling for robots operating in unstructured settings, in this work we proposed a novel framework for adaptive robot decision-making based on representing the latent world- and ego-factors that introduce ambiguity and increase uncertainty. 
Our approach demonstrates that modeling latent factors through the joint distribution of state transitions—capturing what we define as the robot's \textit{situation}—is both effective and efficient for adaptive decision-making. This is achieved by extending the BOCD framework to a multivariate setting, enabling the robot to dynamically detect and adapt to shifts in its operational situation. Notably, we achieve this entirely through online, unsupervised learning, without reliance on privileged information or data-intensive training regimes typically reserved for large-scale simulation enterprises.
We demonstrated substantial improvements in 
the challenging domain of latent factor discovery, where achieving a balance between expressive representations, efficient data utilization, and minimization of inductive biases is critical.

This work lays the foundation for a promising new direction in adaptive dynamics learning, where robots can learn to adapt directly in real-world settings with limited data. Future research will explore how this approach can be further generalized, potentially transforming the way robots learn and operate in dynamic, uncertain environments.

\section{Acknowledgement}

The authors acknowledge Youwei Yu for his support during the real world experiments. 

\section{Declaration of Conflicting Interests}
The authors declared no potential conflicts of interest with
respect to the research, authorship, and/or publication of this
article. 

\section{Funding}
This work was primarily supported by NSF: CAREER: Autonomous Live Sketching of Dynamic Environments by Exploiting Spatiotemporal Variations (grant \#2047169) and partially supported by Army Research Lab (grant \#W911NF2520083). This work used Amazon Web Services through the CloudBank project, which is supported by National Science Foundation grant \#1925001.

\bibliographystyle{sageh}  
\bibliography{references.bib}

\appendix

\clearpage

\section{Latent Factor Effect Characterization}\label{sec:appendix-latent-effect-char}

Here we seek to present a real world example of the motivation for a latent factor representation learning method. Specifically, we replicated the motivating example from Figure \ref{fig:output-discrepancy} with a UGV traversing flat terrain. There we explained that executing the \emph{same action from similar states} can yield \emph{significantly different} outcomes due to the presence of latent, unmodeled factors (e.g., wind, terrain friction, actuator wear). Our method shows that by explicitly representing these latent factors and additionally conditioning the dynamics model on that representation helps reduce the prediction error and improve the control performance in unstructured environments.

We proceed to describe our focused real-world experiment that directly illustrates the phenomenon and motivates the conditional modeling approach: we collected \textbf{250} state transitions from comparable states after applying the same action ($v_x=0.5\ \mathrm{m/s}$, $\omega_z=0.0\ \mathrm{rad/s}$) across \textbf{14} different environments, for a total of $3500$ state transitions. These environments are presented in Figure~\ref{fig:latent_factor_effect_terrains}. The \textbf{Supplementary Video 4 (SV4)} showcases this data collection process: \url{https://youtu.be/xGSFT9QvtJ0}. This video illustrates our best effort at a controlled data collection process for a \textit{single} $(s, a, s')$ transition, where terrain type serves as the latent factor. The goal is to demonstrate that even for this simple case, aggregating all collected data without accounting for the latent factor can significantly affect predictions. Overall, this let us see that conditioning the dynamics model on the latent factor representation yields predictive improvements. While these gains are expected to be even more pronounced in challenging scenarios, this simple setup highlights that the issue already arises in the most basic case.

Figure~\ref{fig:terrain_latent_factor_effect_statistics} visualizes the resulting outcome distributions for this single case, to show where distributions align and where they diverge, seeking to explicitly demonstrate the effect of the latent factors after applying the same action from comparable states\footnote{We cannot say they are the same state as there are natural differences, but we tried to replicate them as closely as possible. Also, as discussed in Section \ref{sec:exp-setup-ugv}, the state is preprocessed to make it position invariant.}. The figures demonstrate that, while some environment pairs produce similar outcomes, others diverge substantially — exactly the behavior that motivates conditioning on an inferred situational representation. In particular, if we were to approximate the dynamics \textbf{for this particular case} using all the collected data (without conditioning on latent factors) as is traditionally done, it would fail to accurately model all dynamics. Even though this might not appear like a critical problem in this isolated motivating example, our overall results have shown that the problem is exacerbated in more complex domains and platforms, such as uneven terrain navigation or flying under evolving wind conditions. 

Figure~\ref{fig:1d_gaussians} also demonstrates this effect. Colored curves correspond to different terrain types, and the dashed black curve denotes the unconditioned (aggregated) estimate, demonstrating how aggregation obscures terrain-specific transition dynamics. Note that we do not claim that every terrain type requires a separate latent-factor representation. Instead, there is a trade-off between the number of modeled latent factors and their predictive benefit. For example, Figure~\ref{fig:terrain_latent_factor_effect_statistics} shows that the distributions for \textit{tiles} and \textit{concrete} are very similar, suggesting that a single representation may be sufficient in this case. Crucially, such discrepancies are often better attributed to an incomplete state representation—e.g., the robot lacks sensors, such as a camera, to detect terrain types—rather than inherent aleatoric errors. Additionally, note that our approach does account for aleatoric uncertainty in the transition model, as done in traditional dynamics learning methods; however, our primary focus is on reducing epistemic uncertainty by identifying and accounting for the latent factors influencing the dynamics.

\begin{figure}
    \centering
    \includegraphics[width=\linewidth]{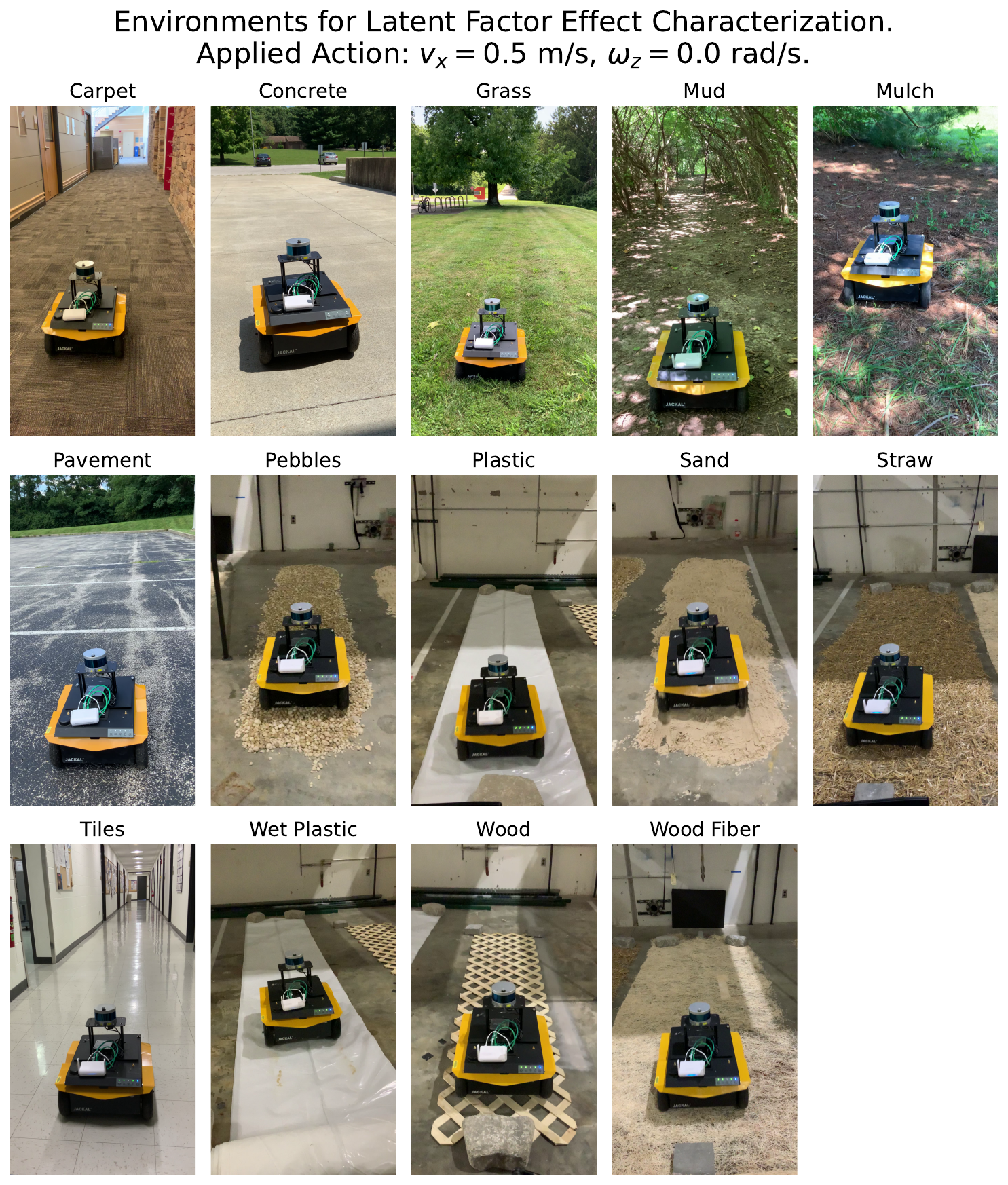}
    \caption{Environments used for data collection to study the influence of latent factors (in this case, terrain type) on UGV dynamics. In each environment, the same action was executed for 50 timesteps starting from a comparable initial state, and the procedure was repeated five times per environment.}
    \label{fig:latent_factor_effect_terrains}
\end{figure}

\begin{figure*}
    \centering
    \includegraphics[width=0.98\linewidth]{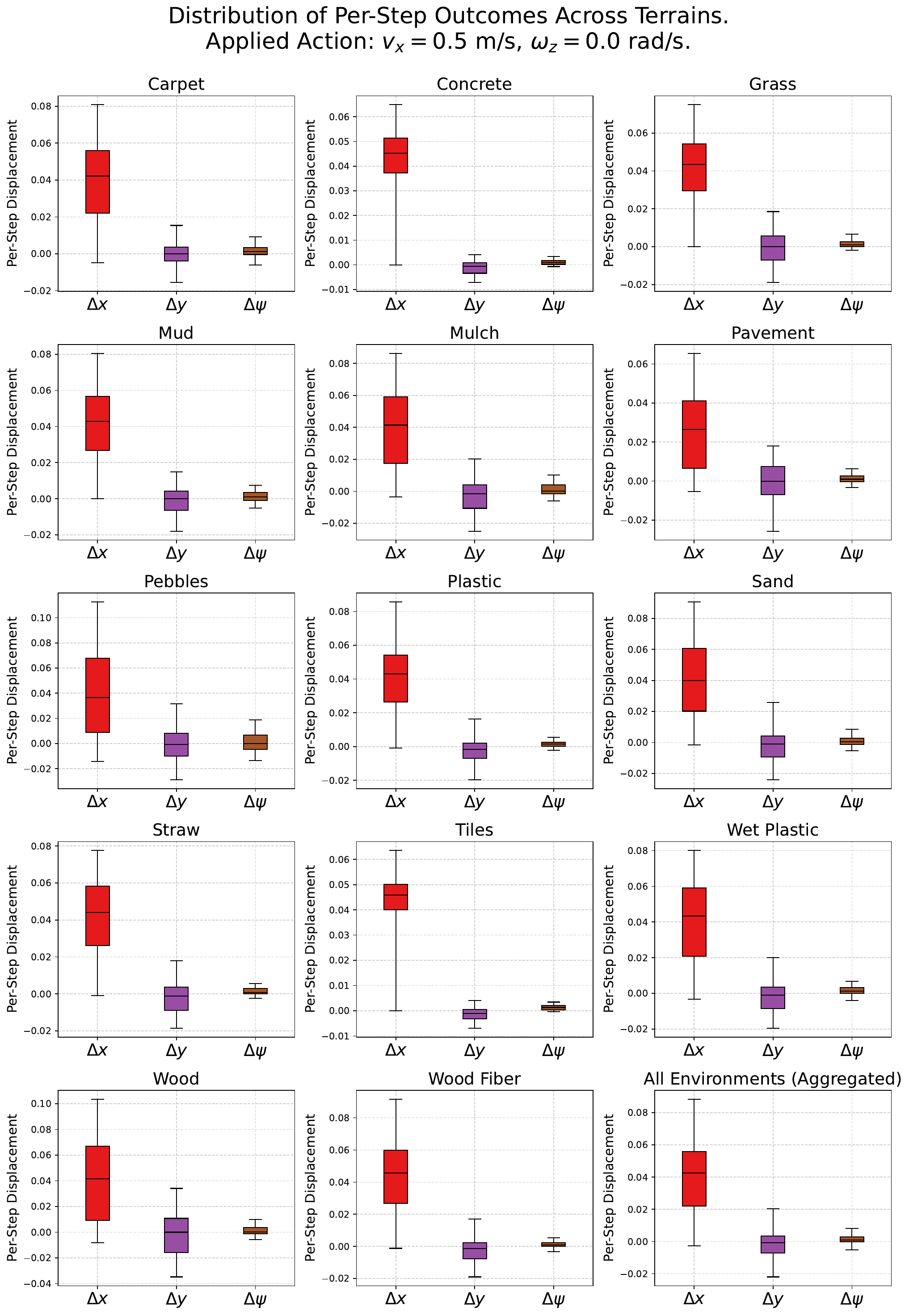}
    \caption{Descriptive statistics of the effect of the latent factor on the dynamics of the UGV. Data obtained from 250 state transitions after applying the same action $v_x=0.5~m/s$ and $\omega_z = 0.0~rad/s$.}
    \label{fig:terrain_latent_factor_effect_statistics}
\end{figure*}

\section{Additional Experimental Results} \label{sec:appendix-sa_results}

\subsection{Lake Simulator} \label{sec:appendix-lake_sim_sa_results}

\begin{figure} 
    \centering
    \includegraphics[width=\linewidth]{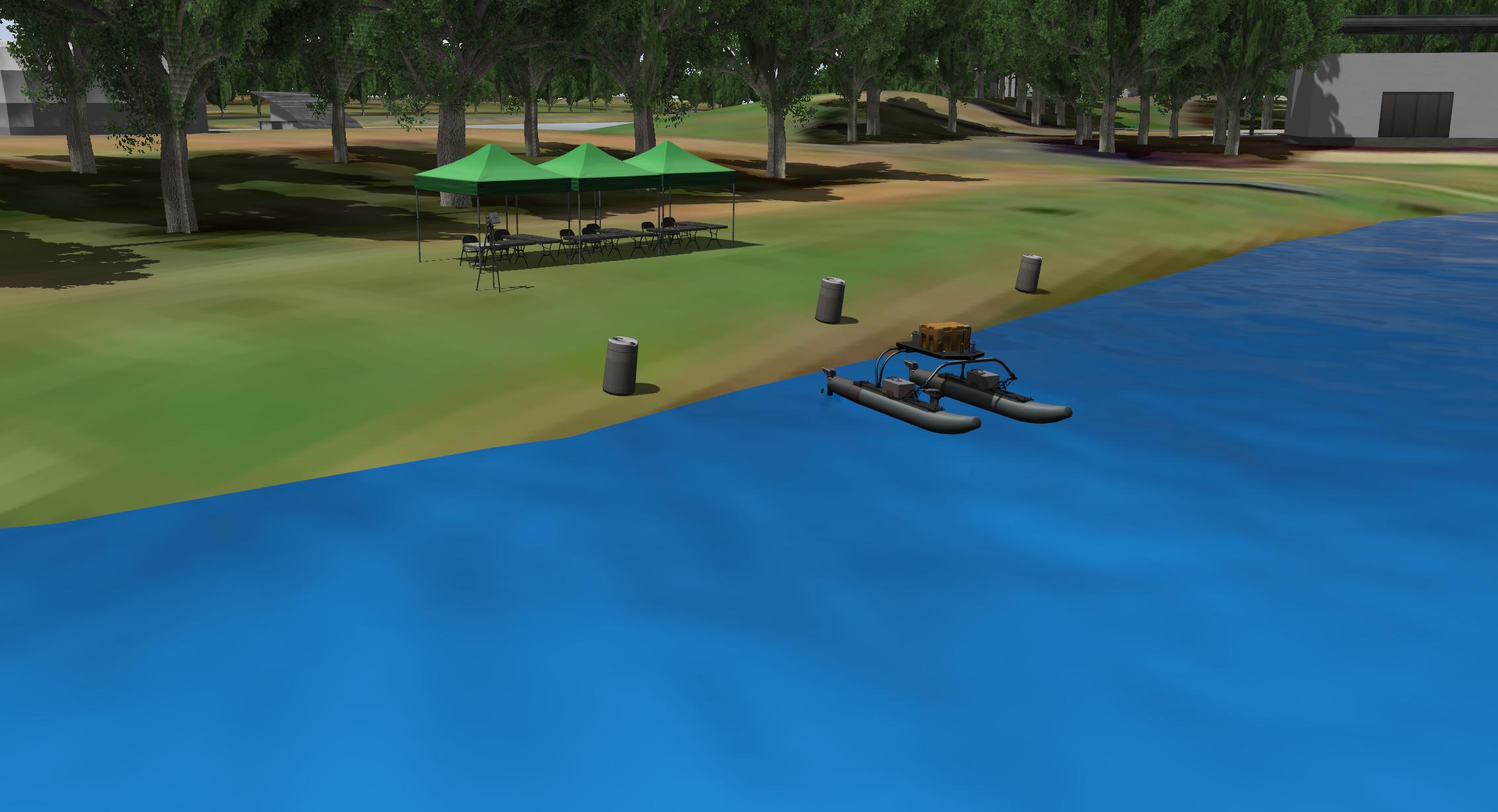}
    \caption{Lake World Simulator, where an ASV must navigate across the lake with intentionally changing wave patterns, relying solely on its proprioceptive odometry measurements for real-time information, without any prior knowledge of the conditions.}
    \label{fig:lake-world}
\end{figure}

Defining and identifying the hidden state in the unstructured terrain navigation task is inherently challenging due to the complexity of natural environments and their dynamic properties. Therefore, to further validate our online hidden state representation learning method, we conduct experiments in a controlled environment where we can directly affect the UDGP of the world. For this, we use the Lake World Simulator \cite[]{bingham19toward}, as shown in Figure \ref{fig:lake-world}. This environment provides a practical testbed for analyzing navigation in aquatic settings, where dynamic and unpredictable conditions are prevalent. Specifically, we simulate a scenario in which an Autonomous Surface Vehicle (ASV) must navigate across the lake while responding to shifting wave patterns about which it does not have any kind of information beyond what it can measure via its proprioceptive odometry measurements. The ASV’s objective is to navigate toward randomly generated targets while managing to do so with the changing wave patterns.

The wave generation patterns in this simulation are governed by parameters that are intentionally randomized during the experiment to induce shifts in the UDGP of the transition dynamics. These shifts represent transitions between different hidden states, such as calm waters, mild ripples, or high waves, which directly impact the ASV’s transition dynamics. Our method needs to obtain hidden state representations that encode this circumstances in real time.
This experiment is particularly meaningful because it reflects real-world challenges faced in aquatic robotics, such as navigating through varying sea states where external conditions like wind, currents, and waves cannot always be directly measured.

Figure \ref{fig:lake-world-changepoints} illustrates the sine and cosine transformations of the boat's roll and pitch angles during its navigation toward different randomly-generated goal locations within the lake. The results reveal that under conditions of significant wave pattern changes, frequent situation jumps are observed, whereas calmer wave conditions result in more stable situational states. Furthermore, the effect of the hyperparameter $\lambda$ during situation segmentation highlights the influence of the prior on changepoint probability. A lower $\lambda$, corresponding to a higher prior probability of changepoints, produces more granular segmentations of the observation space. In contrast, larger values of $\lambda$ demand more substantial changes in the UDGP to register a situation change.

\begin{figure*}[htbp]
    \centering
    \captionsetup{position=below, justification=centering}
    \subcaptionbox{$\mathbf{\lambda} = 20$.}{%
        \includegraphics[width=0.492\textwidth]{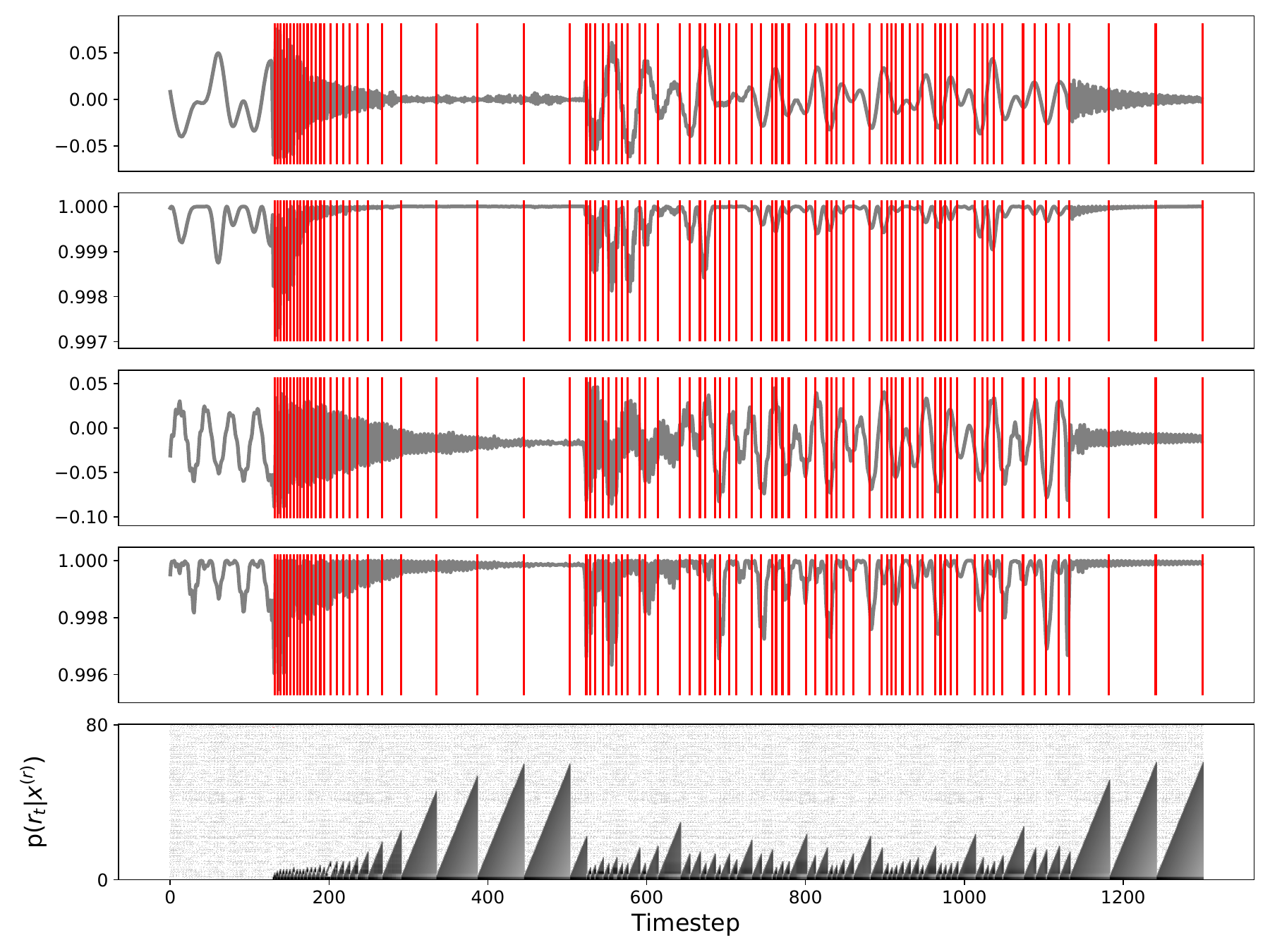}}
    \hfill
    \subcaptionbox{$\mathbf{\lambda} = 250$.}{%
        \includegraphics[width=0.492\textwidth]{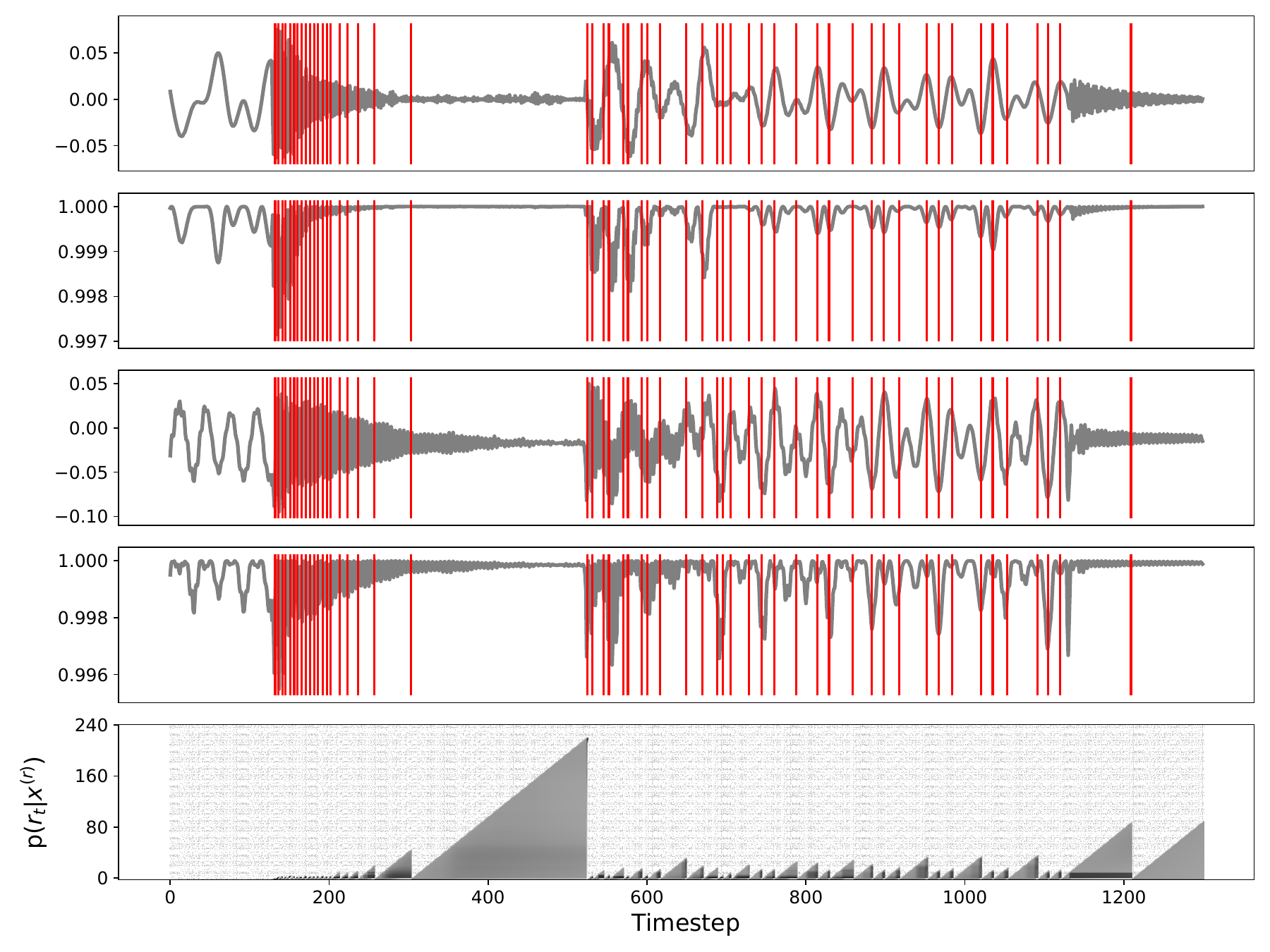}}
    \vfill
    \subcaptionbox{$\mathbf{\lambda} = 2~000$.}{%
        \includegraphics[width=0.492\textwidth]{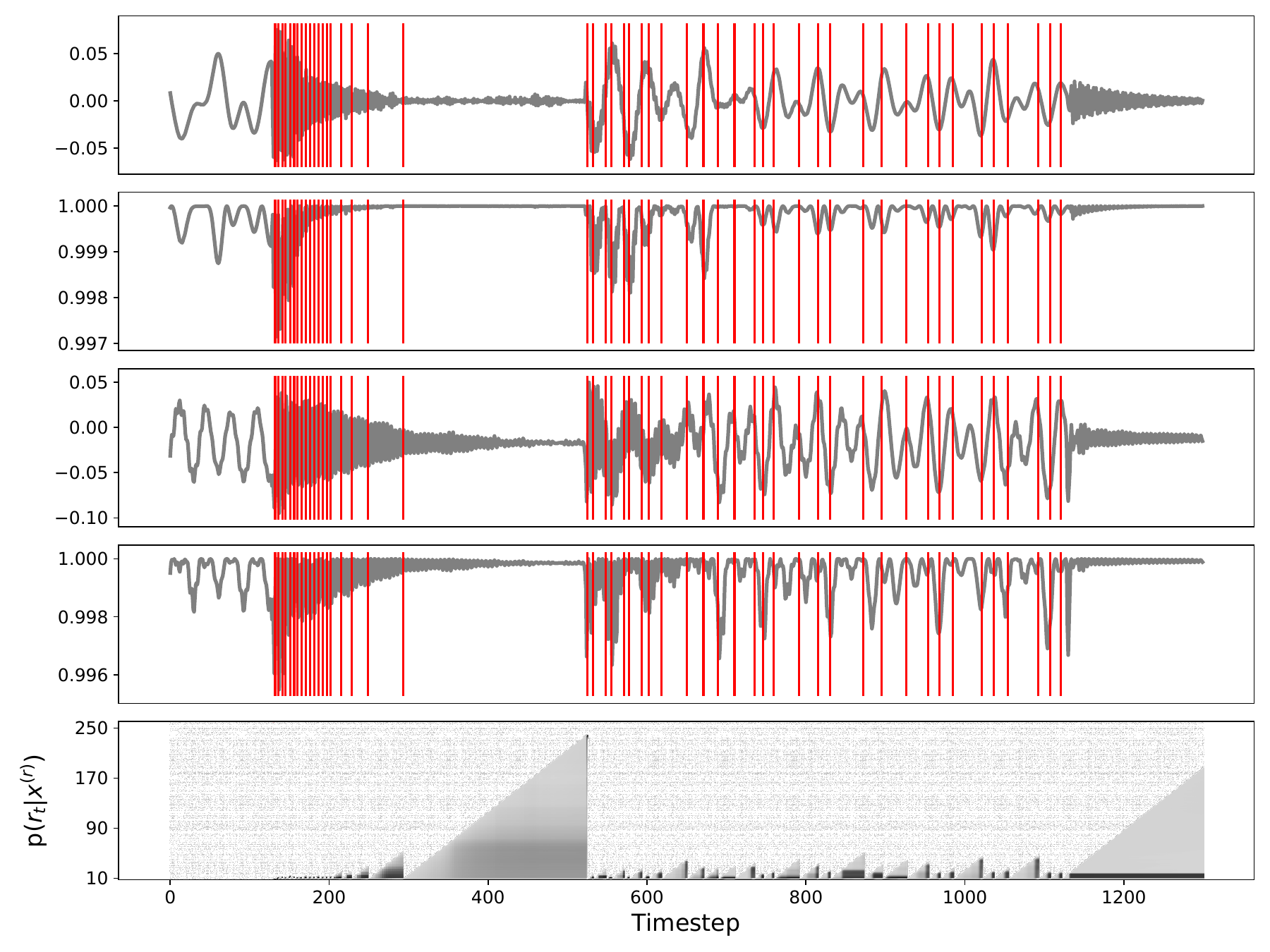}}
    \hfill
    \subcaptionbox{$\mathbf{\lambda} = 1~000~000$.}{%
        \includegraphics[width=0.492\textwidth]{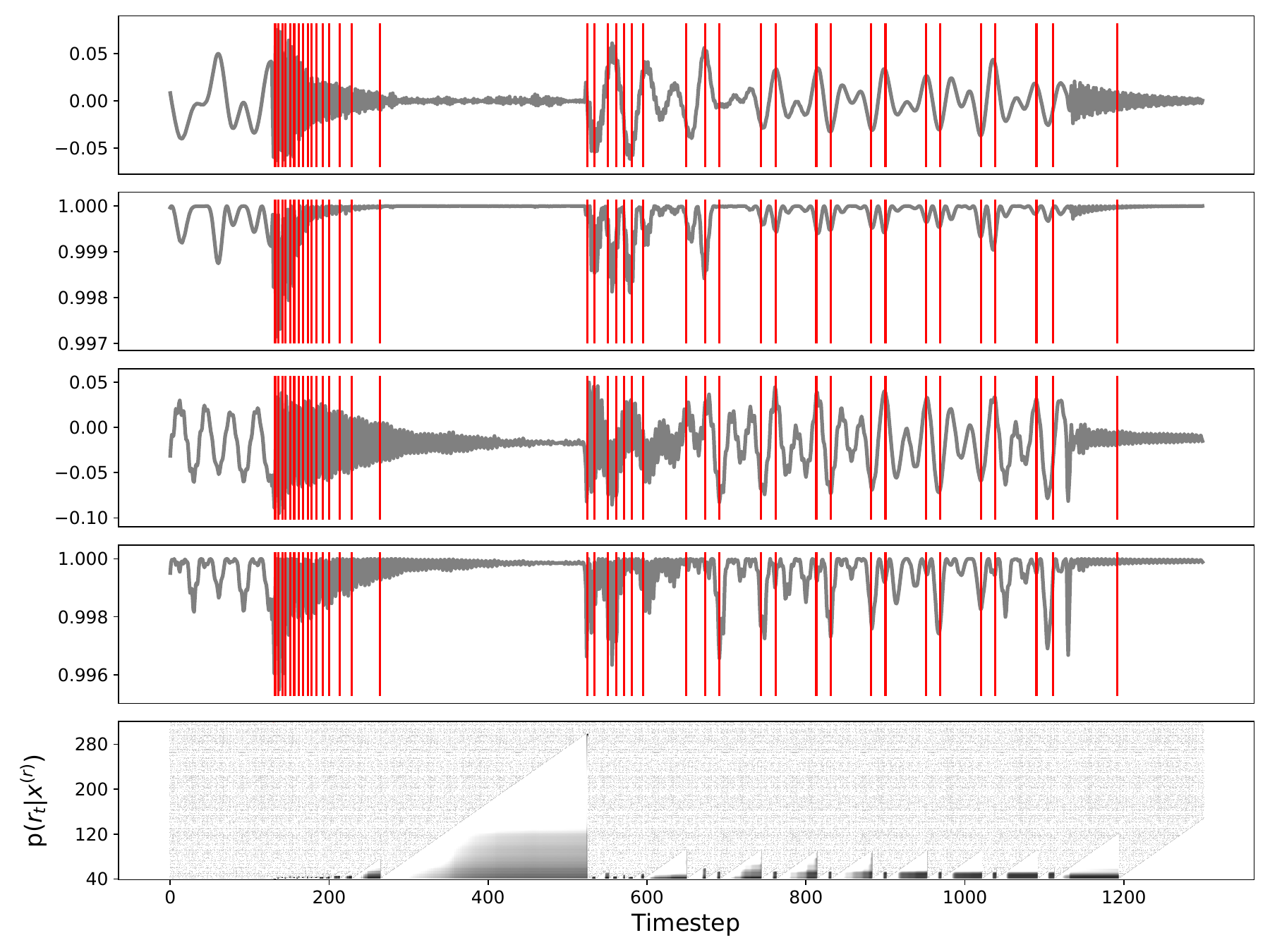}}
    \captionsetup{position=below, justification=justified}
    \caption{Online and unsupervised situation change detection results in the Lake World with randomly changing wave parameters. The red lines signal where our method detected a situation change. The last row of each subplot shows the probability that the situation stays the same. We present the results for different values of $\lambda$ to show our online hidden state representation method's sensitivity to the prior on the changepoint probability.}
    \label{fig:lake-world-changepoints}
\end{figure*}

\subsection{Difficult Terrain Navigation} \label{sec:appendix-inspection_world_sa_results}

Figure \ref{fig:reduced-observation-history} in the main manuscript illustrates changes in detection and modeling results for the UGV. To preserve the manuscript's flow and address space limitations, Figure \ref{fig:reduced-observation-history} presents only the first six dimensions of the 26-dimensional observations used for online situation modeling and identification. Figure \ref{fig:observation-history}, however, presents the full observation history, offering a more detailed view that highlights how effectively our multivariate extension of BOCD detects changes in the UDGP of the robot's dynamics during an unstructured terrain navigation task.

\begin{figure*}
 \centering
    \includegraphics[width=0.8\linewidth]{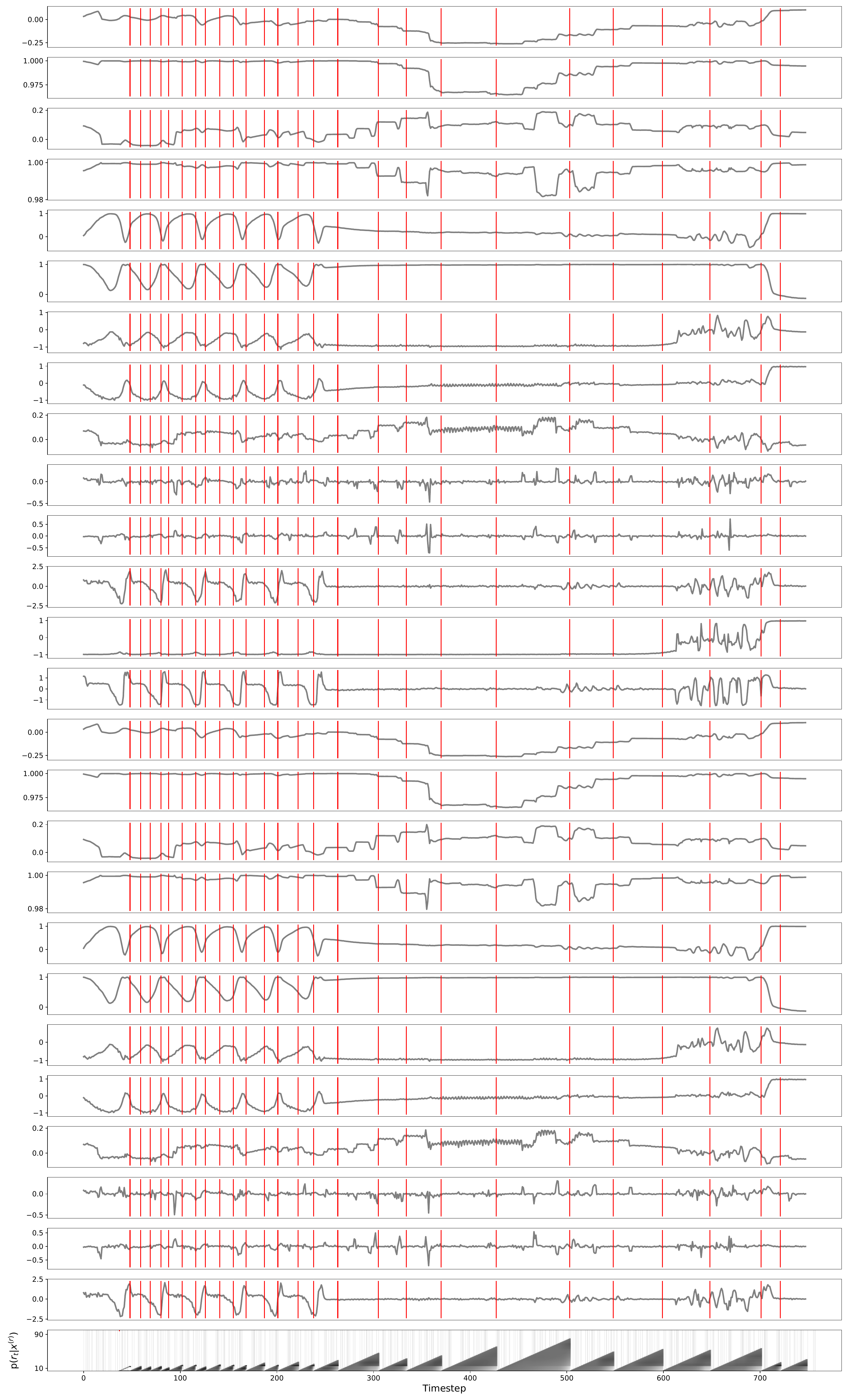}
    \captionof{figure}{Observations received by the proposed situation identification algorithm during a mission in the Inspection World. The red vertical lines mark the time when we predict a situation change happened. The last row shows the run length probabilities indicating how likely it is that the robot stayed in the same situation (run length grows) or a situation change happened (run length becomes zero).}
    \label{fig:observation-history}
\end{figure*}

\end{document}